\documentclass{article}

\makeatletter
\def\input@path{{neurips_2026/}}
\makeatother
\usepackage[preprint]{neurips_2026}

\usepackage[utf8]{inputenc}
\usepackage[T1]{fontenc}
\usepackage{hyperref}
\hypersetup{colorlinks=true,
linkcolor=orange,
citecolor=black,
urlcolor=black}

\usepackage{url}
\usepackage{booktabs}
\usepackage{titletoc}
\usepackage{amsfonts}
\usepackage{amsmath}
\usepackage{amssymb}
\usepackage{nicefrac}
\usepackage{microtype}
\usepackage{xcolor}
\usepackage{graphicx}
\usepackage{subcaption}
\usepackage{appendix} 
\usepackage{hyperref} 
\usepackage{cleveref} 
\usepackage{amsthm}
\usepackage{enumitem}
\usepackage{thm-restate}
\usepackage{caption}
\usepackage{mathrsfs}
\newcommand{\syml}[1]{\begingroup\color{orange}#1\endgroup}
\newcommand{\symr}[1]{\begingroup\color{cyan}#1\endgroup}

\setlist[enumerate]{noitemsep, nolistsep}

\newcommand{\N}{\mathbb N}

\newcommand{\R}{\mathbb R}

\newcommand{\E}{\mathbb E}

\newcommand{\tr}{\operatorname{tr}}

\DeclareMathOperator*{\mat}{Mat}
\DeclareMathOperator*{\diag}{diag}

\newcommand{\inv}{^{-1}}

\usepackage[ruled,vlined]{algorithm2e}

\crefname{algocf}{alg.}{algs.}
\Crefname{algocf}{Algorithm}{Algorithms}

\newtheorem{theorem}{Theorem}[section]
\newtheorem{lemma}[theorem]{Lemma}

\theoremstyle{definition}
\newtheorem{nb}[theorem]{Note}
\newtheorem{definition}[theorem]{Definition}
\theoremstyle{remark}
\newtheorem{remark}[theorem]{Remark}

\usepackage{xspace}

\newcommand{\lmo}{\operatorname{lmo}_{\|\cdot\|}}
\newcommand{\Muon}{\texttt{Muon}\xspace}
\newcommand{\muon}{\texttt{Muon}\xspace}
\newcommand{\Freon}{\texttt{Freon}\xspace}
\newcommand{\freon}{\texttt{Freon}\xspace}
\newcommand{\Kaon}{\texttt{Kaon}\xspace}
\newcommand{\kaon}{\texttt{Kaon}\xspace}
\newcommand{\sgd}{\texttt{SGD}\xspace}
\newcommand{\tsgd}{\texttt{TruncatedSGD}\xspace}

\usepackage[most]{tcolorbox}

\newtcolorbox{takeaway}[1][]{
    colback=gray!10,             
    colframe=black,            
    boxrule=1.0pt,             
    arc=5pt,                  
    auto outer arc,            
    left=10pt, right=10pt,     
    top=10pt, bottom=10pt,     
    #1                         
}

\title{Muon is Not That Special: \\ Random or Inverted Spectra Work Just as Well}

\author{%
  Zakhar Shumaylov$^1$ \,
  Natha\"el Da Costa$^2$ \,
  Peter Zaika$^1$ \,
  B\'alint Mucs\'anyi$^2$ \,
  Alex Massucco$^1$ \\
  \textbf{Yoav Gelberg}$^3$ \,
  \textbf{Carola-Bibiane Sch\"onlieb}$^1$ \,
  \textbf{Yarin Gal}$^3$ \,
  \textbf{Philipp Hennig}$^2$ \\
    \\
     $^1$University of Cambridge\phantom{$^1$} \; $^2$University of T\"ubingen \phantom{$^2$} \;  $^3$University of Oxford\phantom{$^3$}  
}

\begin{document}

\maketitle

\begin{abstract}
The recent empirical success of the \texttt{Muon} optimizer has renewed interest in non-Euclidean optimization, typically justified by similarities with second-order methods, and linear minimization oracle (LMO) theory. In this paper, we challenge this geometric narrative through three contributions, demonstrating that precise geometric structure is not the key factor affecting optimization performance. First, we introduce \texttt{Freon}, a family of optimizers based on Schatten (quasi-)norms, powered by a novel, provably optimal QDWH-based iterative approximation. \Freon naturally interpolates between \texttt{SGD} and \texttt{Muon}, while smoothly extrapolating into the quasi-norm regime. Empirically, the best-performing Schatten parameters for GPT-2 lie strictly within the quasi-norm regime, and thus \textit{cannot} be represented by any unitarily invariant LMO.
Second, noting that \texttt{Freon} performs well across a wide range of exponents, we introduce \texttt{Kaon}, an absurd optimizer that replaces singular values with random noise. Despite lacking any coherent geometric structure, \texttt{Kaon} matches \texttt{Muon}'s performance and retains classical convergence guarantees, proving that strict adherence to a precise geometry is practically irrelevant.
Third, having shown that geometry is not the primary driver of performance, we demonstrate it is instead controlled by two local quantities: \textit{alignment} and \textit{descent potential}. Ultimately, each optimizer must tune its step size around these two quantities. While their dynamics are difficult to predict a priori, evaluating them within a stochastic random feature model yields a precise insight: \Muon succeeds not by tracking an ideal global geometry, but by guaranteeing step-size optimality.
\end{abstract}

\section{Introduction}

First-order optimization algorithms have become increasingly central to modern machine learning, fueled by their routine use in training models with billions of parameters on trillions of tokens. While adaptive gradient methods like \texttt{AdamW} \citep{kingma2017adammethodstochasticoptimization, duchi2011adaptive, loshchilov2019decoupledweightdecayregularization} have long served as the dominant baselines, recent years have seen a surge of interest in matrix-based and spectral optimizers \citep{gupta2018shampoo, vyassoap, jordan2024muon}.

Algorithms such as \texttt{Shampoo} and \Muon have demonstrated remarkable success at scale \citep{liu2025muonscalablellmtraining, 5team2025glm45agenticreasoningcoding, kimiteam2025kimik2openagentic, deepseekai2026deepseekv4}, sparking a renaissance in the study of non-Euclidean descent methods \citep{bernstein2024old} and an increasing suite of proposed modifications \citep{du2026newtonmuonoptimizer, ahn2025dion, riabinin2025gluon, si2025adamuon, amsel2025polarexpressoptimalmatrix, gong2026aronewlensmatrix}. Theoretically, this success is almost universally justified through the lens of Linear Minimization Oracles (LMOs) and strict geometric preconditioning \citep{pethick2025training, kovalev2025understanding, fan2025implicit}. In this prevailing narrative, \Muon's performance stems from its exact adherence to a specific target geometry via complete spectrum whitening. Analysis of simple models \citep{kim2026sharpcapacityscalingspectral, ma2026preconditioningbenefitsspectralorthogonalization} commonly confirms this narrative, especially under distributional heavy tails \citep{yu2026signbasedoptimizerseffectiveheavytailed, wang2025muonoutperformsadamtailend}.

However, recent benchmarking studies have begun to reveal cracks in this geometric facade, noting that advantages of spectral optimizers diminish under proper baseline tuning and are highly sensitive to batch size \citep{wen2025fantasticpretrainingoptimizers, semenov2025benchmarkingoptimizerslargelanguage}. Concurrently, \citet{su2025isotropiccurvaturemodelunderstanding} 
showed with an isotropic curvature model that optimal updates need only preserve the ordering of singular values, whereas \Muon's whitening imposes much stronger constraints. Motivated by this, we initially set off to determine an optimal spectral geometry. Yet, much like analyses of \Muon on quadratics \citep{gonon2026insightsmuonsimplequadratics}, we ultimately find no simple geometric recipe reliably improves performance. Instead, we are forced to ask: does the exact geometric formulation of spectral optimizers matter, or is the LMO framework masking a more fundamental driver of optimization success?

In this work, we argue that while viewing optimization through the lens of non-Euclidean geometry provides an elegant theoretical framework, it fails to capture the core phenomena underlying the performance of spectral optimizers in deep learning. By systematically relaxing the assumptions of geometric preconditioning, we make the following specific contributions:

\begin{figure*}[t]
    \centering
    \includegraphics[width=\textwidth]{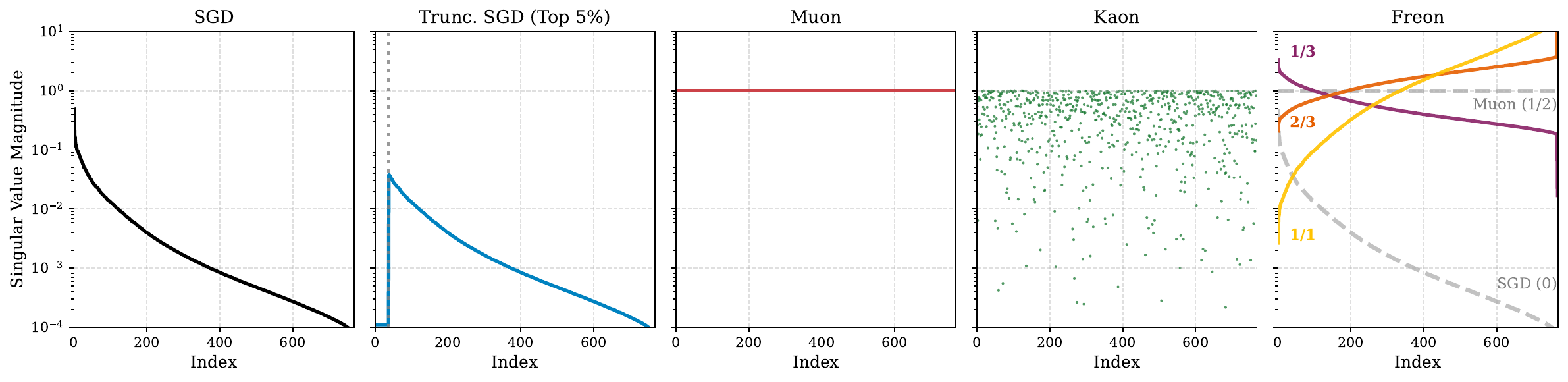}
    \caption{Transformations of the gradient singular value spectrum under various spectral optimizers for a GPT-2 layer. From left to right: \textbf{SGD} shows standard gradient spectrum. \textbf{Truncated SGD} explicitly zeros out the largest $5\%$ of singular values. \textbf{Muon} maps all singular values uniformly to $1.0$. \textbf{Kaon} sets the values randomly. Finally, \textbf{Freon} applies normalized updates of the form $(GG^\top)^{-c}G$, smoothly interpolating the spectral decay between standard SGD ($c=0$), Muon ($c=1/2$), and the pseudoinverse-like endpoint ($c=1$). Note the logarithmic scale on the y-axis.}
    \label{fig:optimizer_spectra}
\end{figure*}
\begin{enumerate}[leftmargin=*]
    \item To explore the broader space of order-preserving updates, we introduce \Freon, a family of Schatten (quasi-)norm updates of the form $(GG^\top)^{-c}G$. \Freon naturally interpolates between SGD ($c=0$) and Muon ($c=1/2$). Through a symmetry analysis analogous to \citet{yen2025lora}, we argue that extrapolation into the $c \geq 1/2$ regime is of immense theoretical interest (\Cref{thm:symm_informal}), despite plunging into a quasi-norm regime that fundamentally breaks standard LMO theory (\Cref{thm:preconditioned}). To compute these updates stably, we developed an optimal (\Cref{thm:best_approximation_fractional,thm:optimality}) QDWH-based iteration utilizing rational approximations (\Cref{alg:freon}) similar to \citet{nakatsukasa2010optimizing,zolopd}, extending the polar express theory of \citet{amsel2025polarexpressoptimalmatrix}. Empirically, unlike simple quadratics, we find that optimal exponents for a GPT-2 model concentrate in the quasi-norm region $c\in[1/2,1]$ (\Cref{fig:freon_layerwise_range_actual,fig:pq_lr_heatmap}), but varies significantly between batches. As this is strictly outside the range of any unitarily invariant norm, it demonstrates conclusively that standard LMO theory cannot be enough to explain performance gains.

    \item Our empirical sweeps in \Cref{fig:nanogpt_lr_sensitivity,fig:val_loss_all} revealed a profound conundrum: almost all \Freon updates with $c>0$ performed remarkably similarly, provided large singular values were sufficiently suppressed. This led us to hypothesize that the exact geometric structure of the update is largely irrelevant. To test this, we introduce \Kaon, an optimizer that computes no LMO, targets no specific Schatten norm, and simply replaces the gradient's singular values with noise (\Cref{alg:kaon}). Despite possessing zero coherent geometry, \Kaon closely matches \Muon's performance at scale (\Cref{fig:nanogpt_lr_sensitivity,fig:nanogpt_dense_val}) and curiously retains classical convergence guarantees (\Cref{thm:kaon_convergence}). This shows that precise target spectra are largely irrelevant for optimization performance.



    \item If LMOs are discarded, how can we mathematically understand the gap between all these updates? Extending the framework of \citet{davis2026spectralgradientupdateshelp}, we decompose the exact local Taylor expansion into two structurally revealing quantities: \textbf{batch gradient alignment} ($\gamma_k$) and \textbf{directional descent potential} ($\Phi_k$). While this is only mechanistic in quadratic models (\Cref{sec:asymptotics}), it cleanly exposes the fundamental tradeoff (\Cref{sec:empirics}): different optimizers implicitly choose how much alignment to sacrifice for increased descent potential, but actually realizing the benefit of a large $\Phi_k$ hinges on using step sizes that are tuned to exploit it.    
\end{enumerate}

\begin{figure*}[t]
    \centering
    \begin{minipage}[t]{0.48\textwidth}
        \centering
        \vspace{0pt}
        \includegraphics[width=\linewidth]{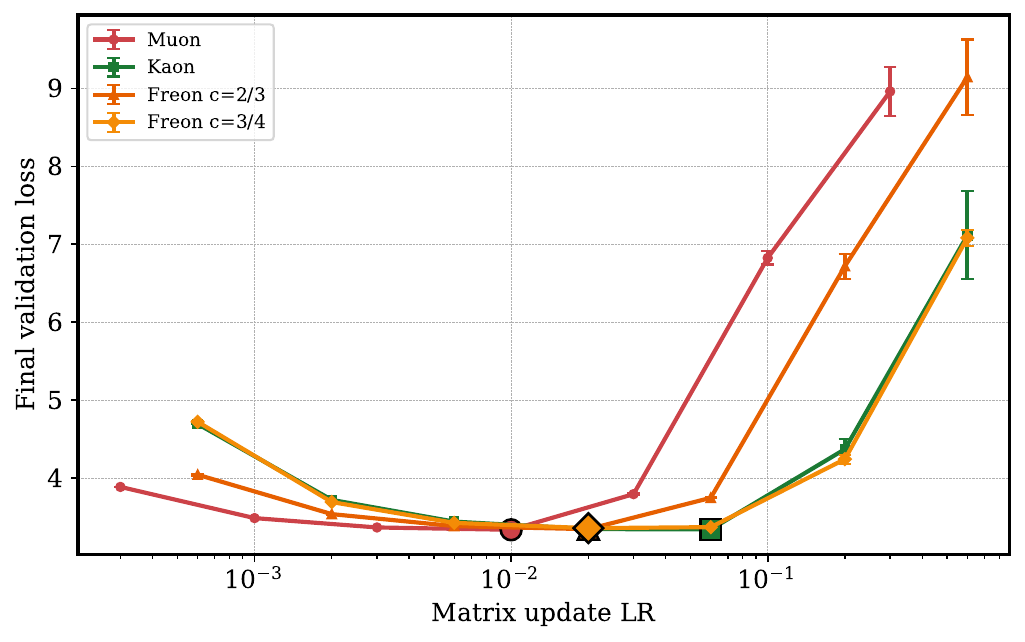}
        \caption{\textbf{NanoGPT learning rate sensitivity.} Final validation loss as the tuned learning rates are jointly scaled, averaged over three seeds with $\pm 2$ std error bars. Black outlines mark the best learning rate per optimiser.}
        \label{fig:nanogpt_lr_sensitivity}
    \end{minipage}
    \hfill
    \begin{minipage}[t]{0.48\textwidth}
        \centering
        \vspace{0pt}
        \includegraphics[width=\linewidth]{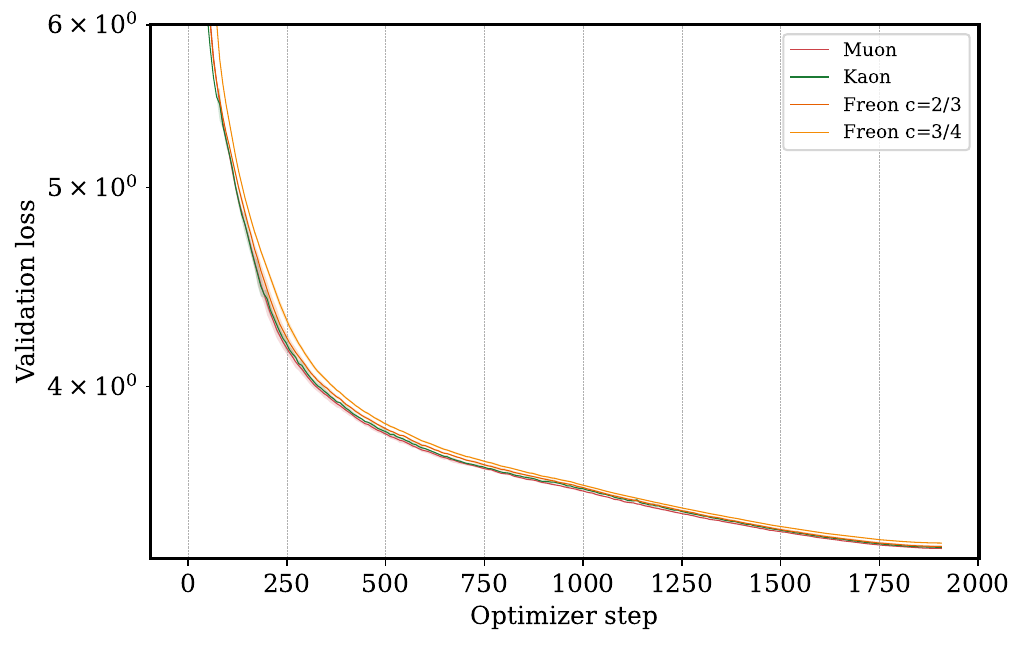}
        \caption{\textbf{NanoGPT validation curves.} Validation-loss curves for Muon, Kaon, and Freon at $c=2/3$ and $c=3/4$. Lines are averages over three seeds, with $\pm 2$ std bands. The y-axis is clipped to $[3.3,6]$ for legibility.}
        \label{fig:nanogpt_dense_val}
    \end{minipage}
\end{figure*}

\begin{figure*}[t]
\centering
\newcommand{\vpad}{\vphantom{(A_t)^p_{t}}}

\begin{minipage}[t]{0.32\textwidth}
\begin{algorithm}[H]
\footnotesize\DontPrintSemicolon\SetAlgoLined
\caption{\Muon}
\label{alg:muon}
Set $X_0 = G/\|G\|_F$\;
\phantom{Set $A_0 = X_0X_0^\top$}\; 
\For{$t = 1, 2, \ldots, T$}{
    $A_t = X_{t-1}X_{t-1}^\top \vpad$\;
    $B_t = bA_t + cA_t^2\phantom{B^{a}\,X_{t-1}} \vpad$\;
    $X_t = aX_{t-1} + B_tX_{t-1}\vphantom{B^{b}\,X_{t-1}} \vpad$\;
}
\phantom{\resizebox{.86\linewidth}{!}{$\mu = \bigl(\frac{1}{n}\langle X_T, G\rangle\bigr)^{(q+2p)/(2p-2q)}$}}\;
\Return $X_T$\;
\end{algorithm}
\end{minipage}\hfill%
\begin{minipage}[t]{0.32\textwidth}
\begin{algorithm}[H]
\footnotesize\DontPrintSemicolon\SetAlgoLined
\caption{\Kaon}
\label{alg:kaon}
Set $X_0 = G/\|G\|_F$\;
\phantom{Set $A_0 = X_0X_0^\top$}\; 
\For{$t = 1, 2, \ldots, T$}{
    $A_t = X_{t-1}X_{t-1}^\top \vpad$\;
    $B_t = (I - A_t)^2\phantom{B^{a}\,X_{t-1}}\vpad$\;
    $X_t = 4.1\cdot B_tX_{t-1}\phantom{B^{b}\,X_{t-1}} \vpad$\;
}
\phantom{\resizebox{.86\linewidth}{!}{$\mu = \bigl(\frac{1}{n}\langle X_T, G\rangle\bigr)^{(q+2p)/(2p-2q)}$}}\;
\Return $X_T / 1.175\vpad $\;
\end{algorithm}
\end{minipage}\hfill%
\begin{minipage}[t]{0.32\textwidth}
\begin{algorithm}[H]
\footnotesize\DontPrintSemicolon\SetAlgoLined
\caption{\Freon}
\label{alg:freon}
Set $X_0 = G/\|G\|_F$\;
Set $A_0 = X_0X_0^\top$\;
\For{$t = 1, 2, \ldots, T$}{
    $B\phantom{_t} = R_t(A_{t-1})\phantom{X_{t-1}X_{t-1}^\top}\vpad$\;
    $X_t = B^{a}\,X_{t-1} \vpad$\;
    $A_t = B^{b}\,A_{t-1} \vpad$\;
}
\resizebox{.86\linewidth}{!}{$\mu = \bigl(\frac{1}{n}\langle X_T, G\rangle\bigr)^{(a+2b)/(2a-2b)}$} \;
\Return $X_T/\mu$\;
\end{algorithm}
\end{minipage}
\caption*{\textbf{Algorithms}: Simplified pseudocode illustrating the main differences between the optimizers considered. Full algorithm for Freon is shown in \Cref{alg:coupled_chol}, utilizing block-QR for the rational map.}
\end{figure*}
\section{Background and Methods}\label{sec:background}
We begin by considering the classical approach to introducing \muon, based on the idea of regularized steepest descent. Given a norm $\|\cdot\|$ on an inner product space (spectral norm in the case of \muon), we define the Linear Minimization Oracle (LMO) and the corresponding dual norm as:
\begin{equation}
\label{eq:lmo}
\lmo({V})=\underset{\|{U}\|\leq1}{\arg\max\;} \langle{U}, {V}\rangle \qquad \|{V}\|_*=\underset{\|{U}\|=1} {\max}\langle U, V\rangle = \langle V,\lmo(V)\rangle.
\end{equation}
Then, for a differentiable loss function $f: \mathbb{R}^d \rightarrow \mathbb{R}$, with $G_k = \nabla f(X_k)$,
regularized steepest descent updates by minimizing a first-order approximation of the loss under a quadratic penalty:
\begin{equation*}
{X}_{k+1}=\underset{{X}}{\arg \min }\left\{f({X}_k)+\left\langle{G}_k,
{X}-{X}_k\right\rangle+\frac{1}{2 \eta}\left\|{X}-{X}_k\right\|^2\right\} = {X}_k-\eta\left\|{G}_k\right\|_* \lmo\left({G}_k\right).
\end{equation*}
For example, for \Muon, $\lmo (W) = (WW^\top)^{-1/2}W$, where the inverse is taken in the sense of a Moore-Penrose pseudo-inverse.
Unless otherwise specified, the proofs of theorems in this section can be found in \Cref{app:proofs}.


\subsection{Limitations of the LMO Framework}\label{sec:limitations_lmo}
For simplicity, in this section we restrict our attention to a single layer, with a matrix domain $\R^{m\times n}$. We consider a differentiable loss function $f:\R^{m\times n} \to \R$. Let $r = \min(m, n)$ denote the maximal rank, and write the singular value decomposition $G_k = \nabla f(X_k) = U_k \diag(\sigma_k) V_k^\top$ with, for simplicity, strictly positive singular values $\sigma_k \in \mathbb{R}_{>0}^r$. We define Preconditioned Spectral Descent updates as updates of the form $X_{k+1} = X_k - \alpha_k\langle G_k,D_k\rangle D_k$ with $D_k = U_k \operatorname{diag}(p_k(\sigma_k)) V_k^{\top}$ for some sequence of vector valued mappings $p_k: \mathbb{R}_{>0}^r \to \mathbb{R}_{>0}^r$.

The two simplest members of this family are gradient descent (GD) and \Muon (spectral descent). GD sets $p_k(S) = S/\|S\|_2$, while \Muon sets $p_k(S) = {1}$, practically achieved using Newton--Schulz iterations (\Cref{alg:muon}). However, not all preconditioned spectral descent updates are steepest descent updates: \freon for $c>1/2$, \tsgd and \kaon are not, based on the following:\begin{restatable}[Preconditioned Spectral Descent vs. Regularized Steepest Descent]{theorem}{thmpreconditioned} \label{thm:preconditioned}Fix the function $f$ and the point $X_k$. 
A regularized steepest descent update w.r.t.~a unitarily invariant matrix norm from $X_k$ is a preconditioned spectral descent update. On the other hand, a preconditioned spectral descent update from $X_k$ is a steepest descent update w.r.t.~a unitarily invariant matrix norm only if $p_k$ preserves the order $\leq$ of the entries of $\sigma_k$.
\end{restatable}
\begin{nb}
This theorem exposes a fundamental limitation of LMOs; maximizing descent cannot remove `bad' large or intermediate singular directions from a batched gradient.
\end{nb}

But, as we show, LMOs are not necessary for convergence of preconditioned spectral descent.\begin{restatable}[Convergence of Preconditioned Spectral Descent]{theorem}{thmcvgspectral}\label{thm:zh_conv}
Let $\|\cdot\|$ be a unitarily invariant matrix norm, and let $\|\cdot\|_{*}$ be its dual norm. Let $f: \mathbb{R}^{m \times n} \to \mathbb{R}$ be continuously differentiable, bounded below by $f_{\min}$, with $L_{\mathcal{X}}$-Lipschitz continuous gradients with respect to $\|\cdot\|_{}$.

Consider preconditioned spectral descent updates as in \Cref{sec:limitations_lmo}, and assume there exist constants $m_k \geq 0$ and $M_k>0$ satisfying the divergence condition $
\sum_{k=0}^\infty \left( {m_k}/{M_k} \right)^2 = \infty
$, such that for all $\sigma \in \mathbb{R}_{>0}^r$, the mapping $p_k$ satisfies the following sufficient descent and boundedness conditions:
\[
\langle \sigma, p_k(\sigma) \rangle \ge m_k \|\sigma\|_{*} \qquad \text{and} \qquad \|p_k(\sigma)\| \le M_k.
\]
Then, choosing $\alpha_k = \eta\frac{1}{L_{\mathcal{X}} M_k^2}$ for any $\eta \in (0, 2)$, the following convergence rate holds for $K \ge 1$:
\[
\min_{0 \le k < K} \|G_k\|_{*}^2 \le \frac{L_{\mathcal{X}} (f(X_0) - f_{\min})}{\eta(1-\frac{\eta}{2}) \sum_{k=0}^{K-1} \left( \frac{m_k}{M_k} \right)^2} \qquad \text{and} \qquad \liminf_{k \to \infty} \|G_k\|_{*} = 0.
\]
\end{restatable}

\begin{nb}
The proof relies on standard convergence arguments, requiring that updates do not become too small, do not blow up and are mostly in the right direction. When $p_k$ is derived from an $\lmo$, these constants evaluate to $m_k = M_k = 1$. Crucially, this theorem accommodates non-norm-descent methods like \Kaon and \Freon. While this result does not yield a tight convergence rate, it provides an essential sanity check: such methods will work in practice.
\end{nb}

\begin{figure*}[t]
    \centering
    \includegraphics[width=\textwidth]{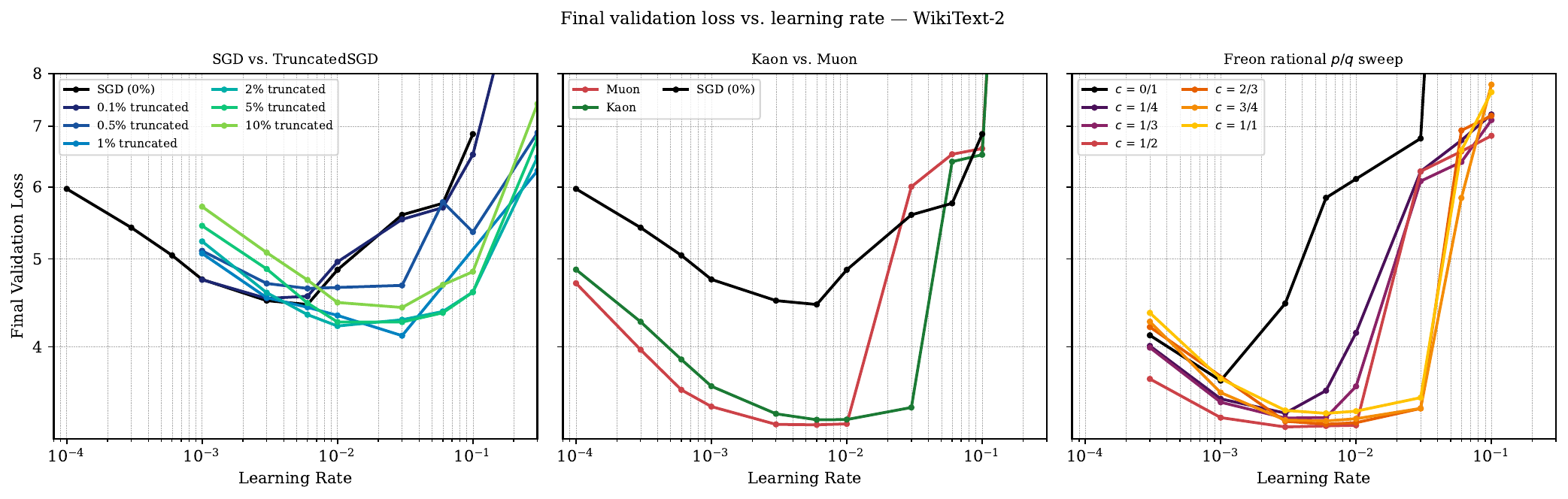}
    \caption{\textbf{Final validation loss versus learning rate for all optimizer families on
    WikiText-2 (118M tokens)}. \textbf{(a)} \tsgd improves over \sgd but fails to
    close the gap to \Muon~-- suppressing large singular values is necessary but not
    sufficient. \textbf{(b)} \Kaon closely matches \Muon despite replacing singular values with
    random noise. \textbf{(c)} \Freon with
    $c \approx 2/3$ closely matches \Muon; the optimal $c$ lies strictly outside the
    range of any proper Schatten norm ($c > 1/2$).}
    \label{fig:val_loss_all}
\end{figure*}

\subsection{TruncatedSGD}
\label{sec:truncated_sgd}
As a first step in relaxing the strict geometric assumptions of the LMO framework, we isolate the role of the largest singular values in optimization stability. We study \tsgd, an optimizer that zeros out the top $p$\% of singular values while leaving the rest unchanged. Although a full SVD is impractical, \tsgd serves a pedagogical purpose: it tests whether standard \sgd fails primarily due to instability induced by excessively large singular values.

Concurrently with our work, \citet{jiang2026enhancingllmtrainingspectral} propose a Newton-Schulz based approximation to fixed-value clipping, similarly targeting large singular values under the assumption that batch noise concentrates in this part of the spectrum. In line with this motivation, we observe in \Cref{fig:val_loss_all} that \tsgd improves over \sgd, confirming that suppressing large singular values is \emph{necessary} for stability. However, it still falls short of \muon, showing that suppression alone is \emph{insufficient}. The view that clipping primarily removes noise implicitly treats the underlying gradient direction as intrinsically meaningful; our results in \Cref{sec:empirics} ultimately challenge this premise, suggesting instead that the gradient direction itself can be poorly aligned with useful progress.

\subsection{Kaon}
\label{sec:kaon}
Motivated by \tsgd's gains over \sgd and observation of \Cref{sec:freon} that only the broad suppression/amplification of singular values matters, we ask: does the target geometry matter at all? To test this, we replaced the gradient’s singular values with uniform random noise. This stochastic spectrum redistribution still performed surprisingly well, closely matching \Muon.

To make this `absurd optimizer' efficient, we draw on ideas for pseudorandom number generation using chaotic iterations \citep{phatak1995logistic}. We use the generalized third-order logistic recurrence $x_{t+1} = \lambda x_t (1 - x_t^2)^2,$ lifted to matrix iterations (\Cref{alg:kaon}). Setting $\lambda = 4.1$ pushes the recurrence deep into the chaotic regime. We visualize an example in \Cref{fig:optimizer_spectra}, and illustrate the stationary distribution of these iterations in \Cref{ap:Kaon}, \Cref{fig:kaon_pdf_grid}.

We emphasize that \Kaon is primarily a pedagogical construction. Its strong performance (\Cref{fig:val_loss_all}) is a direct counter-argument to LMO-based explanations of spectral optimizers: state-of-the-art performance does not require adherence to a precise target geometry, and can even be achieved with random spectral reshaping, while retaining standard convergence guarantees:

\begin{restatable}[Convergence of Random Spectral Descent]{theorem}{thmcvgrandom}\label{thm:kaon_convergence}
Let the setting be identical to \Cref{thm:zh_conv}. Suppose at each iteration $k$, the stochastic mapping is defined specifically as $p_k(\sigma) = \frac{E_k}{\|E_k\|}$, where the entries of $E_k$ are i.i.d.~samples of a probability distribution on $\R_{>0}$ independent of $k$.

If $\alpha_k = \frac{\eta}{L_{\mathcal{X}}}$ for any constant $\eta \in (0, 2)$, then almost surely, the sequence $\min_{0 \le k < K} \|G_k\|_{*}^2$ achieves a convergence rate of $\mathcal{O}(1/K)$.
\end{restatable}
\subsection{Freon}\label{sec:freon}
Following the setup of \Cref{sec:background}, we consider steepest descent in Schatten $p$-norms, defined as $\|X\|_p = \left(\sum \sigma_i\{X\}^p\right)^{1/p}$. By H\"older's inequality, the dual of the $p$-Schatten norm is the $q$-Schatten norm \citep{bhatia2013matrix}, where $\frac{1}{p} + \frac{1}{q} = 1$ for $p \in [1, \infty]$. As before, let the singular value decomposition of the gradient be $G = U \operatorname{diag}(\sigma) V^\top$.
The steepest descent direction for $p\in[1,\infty]$ is:
\begin{equation}
\operatorname{lmo}_{\|\cdot\|_p}(G) = U \operatorname{diag}\left(\left(\frac{\sigma}{\|G\|_q}\right)^{q-1}\right) {V}^\top\quad \text{or equivalently} \quad D = \left(G_n G_n^\top\right)^{-c}G_n,\label{eq:p_schatten_update}
\end{equation}
where we denote the dual-normalized gradient as $G_n=G/\|G\|_q$ and $c= 1-q/2$. This update structure, parametrized by the dual exponent $q$, forms the basis of the general \freon optimizer.

\paragraph{A note on $p \in (0, 1)$:}
While the derived update rule extends analytically to $p \in (0, 1)$ (where $q < 0$), the $p$-Schatten ceases to be a proper norm in this regime, becoming a \textit{quasi-norm} due to the violation of the triangle inequality. Strictly speaking, the dual of a $p$-quasi-norm for $p < 1$ is the $\infty$-norm (spectral norm). Consequently, exact steepest descent under a true $p$-quasi-norm constraint does not yield the continuous deformation above (\Cref{thm:preconditioned} prevents that); rather, the LMO collapses, activating only the singular vector of the largest singular value. Therefore, the general \Freon update with $q < 0$ does not correspond to a steepest descent under the quasi-norm either. It is better viewed as a preconditioned spectral descent (c.f.~\Cref{sec:limitations_lmo}).
\subsubsection{Norm Hyperparameter Transfer}\label{sec:hyper_trans}
If one were to execute the exact update derived above with the same step size across different values of $p$, the optimization would fail dramatically. This is because restricting the update to the standard unit $p$-ball $(\|\Delta W\|_p \leq 1)$ causes the maximum achievable step size in the spectral space to shrink at a rate dependent on $p$, as the rank $r=\min(m,n)$ increases. To anchor the updates to a consistent scale regardless of $p$, we must tie the updates to a single overarching invariant ball: the spectral ball. We achieve this by scaling the update constraint precisely by $r^{1/p}$, enforcing
$\|\Delta W\|_p = r^{1/p}.$ This means that for any $p>0$, $\frac{1}{r^{1/p}}\|\Delta W\|_p\leq\|\Delta W\|_\infty$. By shifting to this mean Schatten norm, defined by $\|\cdot\|_{p, \texttt{mean}} = \frac{1}{r^{1/p}}\|\cdot\|_p$, we geometrically inflate the $p$-Schatten ball, guaranteeing that the baseline maximum step size is always $\mathcal{O}(1)$ independent of $p$, as domain of the $\lmo$: $\{\|D\|_{p, \texttt{mean}}\leq 1\}\subseteq\{\|D\|_{\infty}\leq 1\}$ for all $p\geq0$. We also note, that concurrently with our work, a similar normalization was introduced by \citet{xu2026widthscalingneuraloptimizers}, restricted to the context of vector-induced row-wise and column-wise $p, q$ norms, based on arguments of dimension independence.

\subsubsection[Particular Case of p=0 (c=1)]{Particular case of $p=0$ ($c=1$)}\label{sec:freon11}

The mean-norm scaling introduced in \Cref{sec:hyper_trans} does more than just stabilize hyperparameters; it mathematically enables us to explore the limiting case where $p \to 0$ (and $q \to 0$). When combined with the normalization from \Cref{sec:hyper_trans} the scaling in \Cref{eq:p_schatten_update} converges: $\|G\|_q / r^{1/q} \xrightarrow[]{}\operatorname{det}(GG^\top)^{1/2n}$ as $q\to 0$.
The beauty of the $p=0$ limit extends far beyond this simple closed form; up to scaling, this specific update turns out to be invariant under layerwise symmetries.
\begin{theorem}[Equivariance of \Freon$(c=1)$. Informal, formal statement in \Cref{thm:symm}]\label{thm:symm_informal} Consider an arbitrary neural network $\phi_W$ with layered parameters $W= (W_1,\dots,W_L)$. We define a symmetry of $\phi_W$ to be any transformation $W\mapsto g\cdot W$ such that $\phi_W = \phi_{g\cdot W}$. Then updates of the form $W_l - \eta(G_lG_l^\top)\inv G_l$ for each layer $W_l$ are equivariant under scaling and orthogonal layerwise symmetries of $\phi_W$, i.e. $g\cdot W^k = [g\cdot W]^k$, where $W^k$ represents the parameters after $k$ updates. Therefore, the optimization trajectory of the neural network $\phi_W$ is invariant, i.e. $\phi_{[g\cdot W]^{k}} = \phi_{W^{k}}$.
\end{theorem}



\subsubsection{Efficient Computation: From Polar to H\"older Express}
To implement \Freon practically, we require an efficient method to compute terms of the form $(GG^\top)^{-\frac{a}{b}}G$ from \Cref{eq:p_schatten_update}, where $q = 2(1-\frac{a}{b})$ and $\frac{a}{b} \in \mathbb{Q}$.
\paragraph{On polynomial iterations:}
A natural approach is a polynomial scheme akin to \citep{jordan2024muon,amsel2025polarexpressoptimalmatrix}. Extending the Polar Express theory (\Cref{thm:optimality}), we find optimal polynomials $P_k$ such that $xP_k(x^b) \to 1$ under composition (Appendix \ref{sec:optimality}, Remark \ref{rem:best_polynomial_approximant}), which can be applied iteratively via updates of the form $O_{k+1} = P_k\left((O_kO_k^\top)^{b/2} (GG^\top)^{a-b/2}\right)O_k$. While theoretically convergent (Appendix \ref{sec:optimality}, Remark \ref{rem:polynomial_error_bound}), these iterations are highly unstable in practice for $a/b \neq 1/2$ (\Cref{app:holder_instability}). Requiring a computationally prohibitive $\sim \mathcal{O}(b\log{\sigma_{\min}})$ steps, they diverge rapidly, especially in lower precision. This instability is well-known for polynomial iterations outside of $a/b=1/2$ \citep{higham2008functions}. In line with \citet{GramNewtonSchulz}, we similarly identify two primary causes: the explicit formation of poorly conditioned matrices (which introduces spurious negative eigenvalues) and the rapid accumulation of floating-point errors. A robust method must therefore avoid explicitly forming such matrices and converge in as few steps as possible. 

\paragraph{On rational iterations:}
Generalized rational functions, rooted in Zolotarev's best rational approximants to the sign function \citep{achieser, zolopd}, offer significantly greater stability. For instance, the QR-based Dynamically Weighted Halley (QDWH) method \citep{nakatsukasa2010optimizing, zolopd} uses these to approximate the polar map. Drawing on this classical theory \citep{achieser, cheney1964generalized} and use in matrix computations \citep{zolopd, gawlik2021approximating}, we use the Remez algorithm \citep{Trefethen2025} to iteratively approximate the sign function using rational functions of the form $xR(x^{2b})=x\frac{a+bx^{2b}}{1+cx^{2b}}$ akin to \citet{amsel2025polarexpressoptimalmatrix}, establishing optimality of this approach:
\begin{theorem}[Proof in \Cref{sec:optimality}]\label{thm:best_approximation_fractional}
Let $a/b \in \mathbb{Q}$, $l > 0$ and suppose $G$ satisfies $\sigma_{\min}(G) \geq l$. Moreover, set $\mathcal{G} := (GG^\top)^{-\frac{a}{b}}G$, $\mathcal{G}^\dagger := G^\top (GG^\top)^{\frac{a}{b} -  1}$ and $\mathfrak{R}_{n, d}\left(\frac{b}{2}, \frac{b}{2}\right)$ as in \eqref{eq:odd/even_rational_rs_up_to_nd} with $n, d \in \mathbb{N}$ such that $\frac{n + 1}{d + 1} \in \mathbb{N}$. Then, there exists $\{R_t\}_{t \in \mathbb{N}} \subset \mathfrak{R}_{n, d}\left(\frac{b}{2}, \frac{b}{2}\right)$ such that, setting $O_T := \widetilde{R}_T  \circ \cdots \circ \widetilde{R}_1 $, with $\widetilde{R}_t(X) = R_t(X\mathcal{G}^\dagger)\mathcal{G} $, $O_T(G)$ optimally approximates $\mathcal{G}$ for every $T \in \mathbb{N}$. Moreover, there exists $C = C(a, b, l) > 0$ such that the following doubly exponential convergence rate holds: $$\|\mathcal{G} - O_T(G)\|_2 \leq C |1 - l^b|^{(n + d + 1)^{T}}.$$
\end{theorem}
In practice, $O_T$ is computed using the coupled system presented in \Cref{alg:freon}, with the complete algorithm summarized in \Cref{alg:coupled_chol}.
Unlike prior work \citep{gawlik2021approximating}, wherein the considered class of rational functions does not contain the best approximant for $p \geq 3$, our recursion guarantees the existence of an optimal approximant at every step. Moreover, the condition on $n, d$ permits choosing higher-degree polynomials in the numerator without increasing the denominator's degree, and thus allowing block-QR tricks.

\section{Theoretical and Empirical Observations}


If LMOs do not dictate optimization performance, what does? We build upon \citet{davis2026spectralgradientupdateshelp}, which provides a foundational theoretical lens for understanding spectral descent. To properly diagnose these failures, we must abandon global bounds (as in \cite{cohen2021gradient,islamov2026non}) and examine the exact local expansion of the loss.


\subsection{The Two Quantities}
\label{sec:two_quantities}
Let $f : \R^d \to \mathbb{R}$ be twice differentiable, and suppose we update $X_{k+1} = X_k - \alpha_k \Delta_k$, with $\Delta_k = \langle \widetilde G_k, D_k \rangle D_k$, $\widetilde G_k$ is a stochastic gradient of $f$, and $D_k$ is the direction given by the optimizer, defined for example through preconditioned spectral descent (\Cref{sec:limitations_lmo}). By Taylor's theorem, $f(X_{k+1}) = f(X_k) - \alpha_k \langle G_k, \Delta_k \rangle + \frac{\alpha_k^2}{2} \langle \Delta_k, \nabla^2 f(Z_k) [\Delta_k] \rangle,$ with $Z_k$ on the line segment between $X_k$ and $X_{k+1}$. Plugging in our update $\Delta_k = \langle \widetilde G_k, D_k \rangle D_k$, and factoring inherently isolates two fundamental quantities that govern the trajectory of the optimizer:
\begin{equation} \label{eq:quantities}
    \underbrace{\gamma_k := \frac{\langle G_k, D_k \rangle}{\langle \widetilde G_k, D_k \rangle},}_{\text{Batch Gradient Alignment}} \qquad 
    \underbrace{\Phi_k := \vphantom{\frac{\langle G_k, D_k \rangle}{\langle \widetilde G_k, D_k \rangle}}\frac{\langle \widetilde G_k, D_k \rangle^2}{\langle D_k, \nabla^2 f(Z_k) [D_k] \rangle}}_{\text{Local Directional Descent Potential}}.
\end{equation}

Define in addition $\lambda_k = \langle D_k, \nabla^2 f(Z_k) [D_k] \rangle$, a quantity which will only appear through multiplication with the learning rate $\alpha_k$. Substituting these into the above yields the exact descent condition:
\begin{equation}
f(X_{k+1}) - f(X_k) = -  \Phi_k \left( \gamma_k - \frac{[\alpha_k\lambda_k]}{2}  \right)[\alpha_k\lambda_k].
\label{eq:descent_main}
\end{equation}
\Cref{eq:descent_main} is a quadratic in $\alpha_k$ (assuming $\lambda_k$ is approximately constant), minimised at $\alpha_k^*=\frac{\gamma_k}{\lambda_k}$. At this optimum, \Cref{eq:descent_main} does not depend on $\lambda_k$, and $\Delta f = -\frac{1}{2}\Phi_k\gamma_k^2$. There are a number of observations we should make about this:

\begin{enumerate}[leftmargin=*]
    \item The derivation of \Cref{eq:descent_main} does not utilize anything apart from smoothness. Crucially, neither quantity invokes an LMO or places any restriction on how $D_k$ is constructed. While we cannot directly prove convergence without further assumptions (like \Cref{thm:zh_conv}), it still remains useful for explaining multiple phenomena, as explained below.
    \item This directly recovers the formula of \citet{davis2026spectralgradientupdateshelp} as a special case, by bounding $\lambda_k$ with a Lipschitz constant, setting $\widetilde G_k=G_k$ and $D_k = \lmo(G_k)$, gives $\gamma_k = 1$ with bound $-\|G_k\|_*^2 / 2L_{\|\cdot\|}$. But, compared with \citep{davis2026spectralgradientupdateshelp}, this includes three important generalizations: a non-fixed step-size, non-LMO-based updates, and stochasticity.
    \item While for any fixed norm, the $D_k$ maximizing $\Phi_k\lambda_k$ is $\lmo(\widetilde{G}_k)$, this may not be the direction that maximizes $\gamma_k$ or minimizes $\lambda_k$.
    \item Because the intermediate point $Z_k$ is unknown a priori, $\lambda_k$ cannot be calculated during practical training. Its exact analytical use requires quadratic models (as explored in \Cref{sec:asymptotics}).
\end{enumerate}
Ultimately, optimization performance is driven by the interplay of these quantities ($\gamma_k, \Phi_k$, and the learning rate $\alpha_k\lambda_k$), rather than the strict form of the update direction itself. In the remainder of this paper, we analyze these quantities exactly in the RF model and empirically for a GPT2 training run.

\subsection{Exact Asymptotics in a Random Feature Model}\label{sec:asymptotics}
We take the well-specified random feature regression model as in \cite{davis2026spectralgradientupdateshelp} with $W$ the weights and $A \in \mathbb{R}^{n \times b}$ batched post-activations of the previous layer:
\begin{equation}\label{eq:rf}
    \min_{W \in \mathbb{R}^{o \times n}} f(W) = \frac{1}{2b}\|WA - Y'A\|^2_F.
\end{equation}
We will consider the proportional asymptotic limit $(n, b \rightarrow \infty$ with $n / b \rightarrow \delta)$ and assume the batched post-activations are mean-zero Gaussians, $a_i \sim \mathcal{N}(0, C)$, with a well-defined limit for $C$. Let $G = Y' - W$ be the oracle gradient. Using standard random matrix theory, one can show almost sure convergence of both $\Phi$ and $\gamma,$ and we give full details of this in \Cref{app:LimitingQuantities}. The general equations are hard to analyze, but in the case that $C$ is aligned with the right singular vectors of $G$, i.e., $V^{\top}CV = \Lambda$ for $\Lambda$ a diagonal matrix, we can obtain certain illuminating results.

\begin{restatable}{proposition}{propalignment}
    If the entries of $\Lambda$ are in non-increasing order then $\gamma_{\operatorname{SGD}} \leq \gamma_{\operatorname{Muon}}$ with equality if and only if $\Lambda = cI$ or $\Sigma = cI$.
\end{restatable}

\begin{restatable}{proposition}{propIsotropic}
    If $\Lambda = I$ then $\Phi_{\operatorname{SGD}} \geq \Phi_{\operatorname{Muon}}$ and therefore $\gamma^2_{\operatorname{SGD}}\Phi_{\operatorname{SGD}} \geq \gamma^2_{\operatorname{Muon}}\Phi_{\operatorname{Muon}}$
\end{restatable}

\begin{restatable}{theorem}{thmdescent}\label{thm:RFAsymptotic}
    In the undersampled case ($\delta \rightarrow \infty$) we have
    $
        \underset{\delta \rightarrow \infty}{\lim} \gamma_{\operatorname{SGD}}^2\Phi_{\operatorname{SGD}} \geq \underset{\delta \rightarrow \infty}{\lim} \gamma_{\operatorname{Muon}}^2\Phi_{\operatorname{Muon}}
    $
\end{restatable}

\begin{figure}[t]
  \centering
  \includegraphics[width=\linewidth]{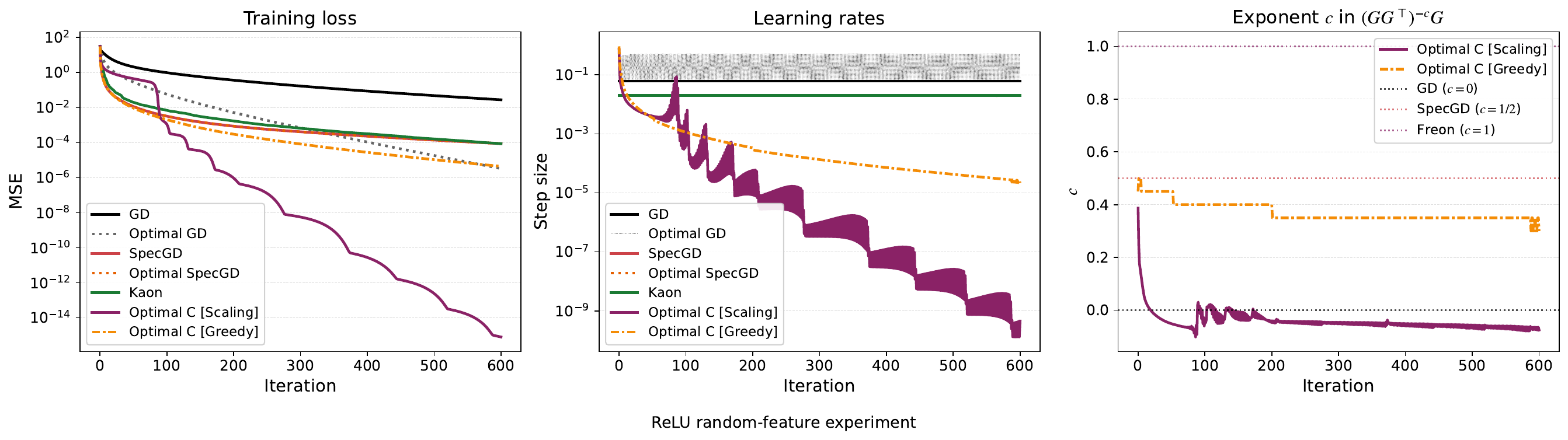}
  \caption{%
    \textbf{Random-feature regression with ReLU activations.}
    We consider the problem in \Cref{eq:rf}, using the setup of \cite{davis2026spectralgradientupdateshelp}.
    \textbf{Left:} training loss for GD, SpecGD, Kaon,
    their optimal-step variants, and two adaptive-exponent methods.
    \textbf{Centre:} effective step sizes; the optimal-GD step size highly oscillates, while optimal-specGD step size is constant, being equal to specGD.
    \textbf{Right:} the exponent $c$ in the update direction $(GG^\top)^{-c}G$
    as tracked by the two Optimal~C methods during training. Further details can be found in \Cref{ap:rf_details}.
  }
  \label{fig:rf-relu}
\end{figure}

\subsection{Empirical Observations}\label{sec:empirics}
\paragraph{The RF optimal step sizes:}
The asymptotic results in \Cref{sec:asymptotics} compare update directions through local descent quantities, but they abstract away a key practical issue of step size stability. Specifically, if we evaluate the optimal per-step step size (i.e., exact line search), the conclusions of \cite{davis2026spectralgradientupdateshelp} completely reverse: standard GD achieves greater loss reduction than spectral descent, a phenomenon also noted by \citet{gonon2026insightsmuonsimplequadratics}. However, optimal GD requires a highly oscillatory and practically unrealizable step-size schedule (\Cref{fig:rf-relu}). Conversely, this exact line-search model reveals the core mechanistic advantage of spectral descent: its optimal step size is constant. Because this step size is proportional to $1/\|DA\|^2_F$, where $D$ is the polar factor of $G$, the value remains practically constant throughout training.
\paragraph{The RF optimal norms:}
Beyond step-size stability, the RF model can be pushed further by dynamically selecting the Schatten norm exponent $c$ that maximizes the local directional potential. We observe, however, that greedily optimizing $c$ at every step yields rapid initial descent but degrades long-term dynamics, even when paired with optimal step sizes (\Cref{fig:rf-relu}). Instead, by assuming a certain power-law decay, we derive a simple, fixed-exponent rule (\Cref{ap:rf_approx}). This approximate choice significantly improves overall performance, overcoming the limitations of spectral descent for mean-zero activations identified by \citet{davis2026spectralgradientupdateshelp}, as demonstrated in \Cref{fig:rf-swiglu}.



\paragraph{The naive RF plug-in for GPT-2:}
Attempting to map the RF model directly onto a GPT-2 architecture exposes distinct limitations, as illustrated by the optimal exponent heatmaps in \Cref{fig:ggn_optimal_c} (Left) and \Cref{fig:freon_layerwise_heatmap_rf}. The RF model fails to provide an accurate quadratic approximation of the layer. Specifically, optimal exponents frequently fall into the strict non-norm regime ($c > 1/2$), indicating that such updates fundamentally deviate from typical RF quadratic landscapes (\Cref{fig:pq_lr_heatmap_rf}).

\paragraph{Single batch vs.~Full dataset geometry in GPT-2:} 

\begin{figure}[t]
    \centering
    \includegraphics[width=\textwidth]{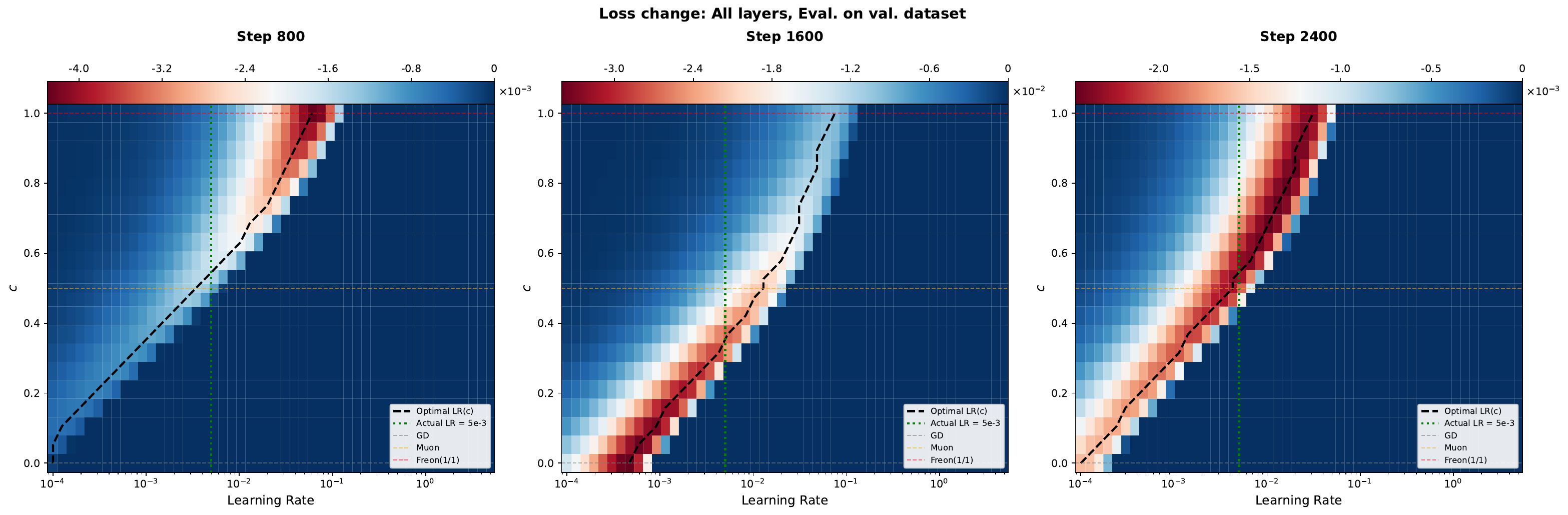}
    \caption{Loss change $\Delta  f$ over the joint $(c, \eta)$ grid, evaluated on the full validation set. Each panel shows a different training point used for the calculation of the gradient. The dashed black curve marks the optimal learning rate per $c$; the green dotted line marks the actual training learning rate.}
    \label{fig:pq_lr_heatmap}
\end{figure}

Given the inapplicability of the RF model, we evaluate both the Generalized Gauss-Newton (GGN) approximation and the exact directional landscape, which qualitatively agree (\Cref{fig:ggn_optimal_c}(ab)). For a single batch, evaluating the per-step decrease confirms that \Freon at $c=1$ yields the optimal local decrease (\Cref{fig:ggn_optimal_c}(c)), consistently across all layers (\Cref{fig:pq_lr_heatmap_perlayer_train}). This shows that the local geometry is severely distorted: favoring $c=1$ is impossible in standard RF models (\Cref{fig:pq_lr_heatmap_rf}). Because local geometry poorly predicts overall performance, we evaluate the global dynamics across the full dataset, revealing a highly stochastic landscape: rather than converging, the optimal exponents fluctuate dynamically over batches during training (\Cref{fig:pq_lr_heatmap}, \Cref{fig:freon_layerwise_range_actual}). Thus, the actual global landscape is significantly more complex than local, single-batch models suggest.
\paragraph{$\Phi$ and $\gamma$ dynamics in GPT-2:}
Tracking $\Phi$ and $\gamma$ throughout training reveals a shifting optimization hierarchy (\Cref{fig:phi_gamma_h3_attn}). Initially, methods like \tsgd exhibit strong local decrease (high $\Phi$), with most \Freon variants similar. At the next snapshot, this inverts: \tsgd and \sgd become heavily suboptimal with larger variance. This mirrors the validation loss trajectories \Cref{fig:val_curves_best_overview}, where the rate of decrease slows down for \sgd variants, but remains good for \freon.

This reveals an interesting empirical observation: gradients tend to agree most in the middle of the singular value spectrum (in the sense of signal-to-noise ratio), rather than at the extremes ($c=0$ or $c=1$). Increasing $c$ to target this high-SNR mid-range (e.g., $c \approx 3/4$) naturally suppresses the mean and variance of $\gamma$, as updates align with these mid-tier rather than maximal singular values.

Overall, these empirical observations clarify the mechanics of \Cref{eq:descent_main}: changing directions of descent from the gradient naturally suppresses $\gamma$, yet this deliberate sacrifice of $\gamma$ enables a massive amplification of $\Phi$ -- the primary driver of loss reduction. Trading $\gamma$ for a boosted $\Phi$ fundamentally explains why the non-standard $c > 1/2$ regime ultimately dominates the optimization landscape.

\section{Conclusion}\label{sec:conclusion}
This work challenges the prevailing geometric narrative of spectral optimizers like \Muon. \Freon reveals optimal updates often violate proper norm regimes, while \Kaon proves randomized spectra perform just as well. What actually drives performance is the interplay between batch gradient alignment and directional descent potential, combined with appropriate choices of step-sizes. We discuss limitations in \Cref{ap:limitations}, leaving a final question: if geometry is not necessary for performance, should we expect theories stemming from it to yield practically useful insights?

\bibliographystyle{plainnat}
\bibliography{references}

\appendix

\crefalias{section}{appendix} 
\crefalias{subsection}{appendix} 
\newpage

\startcontents[sections]
\printcontents[sections]{l}{1}{\setcounter{tocdepth}{2}}
\counterwithin{figure}{section}
\counterwithin{table}{section}

\newpage

\section[NanoGPT Tuning Details for Figures 1 and 2]{NanoGPT Tuning Details for \Cref{fig:nanogpt_lr_sensitivity,fig:nanogpt_dense_val}}
\label{app:nanogpt_tuning_details}

The NanoGPT results in \Cref{fig:nanogpt_lr_sensitivity,fig:nanogpt_dense_val} were run through the same training pipeline on NanoGPT. Every tuning run and learning rate sweep used 8 NVIDIA H100 80GB GPUs.

For each optimizer, we tuned a matrix-update learning rate $\eta_{\mathrm{matrix}}$ and a base learning rate $\eta_{\mathrm{base}}$. The recorded tuning grid was
\begin{center}
\begin{tabular}{llll}
\texttt{(5e-3, 1e-3)} & \texttt{(1e-2, 1e-3)} & \texttt{(2e-2, 1e-3)} & \texttt{(1e-2, 5e-4)} \\
\texttt{(5e-3, 2e-3)} & \texttt{(5e-3, 4e-3)} & \texttt{(5e-3, 8e-3)} & \texttt{(5e-3, 2e-2)} \\
\texttt{(1e-2, 2e-3)} & \texttt{(1e-2, 4e-3)} & \texttt{(1e-2, 8e-3)} & \texttt{(1e-2, 2e-2)} \\
\texttt{(2e-2, 2e-3)} & \texttt{(2e-2, 4e-3)} & \texttt{(2e-2, 8e-3)} & \texttt{(2e-2, 2e-2)}
\end{tabular}
\end{center}
for $(\texttt{matrix\_lr}, \texttt{base\_lr})$.
The selected learning rates used for \Cref{fig:nanogpt_lr_sensitivity,fig:nanogpt_dense_val} are listed in \Cref{tab:nanogpt_tuned_lrs}. \Kaon and both \Freon variants used 5 Newton--Schulz/QDWH-style matrix iterations.

\begin{table}[t]
    \centering
    \caption{Learning rates selected by the seed-42 tuning sweep and used for \Cref{fig:nanogpt_lr_sensitivity,fig:nanogpt_dense_val}.}
    \label{tab:nanogpt_tuned_lrs}
    \begin{tabular}{lccc}
        \toprule
        Optimizer & $\eta_{\mathrm{matrix}}$ & $\eta_{\mathrm{base}}$ & Extra setting \\
        \midrule
        \Muon & $2\cdot 10^{-2}$ & $2\cdot 10^{-2}$ & Polar update \\
        \Kaon & $2\cdot 10^{-2}$ & $8\cdot 10^{-3}$ & 5 matrix iterations \\
        \Freon $c=2/3$ & $2\cdot 10^{-2}$ & $2\cdot 10^{-2}$ & 5 matrix iterations \\
        \Freon $c=3/4$ & $2\cdot 10^{-2}$ & $8\cdot 10^{-3}$ & 5 matrix iterations \\
        \bottomrule
    \end{tabular}
\end{table}

\Cref{fig:nanogpt_lr_sensitivity} then jointly scales the tuned pair as $(\eta_{\mathrm{matrix}},\eta_{\mathrm{base}})=\rho(\eta_{\mathrm{matrix}}^\star,\eta_{\mathrm{base}}^\star)$. The full sensitivity sweep ran 4 optimizers, 7 scale factors, and 3 seeds, for 84 runs in total. The scale factors were $\rho\in\{0.03,0.1,0.3,1,3,10,30,100\}$ and the seeds were $42,43,44$. The plotted version omits $\rho=100$, so the x-axis in \Cref{fig:nanogpt_lr_sensitivity} shows the matrix-update learning rates induced by $\rho\in\{0.03,0.1,0.3,1,3,10,30\}$. Error bars show $\pm 2$ std over the three seeds.

\Cref{fig:nanogpt_dense_val} reruns the selected curve configurations with denser validation logging. These dense-validation jobs used the same tuned learning rates and the same seeds $42,43,44$, with validation every 8 optimizer steps. The dense runs generated Muon at $\rho=1$, \Kaon at $\rho=1$ and $\rho=3$, \Freon $c=2/3$ at $\rho=1$, and \Freon $c=3/4$ at $\rho=1$. The plotted dense-validation figure uses Muon at $\rho=1$, \Kaon at $\rho=3$, and both \Freon variants at $\rho=1$. The plotted lines show averages across three seeds and shaded regions showing $\pm 2$ std.

\section{GPT2 Training Details}
We utilize the existing codebase from \citep{amsel2025polarexpressoptimalmatrix,crawshaw2025explorationnoneuclideangradientdescent}. All experiments train a GPT-2 model (12 layers, 12 attention heads, embedding dimension 768, 124M parameters) with GELU activations and FlashAttention on the WikiText-2 dataset \citep{merity2016wikitext}. Each run uses a context length of $512$ tokens, batch size $8$, gradient norm clipping at $1.0$, and a constant-then-linear decay schedule with a 5\% warm-up. The combined figure \Cref{fig:val_loss_all} reports final validation loss as a function of learning rate, sweeping 8--10 learning rates per optimizer configuration.

\textbf{SGD and TruncatedSGD.} Runs use SGD with momentum $0.9$. Truncation levels of $0, 0.1, 0.5, 1, 2, 5, 10$\% are swept, where $0$\% recovers vanilla SGD. Runs with non-zero truncation rescale the update 2-norm to match that of the untruncated gradient.

\textbf{Muon.} Standard Muon with momentum $0.95$ and 5 Newton--Schulz iterations ($p/q = 1/2$).

\textbf{Kaon.} Same Newton--Schulz backbone as Muon ($5$ steps, momentum $0.95$) but with the chaos iterations.

\textbf{Freon.} FreonRational with momentum $0.95$, 5 Newton--Schulz steps, and the exponent $c = a/b$ swept over ${0, 1/4, 1/3, 1/2, 2/3, 3/4, 1}$, where $c=0$ recovers gradient descent and $c=1/2$ recovers Muon.

\section{Limitations}\label{ap:limitations}
While our analysis provides novel mechanistic insights into the behavior of spectral optimizers, there is a number of theoretical and practical limitations:
\begin{enumerate}[leftmargin=*]
    \item \textbf{Scope of Empirics and Theory:} Our empirical validation is currently restricted to a random feature model and a single architecture modality (a GPT-2 language model). Furthermore, our rigorous analysis of the two core optimization quantities, batch gradient alignment and directional descent potential, is limited to the RF setting. It remains an open question how precisely these dynamics map to other modalities (e.g., vision) or more theoretically amenable non-quadratic loss landscapes.
    \item \textbf{Theoretical Assumptions:} The exact asymptotic results of \Cref{sec:asymptotics} rely strictly on the assumption of mean-zero activations, excluding activations such as ReLU.
    \item \textbf{Lack of Predictive Power:}  Our framework primarily serves as a diagnostic lens. The quantities governing the optimization trajectory are highly stochastic, meaning they can only be reliably evaluated post-factum rather than predicted a priori.
    \item \textbf{Optimal $c$:} Because the optimal local geometry is heavily distorted and highly sensitive to individual batch dynamics, our attempts to dynamically track and utilize the optimal Schatten parameters ($p, q$) step-by-step ultimately failed. The optimal exponents fluctuate too wildly across batches to formulate a stable, greedily adaptive schedule.
    \item \textbf{QDWH Coupled Cholesky} The current implementation of the \Freon algorithm serves primarily to demonstrate the theoretical viability of quasi-norm updates. While the algorithm could certainly be engineered for greater computational efficiency, it is currently unclear whether the extensive engineering effort is even necessary for \Freon, given the similar behaviour across all the optimizers.
\end{enumerate}
\section{Methods Proofs}
\label{app:proofs}

\subsection{Preconditioned Spectral Descent}

\thmpreconditioned*

\begin{proof}
    Let $\|\cdot\|$ be a unitarily invariant matrix norm and $\|\cdot\|_*$ its dual norm. Write the SVD $G_k = U_k \operatorname{diag}(\sigma_k) V_k^{\top}$. Then since $\|\cdot\|$ is unitarily invariant, there is a norm $\|\cdot\|_{\sigma}$ on $\R^r$ symmetric in the coordinates such that $\|G_k\| = \|\sigma_k\|_\sigma$. The first part of the theorem then follows from formulating steepest descent as
    \begin{equation*}
        \begin{aligned}
            X_{k+1} &= {X}_k-\eta\left\|{G}_k\right\|_* \lmo\left({G}_k\right) \\
            &= {X}_k-\alpha_k \langle G_k, D_k\rangle D_k
        \end{aligned}
    \end{equation*}
    with $\alpha_k =\eta$ and $D_k = \lmo(G_k) = U_k\operatorname{lmo}_{\|\cdot\|_\sigma}(\sigma_k)V_k^{\top}$.

    Now consider a preconditioned spectral descent update. Then if it is a steepest descent update, we have as in the above argument that
    \begin{equation}\label{eq:equiv_precond_steep}
        p_k(\sigma_k) = \frac{\eta\|G_k\|_*}{\alpha_k\langle G_k,D_k\rangle} \operatorname{lmo}_{\|\cdot\|_\sigma}(\sigma_k) \propto \operatorname{lmo}_{\|\cdot\|_\sigma}(\sigma_k).
    \end{equation}

    Write $T_k = \operatorname{lmo}_{\|\cdot\|_\sigma}(\sigma_k)$. Suppose $[\sigma_k]_i \leq  [\sigma_k]_j$ and $[T_k]_i > [T_k]_j$. Define $\tilde T_k\in\R^r$ such that $[\tilde T_k]_i = [T_k]_j$, $[\tilde T_k]_j = [T_k]_i$ and $[\tilde T_k]_{\tilde i} = [T_k]_{\tilde i}$ for all $\tilde i \neq i,j$. Then since $\|\cdot\|_\sigma$ is symmetric in the coordinates, $\|\tilde T_k\|_\sigma = \|T_k\|_\sigma = 1$. But we see that $\langle \sigma_k, \tilde T_k\rangle > \langle \sigma_k, T_k\rangle$. This contradicts $T_k = \operatorname{lmo}_{\|\cdot\|_\sigma}(\sigma_k)$. Thus $\operatorname{lmo}_{\|\cdot\|_\sigma}$ preserves the order $\leq$ of the entries of $\sigma_k$, hence so does $p_k$.

\end{proof}

\thmcvgspectral*

\begin{proof}
As $f$ has $L_{\mathcal{X}}$-Lipschitz continuous gradients with respect to $\|\cdot\|$, we can use the generalized descent lemma, which for the total update $\Delta_k = \langle G_k, D_k \rangle D_k$ yields:
\begin{align*}
f(X_{k+1}) &\le f(X_k) - \alpha_k \langle G_k, \Delta_k \rangle + \frac{L_{\mathcal{X}}}{2} \alpha_k^2 \|\Delta_k\|^2 \\
&= f(X_k) - \alpha_k \langle G_k, D_k \rangle^2 + \frac{L_{\mathcal{X}}}{2} \alpha_k^2 \langle G_k, D_k \rangle^2 \|D_k\|^2
\end{align*}

Factoring out the squared inner product isolates the step parameter logic:
\[
f(X_{k+1}) \le f(X_k) - \langle G_k, D_k \rangle^2 \left( \alpha_k - \frac{L_{\mathcal{X}}}{2} \alpha_k^2 \|D_k\|^2 \right)
\]

By the absolute upper bound (Condition 2) and unitary invariance, $\|D_k\| = \|p_k(\sigma_k)\| \le M_k$. Substituting this guarantees the worst-case penalty:
\[
f(X_{k+1}) \le f(X_k) - \langle G_k, D_k \rangle^2 \left( \alpha_k - \frac{L_{\mathcal{X}}}{2} \alpha_k^2 M_k^2 \right)
\]

Substituting the generalized step size $\alpha_k = \frac{\eta}{L_{\mathcal{X}} M_k^2}$:
\begin{align*}
f(X_{k+1}) &\le f(X_k) - \langle G_k, D_k \rangle^2 \left( \frac{\eta}{L_{\mathcal{X}} M_k^2} - \frac{L_{\mathcal{X}}}{2} \frac{\eta^2}{L_{\mathcal{X}}^2 M_k^4} M_k^2 \right) \\
&= f(X_k) - \frac{\eta(1 - \frac{\eta}{2})}{L_{\mathcal{X}} M_k^2} \langle G_k, D_k \rangle^2
\end{align*}

By the scale-invariant lower bound (Condition 1), the inner product is bounded below by the scaled dual norm of the gradient: $\langle G_k, D_k \rangle = \langle \sigma_k, p_k(\sigma_k) \rangle \ge m_k \|\sigma_k\|_{*} = m_k \|G_k\|_{*}$. Squaring this and substituting yields:
\[
f(X_{k+1}) \le f(X_k) - \frac{\eta(1 - \frac{\eta}{2}) m_k^2}{L_{\mathcal{X}} M_k^2} \|G_k\|_{*}^2
\]

Rearranging and summing over $K$ iterations gives a telescoping sum:
\[
\frac{\eta(1 - \frac{\eta}{2})}{L_{\mathcal{X}}} \sum_{k=0}^{K-1} \left( \frac{m_k}{M_k} \right)^2 \|G_k\|_{*}^2 \le f(X_0) - f(X_K) \le f(X_0) - f_{\min}
\]

We can lower-bound the sum on the left by replacing each gradient norm with the minimum gradient norm encountered over the $K$ iterations:
\[
\left[\min_{0 \le k < K} \|G_k\|_{*}^2 \right]\sum_{k=0}^{K-1} \left( \frac{m_k}{M_k} \right)^2 \le \sum_{k=0}^{K-1} \left( \frac{m_k}{M_k} \right)^2 \|G_k\|_{*}^2
\]

Substituting this into the inequality and isolating the minimum gradient norm perfectly recovers the finite-time rate:
\[
\min_{0 \le k < K} \|G_k\|_{*}^2 \le \frac{L_{\mathcal{X}} (f(X_0) - f_{\min})}{\eta(1-\frac{\eta}{2}) \sum_{k=0}^{K-1} \left( \frac{m_k}{M_k} \right)^2}
\]

Because the series $\sum (m_k/M_k)^2$ diverges to infinity as $K \to \infty$, the right-hand side converges to exactly zero. Therefore, $\liminf_{k \to \infty} \|G_k\|_{*} = 0$.
\end{proof}

\begin{theorem}[Convergence of Stochastic Spectral Descent]\label{thm:ap_stoch_conv}
Let the setting be identical to \Cref{thm:zh_conv}. Suppose $p_k(\sigma)$ is generated by a stochastic process over $(k,\sigma) \in \N\times \R^r_{>0}$, and denote by $\mathbb{E}_{k-1}[\cdot]$ the conditional expectation given the history up to step $k-1$. Assume that the random variables $m_k, M_k$ from \Cref{thm:zh_conv} satisfy the following bounds almost surely, $m_k \ge 0$ and $M_k \in (0, \infty)$.

Let $\mu_k = \mathbb{E}_{k-1}[(m_k/M_k)^2]$, and assume $\sum_{k=0}^\infty \mu_k = \infty$ almost surely. Then, for step size chosen as $\alpha_k = \frac{\eta}{L_{\mathcal{X}} M_k^2}$ for $\eta \in (0, 2)$, almost surely, the sequence achieves the finite-time convergence rate:
\[
\min_{0 \le k < K} \|G_k\|_{*}^2 = \mathcal{O}\left( \frac{1}{\sum_{k=0}^{K-1} \mu_k} \right) \qquad \text{and} \qquad \liminf_{k \to \infty} \|G_k\|_{*} = 0.
\]
\end{theorem}

\begin{proof}
First, using Cauchy-Schwarz, we have $\langle \sigma, p_k(\sigma) \rangle \le \|\sigma\|_{*} \|p_k(\sigma)\|$. Substituting our assumed almost-sure bounds into this inequality yields:
\[
m_k \|\sigma\|_{*} \le \langle \sigma, p_k(\sigma) \rangle \le \|\sigma\|_{*} \|p_k(\sigma)\| \le M_k \|\sigma\|_{*}
\]
Dividing by $\|\sigma\|_{*}$ reveals that $m_k \le M_k$ almost surely. Therefore, the sequence of random variables $Z_k = (m_k/M_k)^2$ is non-negative and uniformly bounded: $Z_k \in [0, 1]$.

We now invoke Lévy's Extended Borel-Cantelli Lemma for conditional expectations. For a sequence of uniformly bounded, non-negative random variables $Z_k$, the lemma states that the sum diverges almost surely and that their asymptotic ratio converges to exactly 1 almost surely if and only if the sum of their conditional expectations diverges:
\[
\sum_{k=0}^\infty Z_k = \infty \;\;\text{and}\;\; \lim_{K \to \infty} \frac{\sum_{k=0}^{K-1} Z_k}{\sum_{k=0}^{K-1} \mathbb{E}_{k-1}[Z_k]} = 1 \;\;\text{a.s.}\quad \Longleftrightarrow\quad  \sum_{k=0}^\infty \E_{k-1}[Z_k] = \infty
\]
By assumption, the sum of the conditional expectations diverges: $\sum_{k=0}^\infty \mu_k = \infty$. By the ratio limit above, the actual observed sum of the random variables grows asymptotically at the exact same rate:
\[
\sum_{k=0}^{K-1} \left( \frac{m_k}{M_k} \right)^2 = \Omega\left( \sum_{k=0}^{K-1} \mu_k \right) \quad \text{a.s.}
\]

Since the random mappings $p_k(\sigma)$ almost surely satisfy the deterministic bounds of \Cref{thm:zh_conv} at every step, we can directly substitute this equivalent asymptotic growth into the denominator of the finite-time rate derived in \Cref{thm:zh_conv}:
\[
\min_{0 \le k < K} \|G_k\|_{*}^2 \le \frac{L_{\mathcal{X}} (f(X_0) - f_{\min})}{\eta(1-\frac{\eta}{2}) \Omega\left(\sum_{k=0}^{K-1} \mu_k\right)} = \mathcal{O}\left( \frac{1}{\sum_{k=0}^{K-1} \mu_k} \right)
\]
As $K \to \infty$, the diverging sum in the denominator guarantees that this bound converges to zero, ensuring asymptotic convergence with probability $1$: $\liminf_{k \to \infty} \|G_k\|_{*} = 0$.
\end{proof}

\thmcvgrandom*

\begin{proof}
Because all matrix norms are equivalent in finite dimensions, there exists a universal constant $c > 0$ such that for any vector $v \in \mathbb{R}^r$, $\|v\|_1 \ge c \|v\|_{*}$. The upper bound is trivially constant from normalization: $M_k = 1$. Next, we verify the scale-invariant lower bound. The inner product is:
\[
\langle \sigma_k, p_k(\sigma_k) \rangle = \sum_{i=1}^r [\sigma_k]_i \frac{[E_k]_i}{\|E_k\|} \ge \frac{[E_k]_{\min}}{\|E_k\|} \|\sigma_k\|_1
\]
where $[E_k]_{\min} = \min_{i} [E_k]_i$. Applying the norm equivalence $\|\sigma_k\|_1 \ge c\|\sigma_k\|_{*}$ gives:
\[
\langle \sigma_k, p_k(\sigma_k) \rangle \ge \left( \frac{c[E_k]_{\min}}{\|E_k\|} \right) \|\sigma_k\|_{*}
\]
We set $m_k = \frac{c [E_k]_{\min}}{\|E_k\|}$. Because the vectors $E_k$ are drawn entirely independently of the algorithmic history up to step $k-1$, the conditional expectation equals the unconditional expectation.

Because $[E_k]_{\min}$ is the minimum of a finite number $r$ of positive random variables, $m_k > 0$. Furthermore, since $m_k/M_k$ is bounded, its expectation exists and evaluates to a strictly positive constant $\mu$:
\[
\mu_k = \mathbb{E}_{k-1}\left[ \left(\frac{m_k}{M_k}\right)^2 \right] = \mathbb{E}\left[ \left( \frac{c[E_k]_{\min}}{\|E_k\|} \right)^2 \right] = \mu > 0
\]

Because $\mu_k = \mu$ is constant, the sum of expectations evaluates to $\sum_{k=0}^{K-1} \mu_k = K\mu$, which diverges as $K \to \infty$. The stochastic mapping thus satisfies all conditions of \Cref{thm:ap_stoch_conv}. Substituting $\sum \mu_k = K\mu$ into the generalized rate yields exactly:
\[
\mathcal{O}\left( \frac{1}{\sum_{k=0}^{K-1} \mu_k} \right) = \mathcal{O}\left( \frac{1}{K\mu} \right) = \mathcal{O}\left( \frac{1}{K} \right)
\]
Therefore, setting the step size $\alpha_k = \frac{\eta}{L_{\mathcal{X}}}$ almost surely guarantees the $\mathcal{O}(1/K)$ finite-time rate and asymptotic convergence.
\end{proof}

\subsection{Equivariance of \Freon$(c=1)$}\label{sec:syms_gen}

\begin{restatable}[Equivariance of \Freon$(c=1)$]{theorem}{thminvfreon}\label{thm:symm}For $W\in \R^d$, let $\phi_W$ be an arbitrary function (the neural network) with a layered structure: $W=(W_1,\dots,W_L)$ with $W_l\in \R^{n_l\times m_l}$. Let $f:\R^d\to \R$ be the differentiable loss function which depends on $W$ only through the neural network $\phi_W$, and let $G_l\in \R^{n_l\times m_l}$ be the gradient of $f$ for layer $l$. Then updates of the form $W_l - \eta(G_lG_l^\top)\inv G_l$ are equivariant under symmetries of $\phi_W$ which act on layer matrices by
\begin{equation*}
    W_l \mapsto L_l W_l R_l \quad\text{for}\quad L_l \in \R^*O(n_l),\; R_l \in \R^*O(m_l)
\end{equation*}
where $\R^*O(m_l) = \{\lambda Q: \lambda\neq 0, Q\in O(m_l)\}$. Therefore, the optimization trajectory of $\phi_W$ is invariant under such symmetries.

If $G_l$ has full row (column) rank, the above holds more generally for symmetries with $L_l \in GL_{n_l}(\R)$ ($R_l \in GL_{m_l}(\R)$ respectively).
\end{restatable}

\begin{proof}
        Write such a symmetry $g$. We have
    \begin{equation*}
        f(g\cdot W) = f(W) \quad \text{and}\quad g\cdot (W_1,\dots, W_L) = (\syml{L_1}W_1\symr{R_1},\dots,\syml{L_L}W_L\symr{R_L})
    \end{equation*}
    where $W= (W_1,\dots,W_L)$ and some $\syml{L_l}\in GL_{n_l}(\R)$ and $\symr{R_l}\in GL_{m_l}(\R)$, and we use colours to denote the matrix symmetries.

    For $1\leq l\leq L$, write $G_l[W] := \nabla_{l} f(W)$ where $\nabla_{l}$ denotes the gradient in the $l$\textsuperscript{th} coordinate. By the chain rule we have
    \begin{equation*}
         G_l[W] = \nabla_{l} f(W) = \nabla_{W_l}f(W) = \nabla_{W_l} f(g\cdot W) = \syml{L_l^\top}\nabla_{l}f(g\cdot W)\symr{R_l^\top} = \syml{L_l^\top}G_l[g\cdot W] \symr{R_l^\top}.
    \end{equation*}

    If $\syml{L_l} \in O(n_l)$ and $\symr{R_l} \in O(m_l)$, note that
    \begin{equation}\label{eq:equivar_derivation}
        \begin{aligned}
            (G_l[g\cdot W]G_l[g\cdot W]^\top)\inv G_l[g\cdot W] &= \left(\syml{L_l^{-\top}}G_l[W]\symr{R_l^{-\top} R_l\inv}G_l[W]^\top\syml{L_l\inv}\right)\inv \syml{L_l^{-\top}}G_l[W]\symr{R_l^{-\top}} \\
            &= \syml{L_l} (G_l[W]G_l[W]^\top)\inv G_l[W] \symr{R_l}
        \end{aligned}
    \end{equation}
    where pseudoinverses are understood wherever necessary. The same holds more generally for $\syml{L_l} \in \R^*O(n_l)$ and $\symr{R_l} \in \R^*O(m_l)$, by cancelling out the scalar multiples which commute around all matrices.

    If $G_l[W]$, or equivalently $G_l[g\cdot W]$, has full row rank, then we have a genuine inverse in \eqref{eq:equivar_derivation}, so we can take $\syml{L_l}$ out of the inverse for any $\syml{L_l}\in GL_{n_l}(\R)$, and obtain the same result for such $\syml{L_l}$.

    If $G_l[W]$ has full column rank, we can apply the same trick on the other side for $\symr{R_l}\in GL_{m_l}(\R)$ by first using the identity

    \begin{equation*}
        (G_l[g\cdot W]G_l[g\cdot W]^\top)\inv G_l[g\cdot W] = G_l[g\cdot W](G_l[g\cdot W]^\top G_l[g\cdot W])\inv.
    \end{equation*}
    
    Thus, under such symmetries, the updates are equivariant in weight space, and hence they are invariant in output space.
\end{proof}
\begin{nb}
    These encapsulate ReLU scaling symmetries, node permutation symmetries, attention linear symmetries, etc, and most particularly transformations that are symmetries only on sampled data, even if the parameterization itself does not admit it globally.
\end{nb}

\paragraph{Comparison with Newton's method.} The preconditioner $GG^\top$ to the update step in \Freon$(1/1)$ can seem reminiscent to the Fisher information matrix used as preconditioner in Newton's method. It is well-known that Newton's steps are equivariant under \emph{all} linear symmetries (not just layerwise ones). Here, we briefly clarify the connection and distinction between the two optimizers.

Consider a general deep network $\phi_W:\R^{n_{\text{in}}} \to \R^{n_{\text{out}}}$, with parameters $W$. As before, $f:\R^d\to \R$ is the loss function.

In the following we abstract away the multivalued data and consider $\phi_W$ implicitly evaluated solely at one input $x\in \R^{n_{\text{in}}}$, the general case follows by summing or averaging over such inputs. The loss $f$ is viewed as a function of $\phi_W$.

Define the quantities $g_\phi := \nabla_W \phi_W \in \R^{n_{\text{out}}}$, $h_{f} := \nabla_{\phi_W}\nabla_{\phi_W}^\top f \in \R^{n_{\text{out}}\times n_{\text{out}}}$, $g_{f} := \nabla_W f\in \R^d$, $J_{\phi} := \nabla_W \phi_W \in \R^{n_{\text{out}}\times d}$, $G_l :=\nabla_{W_l} f \in \R^{n_l\times m_l}$, $J_l := \nabla_{W_l} \phi_W \in \R^{n_{\text{out}}\times (n_l\times m_l)}$.

By the chain rule $g_\phi = J_\phi^\top g_{f}$. For the usual $l_2$ or softmax-cross entropy losses, we have the probabilistic interpretation $f(W)= -\log p(y^* \mid x,W)$ where $y^*\in\R^{n_{\text{out}}}$ is the label associated to $x$. The Fisher information matrix (or GGN) is given by
\begin{equation*}
    \E_{p(y\mid x,W)}[-\nabla_W\nabla_W^\top \log p(y\mid x, W)] = J_\phi^\top h_{\phi} J_\phi.
\end{equation*}
A common approximation is the empirical Fisher \citep{kunstner_limitations_2019}:
\begin{equation}\label{eq:precond_newton}
    \begin{aligned}
        \E_{p(y\mid x,W)}[-\nabla_W\nabla_W^\top \log p(y\mid x, W)] &= \E_{p(y\mid x,W)}[\nabla_W \log p(y\mid x, W)\nabla_W \log p(y\mid x, W)^\top] \\
        &\approx \nabla_W \log p(y^*\mid x, W)\nabla_W \log p(y^*\mid x, W)^\top \\
        &= g_f g_f^\top \\
        &= J_\phi^\top g_f g_f^\top J_\phi
    \end{aligned}
\end{equation}
This is the preconditioner used in Newton's method.

Now by the chain rule $G_l = \mat(J_l^\top g_f)$, where the $\mat$ operator makes the output a matrix of the appropriate shape. The preconditioner in \Freon$(1/1)$ is
\begin{equation}\label{eq:precond_freon}
    G_lG_l^\top = \mat(J_l^\top g_f) \mat(g_f^\top J_l).
\end{equation}

Compare \eqref{eq:precond_freon} with \eqref{eq:precond_newton}.

\section{From Polar to H\"older Express}

\subsection[Fundamental Properties of the a/b-Method]{Fundamental Properties of the $a/b$-Method}\label{sec:optimality}
We aim in this section to generalise the results in  \cite{amsel2025polarexpressoptimalmatrix} to our setup. To this end, let $r, p \in \mathbb{N}$ and consider the following subsets of the polynomial ring $\mathbb{P}$
\begin{align}
& \mathbb{P}^{\mathrm{odd}}(r) := \mathrm{span}\{x^{2ir + 1}: i \in \mathbb{N}\}\label{eq:odd_pol_q},\\
&\mathbb{P}^{\mathrm{even}}(r) := \mathrm{span}\{x^{2ir}: i \in \mathbb{N}\}\label{eq:even_pol_q},\\
& \mathbb{P}^{\mathrm{odd}}_p(r) := \{P \in \mathbb{P}^{\mathrm{odd}}(r): \deg(P) \leq 2rp + 1 \}\label{eq:odd_pol_q_up_to d},\\
& \mathbb{P}^{\mathrm{even}}_p(r) := \{P \in \mathbb{P}^{\mathrm{even}}(r): \deg(P) \leq 2rp \}\label{eq:even_pol_q_up_to d}.
\end{align}
Moreover, for any $n, d \in \mathbb{N}$, we also define the following sets of rational functions
\begin{align}
& \mathfrak{R}(r, s) := \left\{\frac{N}{D}: N \in \mathbb{P}^{\mathrm{odd}}(r), D \in \mathbb{P}^{\mathrm{even}}(s),\, D > 0\right\}\label{eq:odd/even_rational_rs},\\
& \mathfrak{R}_{n, d}(r, s) := \left\{\frac{N}{D}: N \in \mathbb{P}_n^{odd}(r), D \in \mathbb{P}_d^{even}(s),\, D > 0\right\}\label{eq:odd/even_rational_rs_up_to_nd}.
\end{align}

\subsubsection{Existence and Characterization of the Optimal Approximants.} The set $\mathfrak{R}(r, s)$ considered above can be understood as a particular subset of generalized rational functions of the form
\[
R(x) = \frac{\sum_{i=0}^n a_i g_i(x)}{\sum_{j=0}^d b_j h_j(x)}, \qquad \sum_{j=0}^d b_j h_j(x) > 0,
\]
where $g_i(x) = x^{2ri + 1}$ and $h_j(x) = x^{2sj}$. For these objects, unlike the classical polynomial case, the existence of an optimal approximant is not guaranteed in general. Furthermore, additional issues arise from the fact that the numerator and denominator of the approximants $R$ can have a common factor, which leads to simplifications that might break the constraint $R \in \mathfrak{R}_{n, d}(r)$. This phenomenon is studied in the literature through the so called deficiency index $\mathfrak{d}$ which strictly depends on the class of functions we want to approximate. In this regards, let us define the following class of functions
\begin{definition}\label{def:d=0}
Let $a, b \in \R$. We define the set of $0$-deficiency continuous functions $C_{\mathfrak{d}=0}^{0}([a, b]) \subset C^0([a, b])$ as
\[
C_{\mathfrak{d}=0}^{0}([a, b]) := \{f \in C^0([a, b]): \mathfrak{d}_f = 0\}.
\]
\end{definition}

We can prove the following result.
\begin{lemma}\label{lem:existence_of_optima}
Let $r \in \mathbb{N}$ and consider $\mathfrak{R}(r, r)$ as in \eqref{eq:odd/even_rational_rs_up_to_nd}. Moreover, let $0 < a \leq b \in \R$ and define
\[
\mathscr{C}^0([a, b]) := \left\{f \in C^0([a, b]): \frac{f(x^\frac{1}{2r})}{x^\frac{1}{2r}} \in C_{\mathfrak{d}=0}^0([a, b])\right\}.
\]
Then, for every $f \in \mathscr{C}^0([a, b])$, there exists $R^* \in \mathfrak{R}_{n, d}(r, r)$ such that
\[
\|f - R^*\|_{C^0([a, b])} \leq \|f - R\|_{C^0([a, b])}, \qquad \forall R \in \mathfrak{R}_{n, d}(r).
\]
\end{lemma}
\begin{proof}
Let us observe that any $R \in \mathfrak{R}_{n, d}(r, r)$ can be rewritten as
\begin{equation}\label{eq:rewriting_of_R}
R(x) = \frac{\sum_{i=0}^n a_i x^{2ri + 1}}{\sum_{j=0}^d b_j x^{2rj}} = x \frac{\sum_{i=0}^n a_i x^{2ri}}{\sum_{j=0}^d b_j x^{2rj}} =: x \frac{P(x^{2r})}{Q(x^{2r})}
\end{equation}
where $P \in \mathbb{P}_n$ and $Q \in \mathbb{P}_d$. By \cite{achieser}, we know that for every $g \in C^0([a, b])$, there always exist $P^* \in \mathbb{P}_n$, $Q^* \in \mathbb{P}_d$ such that $\frac{P^*}{Q^*}$ best approximates $g$. Then, since $0 < a \leq b$, considering $f \in C^0([a, b])$ and setting
\begin{equation}\label{eq:g_tilde}
\widetilde{g}(x) = \frac{f\left(x^\frac{1}{2r}\right)}{x^\frac{1}{2r}},
\end{equation}
we get a best approximation candidate $R^*(x) = x \frac{P^*(x)}{Q^*(x)}$ for $f$.

To conclude, by Definition \ref{def:d=0}, when $f \in \mathscr{C}^0([a, b])$, we have $\mathfrak{d}_{\widetilde{g}} = 0$ and we can ensure $R^* \in \mathfrak{R}_{n, d}(r, r)$.
\end{proof}

Finally, mimicking the polynomial case, we are interested in providing a characterization of the optimal approximants. The following holds.
\begin{lemma}\label{lem:characterization_of_optima}
Let us assume the same setup as in Lemma \ref{lem:existence_of_optima}. Then, for every $f \in \mathscr{C}^0([a, b])$, any $R^* \in \mathfrak{R}_{n, d}(r, r)$ such that
\begin{equation}\label{eq:haar_condition_generalized}
\#\{x \in [a, b]: |f(x) - R^*(x)| = \|f - R^*\|_{C^0}\} \geq n + d + 2
\end{equation}
is an optimal approximant of $f$.
\end{lemma}
\begin{proof}
Given $f \in \mathscr{C}^0([a, b])$, exploiting \eqref{eq:rewriting_of_R}, any optimal $R^* \in \mathfrak{R}_{n, d}(r, r)$ fulfils the same properties that $P^*/Q^*$ satisfies approximating $\widetilde g$. Then, by \cite[Theorem 1]{Loeb1966ApproximationBG}, we have the thesis.
\end{proof}

\subsubsection{Optimality}
We will now address the problem of the optimality of the . To start with, the following closure property holds.
\begin{lemma}\label{lem:p_odd_q_is_closed}
Let $r \in \mathbb{N}$. Then, $\mathbb{P}^{\mathrm{odd}}(r)$, $\mathbb{P}^{\mathrm{even}}(r)$ and $\mathfrak{R}(r, r)$ are closed with respect to composition.
\end{lemma}
\begin{proof}
Let us consider $\mathbb{P}^{\mathrm{odd}}(r)$. By induction on the monomials, it suffices to show that
\[
p = p_1 \circ p_2 \in \mathbb{P}^{\mathrm{odd}}(r)
\]
for $p_1 = x^{2lr + 1}$, $p_2 = x^{2nr + 1} + x^{2mr + 1}$. Recalling now Newton's binomial formula, it holds
\[
\begin{aligned}
p(x) & = \sum_{i = 0}^{2lr + 1} \left(x^{2nr + 1}\right)^{i} \left(x^{2mb + 1}\right)^{2lb + 1 - i}\\
& = \sum_{i = 0}^{2lr + 1} x^{2nri + i + 2m(2lr + 1 - i)r + k - i}\\
& = \sum_{i = 0}^{2lr + 1} x^{2\left[nri + m(2lr + 1 - i) + 2l\right]r + 1}.
\end{aligned}
\]
Then, we conclude by observing that, since $nri + m(2lr + 1 - i) + 2l \in \mathbb{N}$, every term in the above expansion is in $\mathbb{P}^{\mathrm{odd}}(r)$.  Exploiting an analogous argument as above, we can also prove the following
\begin{enumerate}
    \item if $p_1$, $p_2 \in \mathbb{P}^{\mathrm{even}}(r)$, then $p \in \mathbb{P}^{\mathrm{even}}(r)$ (which implies the closure of $\mathbb{P}^{\mathrm{even}}(r)$ as claimed);
    \item if $p_1 \in \mathbb{P}^{\mathrm{odd}}(r)$, $p_2 \in \mathbb{P}^{\mathrm{even}}(r)$ (or vice versa), then $p \in \mathbb{P}^{\mathrm{even}}(r)$.
\end{enumerate}

Let us consider $q_1$, $q_2 \in \mathfrak{R}_{n, d}(r, r)$. We have
\[
\begin{aligned}
D_1^{2nr + 1}(x)(N_2 \circ q_1)(x) & = \sum_{i = 1}^{n} c_i N_1^{2ir + 1}(x) D_1^{2r(n - i)}(x),\\
D_1^{2dr}(x)(D_2 \circ q_1)(x) & = \sum_{i = 1}^{d} e_i N_1^{2ir}(x) D_1^{2r(d - i)}(x).
\end{aligned}
\]
The arguments above implies then that $N_1^{2ir + 1} \in \mathbb{P}^{\mathrm{odd}}(r)$ while $D_1^{2r(n - i)}$, $N_1^{2ir}$, $D_1^{2r(d - i)} \in \mathbb{P}^{\mathrm{even}}(r)$. Finally, recalling that the product of even polynomials is even and the product of an even polynomial and an odd one is odd, since only $r$-powers monomial appears, we immediately get $N_2 \circ q_1 \in \mathfrak{R}(r, r)$ and $D_1^{2dr}(D_2 \circ q_1) \in \mathbb{P}^{\mathrm{even}}(r)$, so that
\[
(q_1 \circ q_2)(x) = \frac{(N_2 \circ q_1)(x)}{D_1^{2(d - n)r - 1}(x)(D_2 \circ q_1)(x)} \in \mathfrak{R}_{n, d}(r, r),
\]
for all $n$, $d \in \mathbb{N}$.
\end{proof}

We can then generalise \cite[Theorem 3.1]{amsel2025polarexpressoptimalmatrix} with the following optimality result.
\begin{theorem}\label{thm:optimality}
Let $l_1 = l, u_1 = u \in (0, 1)$, $r, T \in \mathbb{N}$, $n, d \in \mathbb{N}$ and consider
\[
R_{t + 1} = \underset{R \in \mathfrak{R}_{n, d}(r, r)}{\mathrm{argmin}} \max_{x \in [l_t, u_t]} |1 - R(x)|, \qquad \forall t \in [1, T],
\]
with $l_{t + 1} = \min_{x \in [l_t, u_t]} R_t(x)$, $u_{t + 1} = \max_{x \in [l_t, u_t]} R_t(x)$. Then, $R^* := R_T \circ R_{T - 1} \circ \cdots \circ R_1$ is optimal and
\[
\max_{x \in [l_T, u_T]} |1 - R^*(x)| = 1 - l_T.
\]
Moreover,
\begin{equation}\label{eq:error_bounds}
\left\{
\begin{aligned}
& l_{t + 1} = R_t(l_t),\\
& u_{t + 1} = 2 - R_t(l_t),\\
& \max_{x \in [l_t, u_t]} |1 - R_t(x)| = 1 - l_t,
\end{aligned}
\right.
\qquad \forall t \in [1, T].
\end{equation}
\end{theorem}
\begin{proof}
Let us first observe that, by \cite{wall1948analytic}, we have that the Hankel determinants of $x \mapsto x^\alpha$, $\alpha \not\in \mathbb{N}$, are all non-zero. Thus, by \cite[Corollary 2]{gragg}, $f(x) \equiv 1 \in \mathscr{C}^0([a, b])$, so that, thanks to Lemma \ref{lem:existence_of_optima}, $R_t$ (and consequently $R^*$) is well defined.

We will now adapt the argument used in \cite[Theorem 3.1]{amsel2025polarexpressoptimalmatrix}, and split our discussion into two parts: first, we prove \eqref{eq:error_bounds} by showing that $l_t$ is a minimizer of $R_t$; second, we argue on the optimality of $R^*$.

\paragraph{$l_t$ is a minimizer of $R_t$.} By direct computation, we have
\[
\begin{aligned}
R_t'(x) & = \frac{N'(x) D(x) - D'(x) P(x)}{D^2(x)}\\
& = \frac{\left(\sum_{i=0}^n a_i(2ri + 1)x^{2ri}\right)\left(\sum_{j=0}^d b_jx^{2rj}\right) - \left(\sum_{j=0}^d 2rjb_j x^{2rj - 1}\right)\left(\sum_{i=0}^n a_i x^{2ri + 1}\right)}{D^2(x)}\\
& = \frac{\sum_{i=0}^n \sum_{j=0}^d (2r(i - j) + 1)a_ib_jx^{2r(i + j)}}{D^2(x)} = \frac{\widetilde{N}(x^{2r})}{\widetilde{D}(x^{2r})}
\end{aligned}
\]
where $\widetilde{N} \in \mathbb{P}_n$, $\widetilde{D} \in \mathbb{P}_d$. Thus, $R_t$ has at most $2(n + d)$ extremal points, of which, due to the symmetry of $\widetilde{N}(x^{2r})$, at most $n + d$ can be contained in $[l_t, u_t] \subset [0, 1]$. Exploiting now Lemma \ref{lem:characterization_of_optima}, we know that, being $R_t$ an optimal approximant, there exist at least $n + d + 2$ extremal points for $|1 - R_t|$ on $[l_t, u_t]$. In particular, we can conclude that the limiting points $l_t$ and $u_t$ need to be extremal points.

Since $R_t(0) = 0$, being $R_t$ optimal, we have $R_t > 0$ on $[l_t, u_t]$ otherwise, $R_t = 0$ would be a better approximant. In particular, there exists $\hat{x} \in [0, l_t]$ such that $R'_t(\hat{x}) > 0$. If we suppose now that $l_t$ is a maximum of $R_t$, then $R'_t(l_t) \leq 0$ and, by the intermediate value theorem, there must exist $\overline{x} \in [\hat{x}, l_t]$ such that $R_t(\overline{x}) = 0$. This is, however, a contradiction since all the nonnegative extremal points of $R_t$ are in $[l_t, u_t]$ so that $l_t$ must be a minimum of $R_t$. Finally, by direct computation, \eqref{eq:error_bounds} immediately follows.

\paragraph{Optimality of $R^*$.}  We will argue by induction. Let us observe that Lemma \ref{lem:existence_of_optima} and the previous paragraph ensure the base step. Then, let us assume that the thesis holds for $T - 1$ and consider $f: [l, u] \to \R$ given by $x \mapsto \overline{R}_{T - 1} \circ \cdots \circ \overline{R}_1(x)$, for some $\overline{R}_1, ... \overline{R}_{T - 1} \in \mathfrak{R}_{n, d}(r, r)$. Moreover, letting $[a, b]$ be the image of $[l, u]$ through $f$, for any constant $c$ we immediately have
\begin{equation}\label{eq:max_on_c}
\max_{x \in [l, u]} |1-cf(x)| = \max\{1 - ca, cb - 1\}.
\end{equation}
Let us suppose now that $\frac{a}{b} > \frac{l_t}{u_t}$ then the following holds
\[
\begin{aligned}
\frac{a}{b} > \frac{l_t}{2 - l_t} \quad \Leftrightarrow \quad 1 - l_t > \frac{b - a}{a + b}.
\end{aligned}
\]
Then, since by the previous discussion $l_t = \min_{x \in [l, u]} R_{T - 1} \circ \cdots \circ R_1(x)$, by choosing $\overline{c} = \frac{2}{a + b}$ in \eqref{eq:max_on_c}, the above implies
\[
\max_{x \in [l, u]} |1 - R_{T - 1} \circ \cdots \circ R_1(x)| > \max_{x \in [l, u]} |1 - \overline{c}f(x)|
\]
contradicting the inductive hypothesis. Finally, since $\frac{a}{b} < \frac{l_t}{u_t}$, for any $\overline{R}_T \in \mathfrak{R}_{n, d}(r, r)$, we have
\[
\begin{aligned}
\max_{x \in [l, u]} |1 - \overline{R}_T(f(x))| & \geq \min_{R \in \mathfrak{R}_{n, d}(r, r)} \max_{x \in [a, b]} |1 - R(x)|\\
&\overset{(*)}{=} \min_{R \in \mathfrak{R}_{n, d}(r, r)} \max_{x \in [a/b, 1]} |1 - R(x)|\\
& \geq \min_{R \in \mathfrak{R}_{n, d}(r, r)} \max_{x \in [l_t/u_t, 1]} |1 - R(x)|\\
& = \min_{R \in \mathfrak{R}_{n, d}(r, r)} \max_{x \in [l_t, u_t]} |1 - R(x)|\\
& = \min_{R \in \mathfrak{R}_{n, d}(r, r)} \max_{x \in [l, u]} |1 - R(R_{T - 1} \circ \cdots \circ R_1(x))|\\
& = \max_{x \in [l, u]} |1 - R^*(x)|
\end{aligned}
\]
where $(*)$ is justified by the invariance of $\mathfrak{R}_{n, d}(r, r)$ under rescaling.
\end{proof}

\begin{remark}\label{rem:best_polynomial_approximant}
We emphasize that, since $\mathbb{P}_n(r) \cong \mathfrak{R}_{n, 0}(r, r)$, the same conclusions of Theorem \ref{thm:optimality} holds if we replace $\mathfrak{R}_{n, d}(r, r)$ with $\mathbb{P}_n(r)$.
\end{remark}

Finally, we are ready for the following
\begin{proof}[Proof of Theorem \ref{thm:best_approximation_fractional} - Optimality]
Let us consider $R_t \in \mathfrak{R}(r, r)$ and define $\widetilde{R}_t(X) := R_t(X \mathcal{G}^\dagger)\mathcal{G}$. It is immediate to verify that, if $R_t$ is an optimal approximant of $f \equiv \mathrm{Id}$ according to Theorem \ref{thm:optimality}, then $\widetilde{R}_t$ is a best approximant of $\mathcal{G}$. More generally, considering $R^*$ as in Theorem \ref{thm:optimality}, let us observe that
\[
\begin{aligned}
R^*(x) = \widetilde{R}_T(\widetilde{R}_{T - 1}(\cdots \widetilde{R}_1(x)))\mathcal{G}^\dagger.
\end{aligned}
\]
Then, it holds $O_T(X) = R^*(X\mathcal{G}^\dagger )\mathcal{G}$, which implies that $O_T(G)$ is the best approximant of $\mathcal{G}$ for every $T \in \mathbb{N}$.
\end{proof}

\subsubsection{Approximation Errors}
We will now provide an error estimate for the proposed method.

To start with, let us point out that, calling ${_2F}_1$ the generalized hypergeometric function given by
\[
{_2F}_1(\alpha, \beta, \gamma, \xi) = \sum_{n = 0}^\infty \frac{(\alpha)_n(\beta)_n}{n!(\gamma)_n}\xi^n,
\]
with $(\alpha)_n := \alpha (\alpha + 1) \cdots (\alpha + n - 1)$, by \cite[Chapter 15, 15.1.8]{AbramowitzStegun1964}, it holds
\[
(1 - x)^{-\frac{1}{2r}} = {_2F}_1 \left(\frac{1}{2r}, 1, 1, x\right).
\]
Then, we immediately get the following generalisation of \cite[Theorem 3]{kenney_pade}.
\begin{lemma}\label{lem:approx_ploynomials}
Let $r \in \mathbb{N}$ and consider $P(x) \in \mathbb{P}_n^{even}(r)$, $Q(x) \in \mathbb{P}_d^{even}(r)$ such that $h(x) := \frac{P(x)}{Q(x)}$ is the $[n/d]$ Padé-approximant of $(1 - x)^{-\frac{1}{2r}}$. Then, for every $t \in \mathbb{N}$, setting
\[
x_{t + 1} = x_t h(1 - x_t^{2r}),
\]
we have
\[
|1 - x_t^{2r}| \leq |1 - x_0^{2r}|^{(n + d + 1)^{\top}}.
\]
\end{lemma}

Finally, the following approximation result holds (concluding the proof of Theorem \ref{thm:best_approximation_fractional}).
\begin{theorem}\label{thm:approximation_error}
Let $T \geq 0$, $l \in (0, 1)$, $\sigma_{\min}(G) \geq l$ and consider $R^*$ as in Theorem \ref{thm:best_approximation_fractional}. Then, $X_T$ given by \eqref{eq:coupled_freon} satisfies the following estimate bound
\[
\|(GG^\top)^{-\frac{a}{b}}G - O_T\|_2 \leq \frac{|1 - l^b|^{(n + d + 1)^{T}}}{l^{\frac{2a - b}{b}}}.
\]
\end{theorem}
\begin{proof}
First, let us observe that, setting $Z_t := (GG^\top)^{\frac{a}{b} - \frac{1}{2}} O_t$, $t \in [1, T]$, it holds
\[
\begin{aligned}
l^{\frac{2a - b}{b}} \left\|(GG^\top)^{-\frac{a}{b}}G - O_T\right\|_2 & \leq \sigma_{min}(G)^{\frac{2a - b}{b}} \left\|(GG^\top)^{-\frac{a}{b}}G - O_T\right\|_2\\
& \leq \left\|(GG^\top)^{\frac{a}{b} - \frac{1}{2}}\left((GG^\top)^{-\frac{a}{b}}G - O_T\right)\right\|_2\\
& = \|\mathrm{sign}(G) - Z_t\|_2
\end{aligned}
\]
It is fundamental to observe that, by direct computation, we have $X_t \in \mathfrak{R}\left(\frac{b}{2}, \frac{b}{2}\right)$.

Let now $h$ be the $[n/d]$ Padé-approximant of $(1 - x)^{-\frac{1}{b}}$ and $\widetilde{R}(x) = x h(1 - x^{b})$. Moreover, by direct computation, we can immediately check that $h(1 - x^{2r}) = \frac{P(x^{b})}{Q(x^{b})}$ for some $P \in \mathbb{P}_n$, $Q \in \mathbb{P}_d$. Thus, $\widetilde{R} \in \mathfrak{R}_{n, d}(\frac{b}{2}, \frac{b}{2})$, and, by Lemma \ref{lem:p_odd_q_is_closed}, it is well defined
\[
f := \widetilde{R} \; \underbrace{\circ \cdots \circ }_{\text{$T$ times}} \;\widetilde{R} \in \mathfrak{R}(\frac{b}{2}, \frac{b}{2}).
\]
Finally, considering $R^*$ as in Theorem \ref{thm:optimality}, by Lemma \ref{lem:approx_ploynomials}, we have
\[
\begin{aligned}
\|\mathrm{sign}(G) - Z_T\|_2 & \leq \max_{x \in [l, 1]} |1 - R^\star(x)| \leq \max_{x \in [l, 1]} |1 - f(x)|\\
& \leq \max_{x \in [l, 1]} \frac{|1 - f(x)^b|^{(n + 1)^T}}{\sum_{i = 0}^{b - 1} f(x)^i} \leq \max_{x \in [l, 1]} \frac{|1 - x^b|^{(n + d + 1)^T}}{\sum_{i = 0}^{b - 1} f(x)^i}\\
& \leq |1 - l^b|^{(n + d + 1)^T},
\end{aligned}
\]
and we can conclude.
\end{proof}

\begin{remark}\label{rem:polynomial_error_bound}
Let us point out that using a polynomial approximant the above estimate reduces to
\[
\|(GG^\top)^{-\frac{a}{b}}G - O_{T}\|_2 \leq \frac{|1 - l^{b}|^{(n + 1)^{T}}}{l^{\frac{2a - b}{b}}}.
\]
\end{remark}

\subsection{Accuracy of approximations}
\begin{figure*}[htbp]
  \centering
  \begin{tabular}{ccc}
    \includegraphics[width=0.3\textwidth]{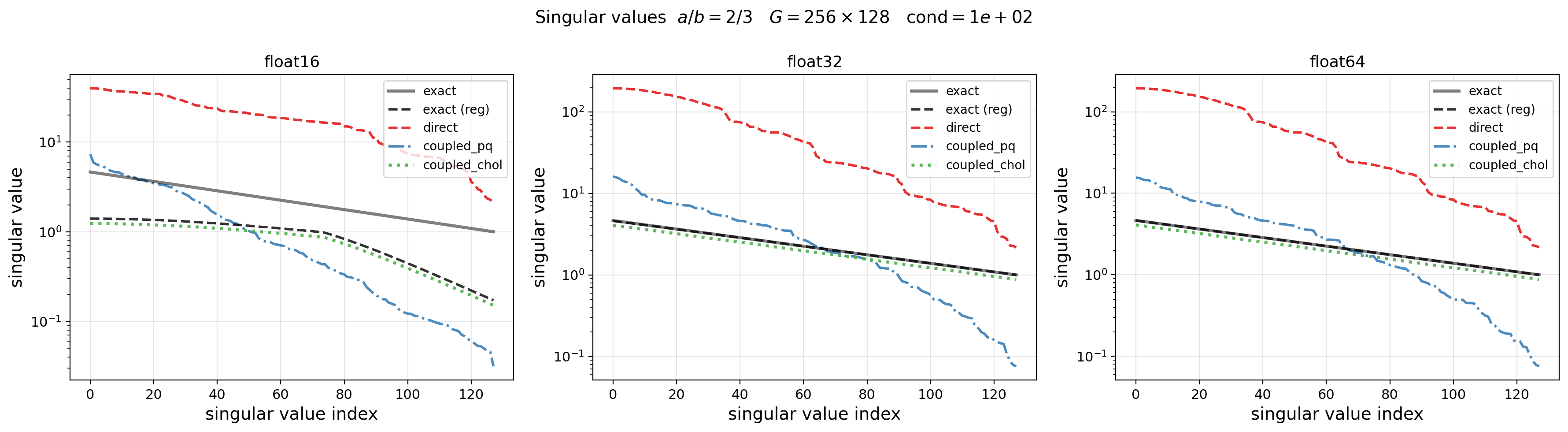} &
    \includegraphics[width=0.3\textwidth]{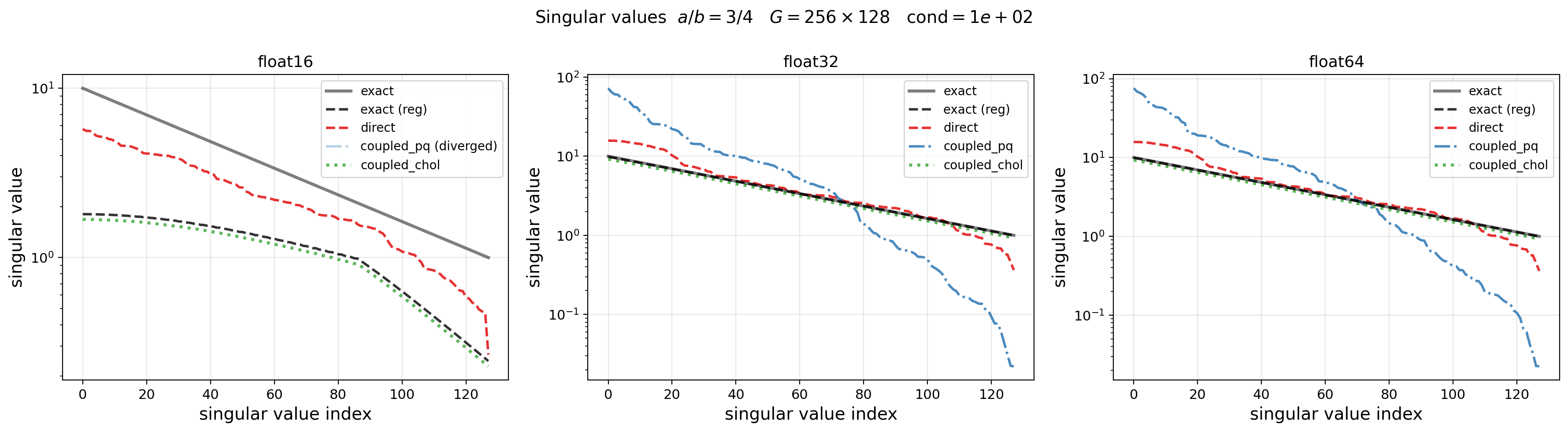} &
    \includegraphics[width=0.3\textwidth]{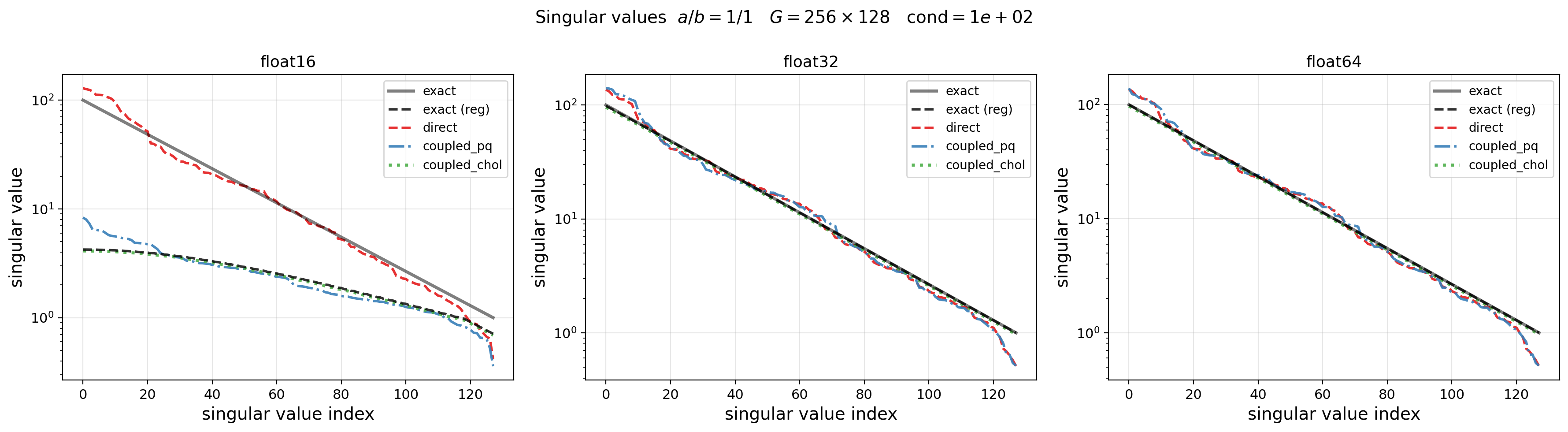} \\[4pt]
    \includegraphics[width=0.3\textwidth]{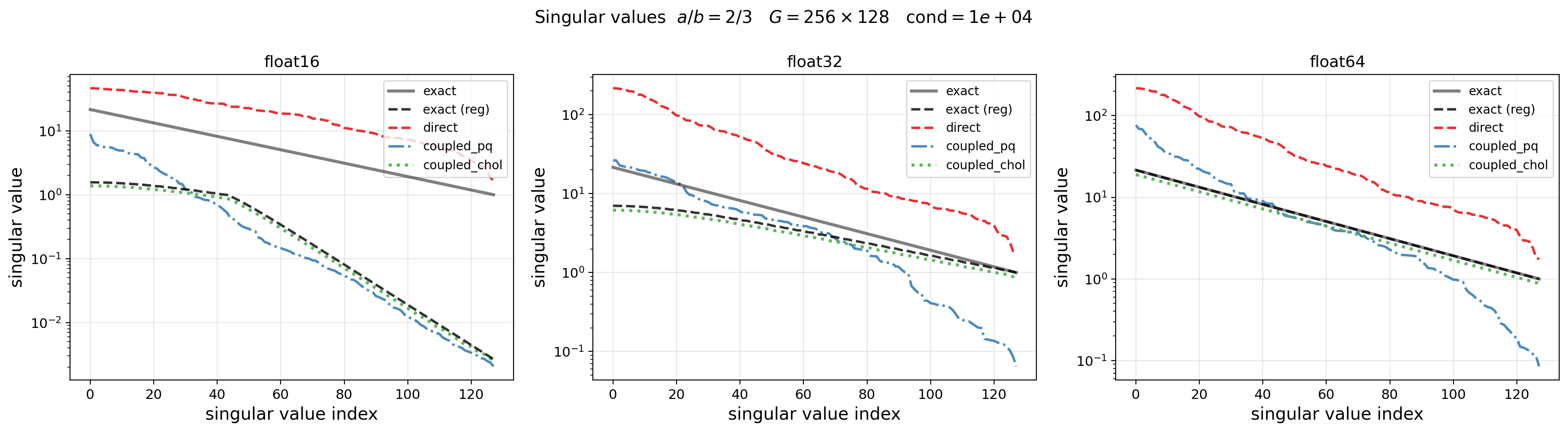} &
    \includegraphics[width=0.3\textwidth]{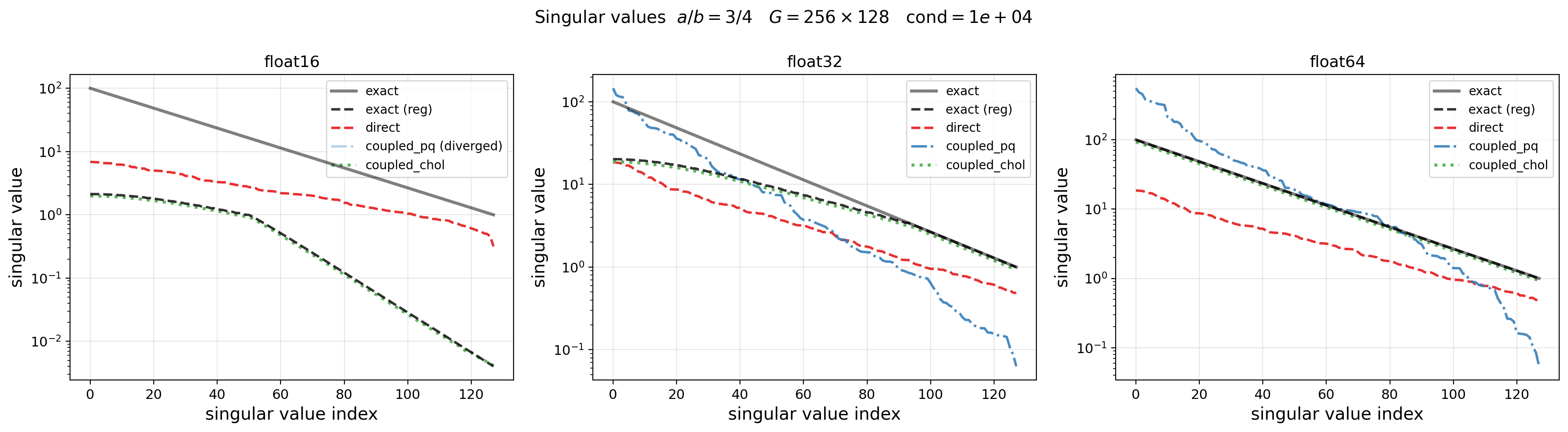} &
    \includegraphics[width=0.3\textwidth]{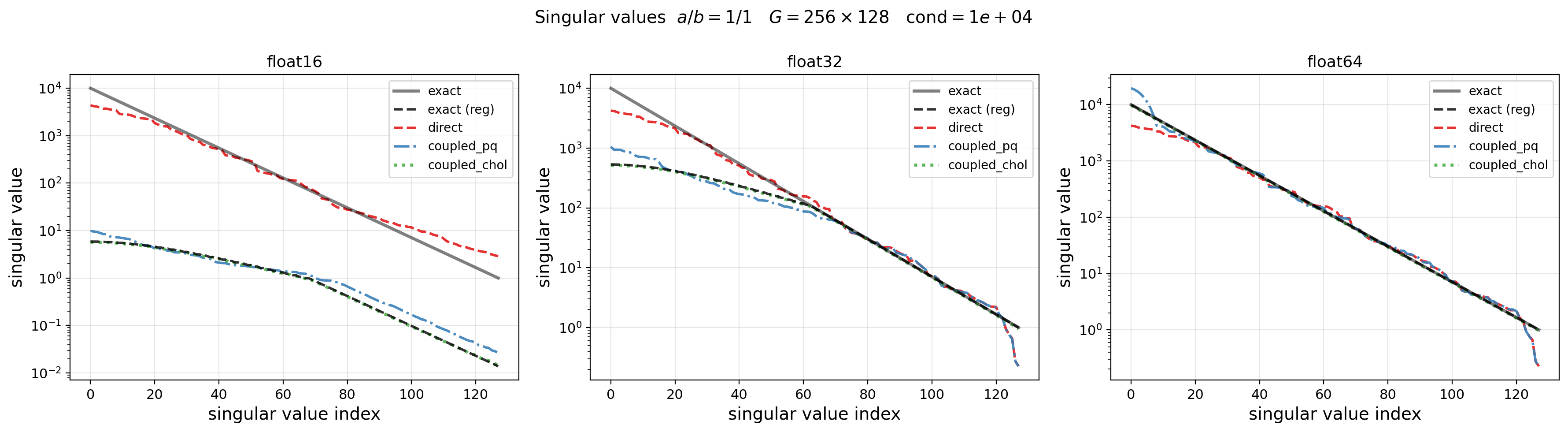} \\[4pt]
    \includegraphics[width=0.3\textwidth]{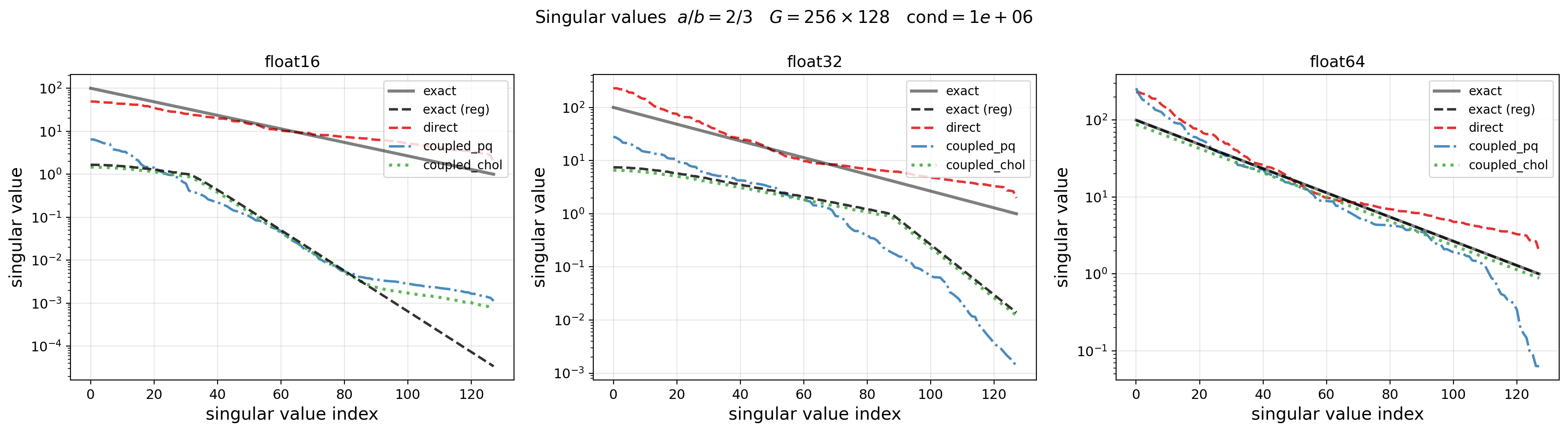} &
    \includegraphics[width=0.3\textwidth]{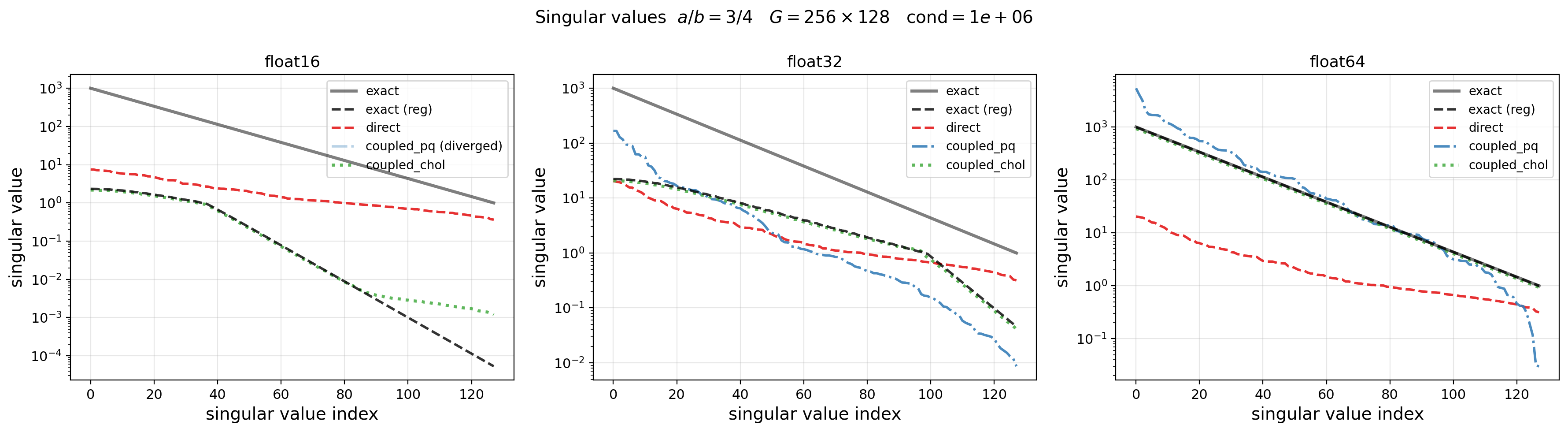} &
    \includegraphics[width=0.3\textwidth]{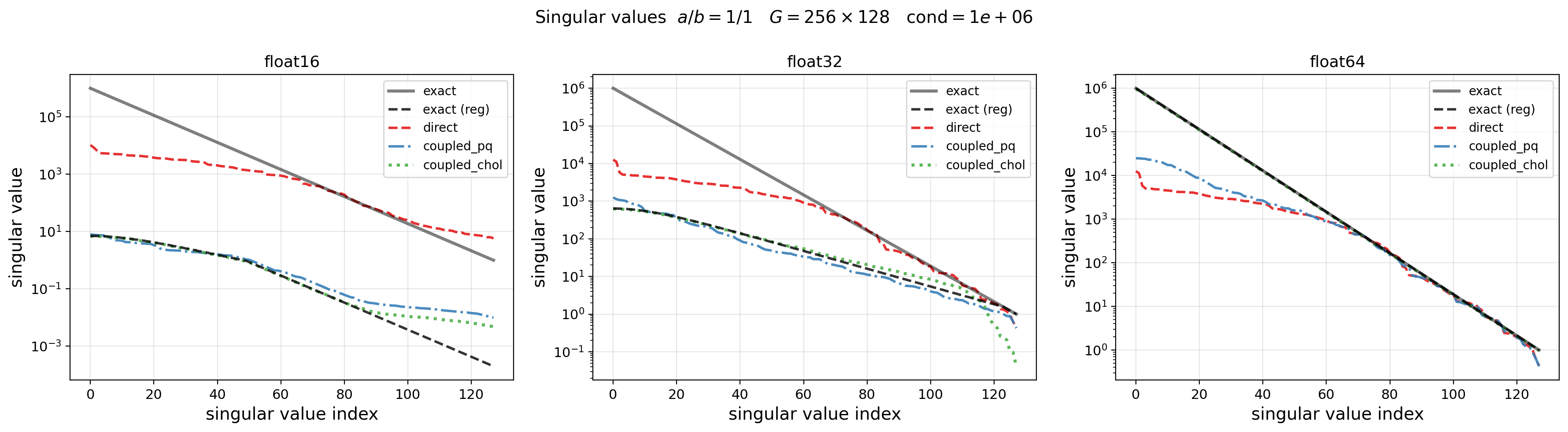} \\
  \end{tabular}
  \caption{Output singular values of three methods for computing $(GG^\top)^{-a/b}G$
  on a synthetic $256\times128$ matrix $G = U\Sigma V^\top$, where the singular
  values of $\Sigma$ are spaced logarithmically between $1$ and $\kappa^{-1}$
  to give condition number $\kappa$.
  Columns correspond to exponents $a/b \in \{2/3,\, 3/4,\, 1\}$;
  rows to condition numbers $\kappa \in \{10^2,\, 10^4,\, 10^6\}$.
  Each panel shows three sub-panels for float16, float32, and float64 (left to right).
  Within each sub-panel the solid black line is the exact target
  $\sigma_i(G)^{1-2a/b}$ (sorted descending by index),
  the dashed black line is the regularized target
  $\sigma_i(G)\,(\sigma_i(G)^2 + \varepsilon)^{-a/b}$
  where $\varepsilon$ is the dtype-dependent regularization threshold scaled
  by the spectral norm of $G$,
  and the coloured curves are the output singular values of
  \texttt{direct} (red, polynomial Newton--Schulz applied directly to $G$),
  \texttt{coupled\_pq} (blue, coupled polynomial iteration tracking $G$ and its Gram matrix),
  and \texttt{coupled\_chol} (green, rational Newton--Schulz via Cholesky tracking).
  All methods use 10 iterations.
  A method that diverges (NaN/Inf output) is shown as a flat phantom line in the legend only.
  \texttt{coupled\_chol} remains stable across all conditions and precisions,
  while \texttt{direct} and \texttt{coupled\_pq} degrade or diverge at high condition
  numbers in low precision.}
  \label{fig:sv_comparison}
\end{figure*}

\subsection{Instability of Polynomial Iterations}\label{app:holder_instability}

\Cref{fig:heatmaps} compares three iterative methods for computing $(GG^\top)^{-a/b}G$ across matrix sizes, condition numbers, and floating-point precisions. Grey cells indicate divergence (NaN or Inf output). Both \texttt{direct} and \texttt{coupled\_pq} exhibit systematic divergence at high condition numbers, particularly in \texttt{float16} and \texttt{bfloat16}, with divergence spreading to \texttt{float32} at extreme conditioning. Regularization delays but does not eliminate this behaviour. In contrast, \texttt{coupled\_chol}, remains numerically stable across all tested conditions and dtypes, achieving near-machine-precision accuracy throughout. This stability comes without any regularization: the Cholesky structure enforces positive definiteness at every step, preventing the eigenvalue collapse that causes the other methods to diverge.
\begin{figure*}[htbp]
\centering
\includegraphics[width=0.32\textwidth]{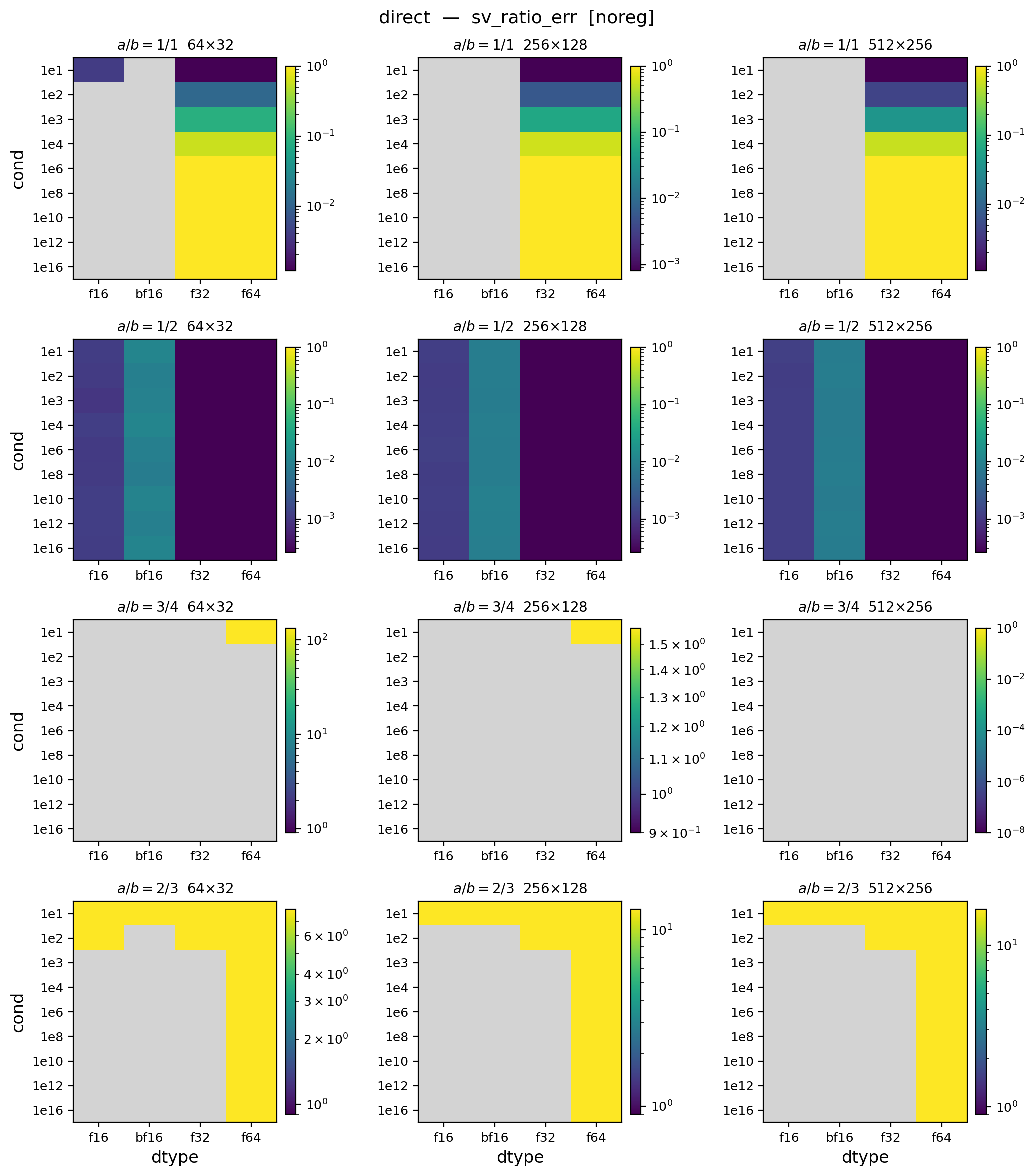}\hfill
\includegraphics[width=0.32\textwidth]{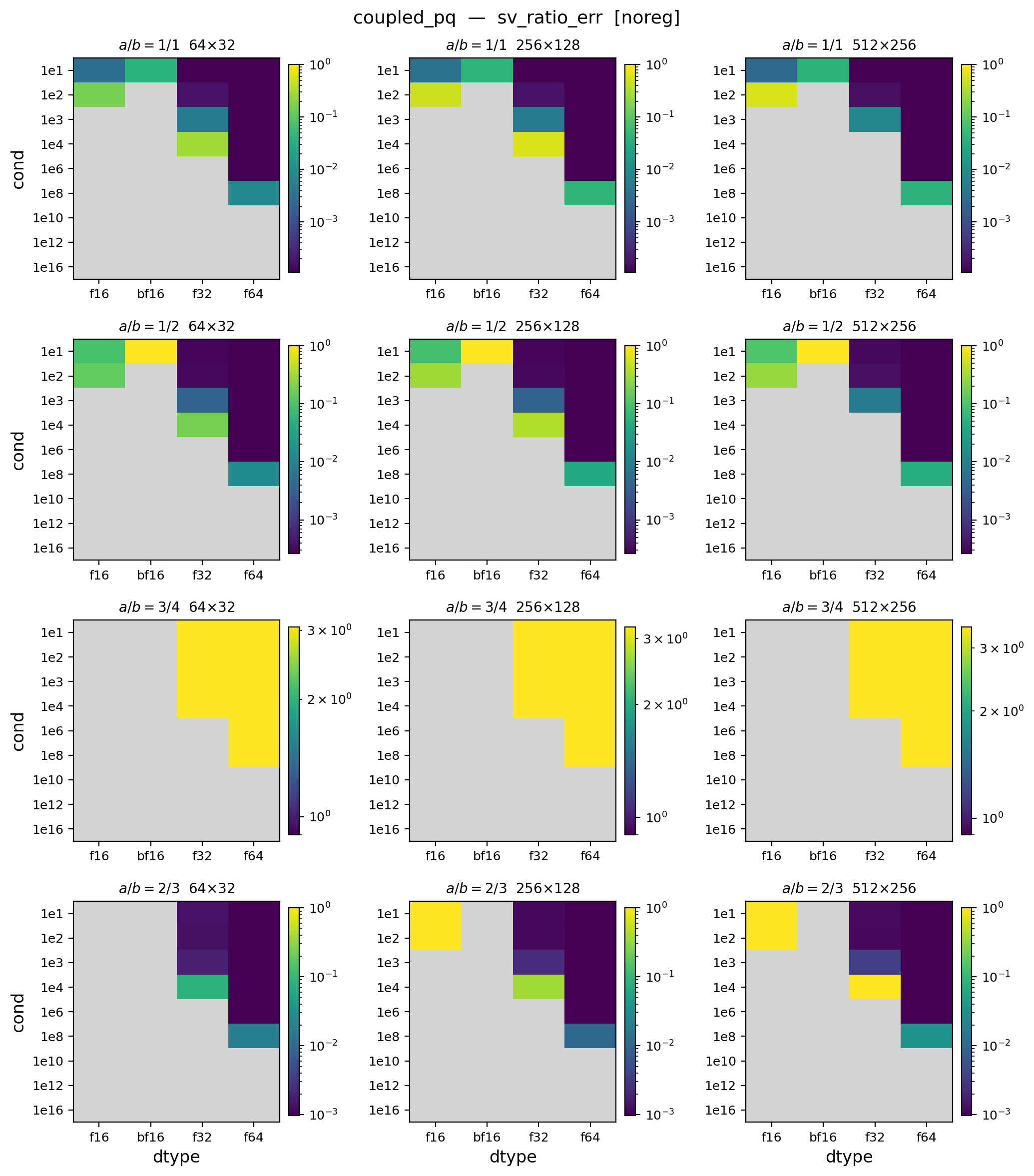}\hfill
\includegraphics[width=0.32\textwidth]{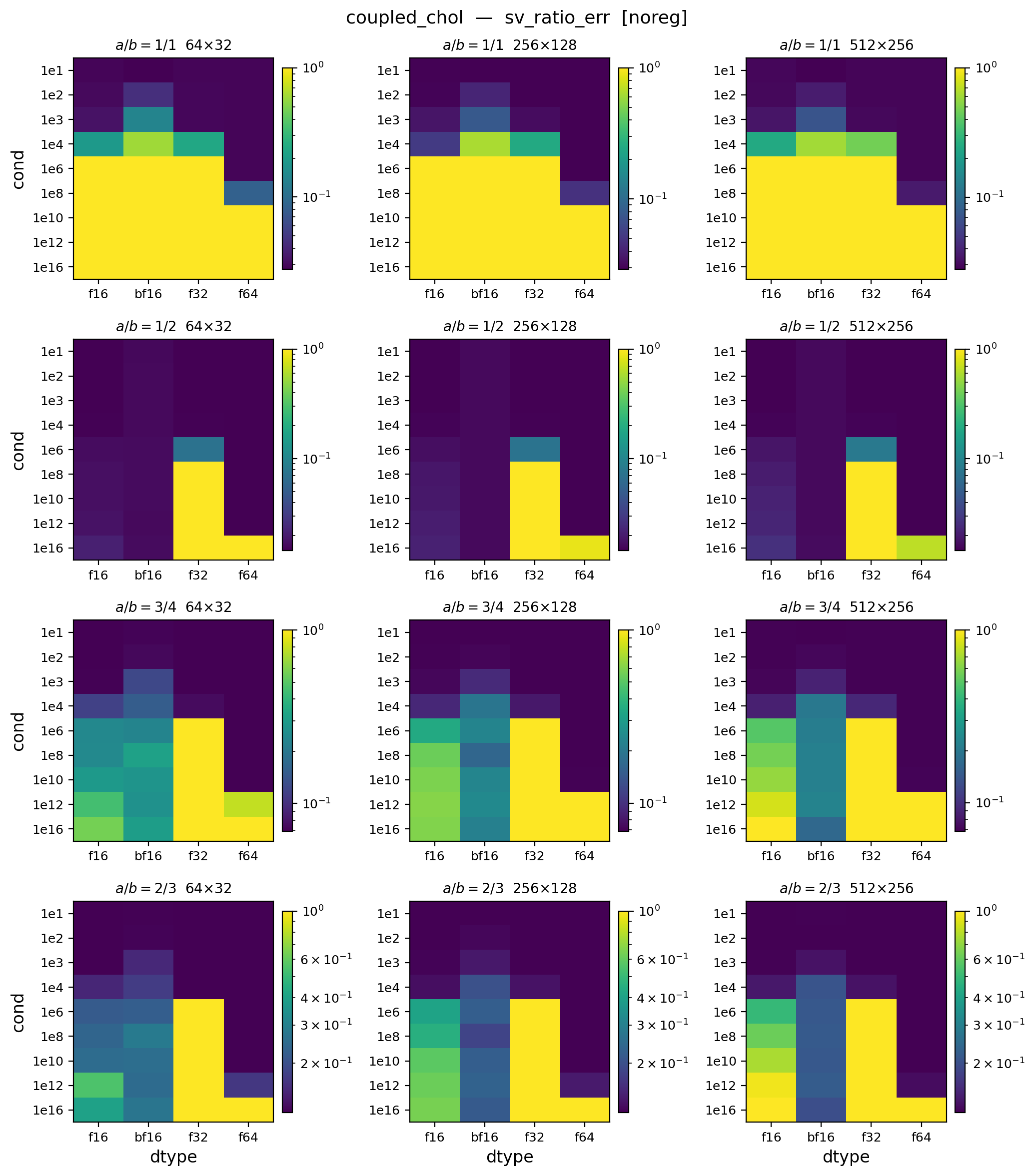}
\caption{Singular-value ratio error $\varepsilon_{\mathrm{sv}} = \max_i|\sigma_i(\hat Z)\cdot\sigma_i(G)^{(2a-b)/b}-1|$ without regularization for three iterative methods computing $(GG^\top)^{-a/b}G$: \texttt{direct} (left), \texttt{coupled\_pq} (centre), and \texttt{coupled\_chol} (right). Each panel is a $4\times3$ grid of heatmaps indexed by exponent $a/b\in{1/1,1/2,3/4,2/3}$ (rows) and matrix size $(m,n)\in{64\times32,256\times128,512\times256}$ (columns); within each cell the $y$-axis encodes condition number $\kappa\in{10^1,10^2,\dots,10^{12},10^{16}}$ (increasing downward) and the $x$-axis encodes floating-point type (f16, bf16, f32, f64, increasing precision left to right). Test matrices are constructed as $G=U\Sigma V^\top$, where $U\in\mathbb{R}^{m\times m}$ and $V\in\mathbb{R}^{n\times n}$ are independent Haar-distributed orthogonal matrices obtained by QR-decomposing Gaussian random matrices in \texttt{float64}, and $\Sigma=\operatorname{diag}(\sigma_1,\dots,\sigma_k)$ with $k=\min(m,n)$ singular values spaced logarithmically from $1$ to $\kappa^{-1}$; each matrix is then cast to the target dtype before the iterative method is applied. The error metric equals zero for a perfect computation and is invariant to the overall scale of $G$; no regularization is applied, so grey cells indicate that the method produced non-finite values for that configuration. \texttt{direct} and \texttt{coupled\_pq} each perform 40 steps of a degree-$(p,q)$ polynomial iteration whose coefficients are precomputed by Zolotarev-optimal minimax approximation on $[\ell,1]$ with $\ell=10^{-16}$, run entirely in the target dtype. \texttt{coupled\_chol} performs 25 steps of a rational Remez iteration with Cholesky-based stabilization. Colour encodes error on a logarithmic scale (viridis); darker shades indicate smaller error. Error grows with condition number and worsens at reduced precision, while \texttt{coupled\_pq} achieves consistently lower error than \texttt{direct} across all settings. Most notably \texttt{coupled\_chol} remains stable at all condition numbers and precisions.}
\label{fig:heatmaps}
\end{figure*}

%
%

%
%

%
%

\subsection{Iterative Methods for $(GG^\top)^{-a/b}\,G$}

Throughout this section, let $G \in \mathbb{R}^{m \times n}$ ($m \le n$) denote the (Frobenius-normalised) gradient matrix. Let $P_t$ be a step-varying degree-2 polynomial with Remez-approximated coefficients, and let $k \in \mathbb{Z}^+$ satisfy $P(x)^k x^{ka/b} \approx 1$ on the fitting interval. All methods target the fixed point $X_\infty = (GG^\top)^{-a/b}\,G$.

\textbf{A critical note on numerical stability:} It is important to state upfront that, as identified by Higham~\cite{higham2008functions}, and numerically illustrated in \Cref{app:holder_instability}, most of the polynomial and coupled iterations detailed below fail to converge in finite-precision arithmetic due to severe numerical stability issues (e.g., condition-number squaring and diverging coupled sequences). In practice, \emph{only the final method presented} (the \textbf{Coupled-Chol Iteration}) exhibits the forward stability required to reliably converge. The preceding methods are documented primarily for their theoretical connections and historical context.

For cost accounting, we distinguish \emph{G-MMs} (full $m \times n$ gradient products, $O(m^2 n)$) from \emph{S-MMs} ($m \times m$ square products, $O(m^3)$) and \emph{QRs} (thin QR of a $2m \times m$ block, $O(m^3)$). Integer powers $A^k$ require $\lceil \log_2 k\rceil$ S-MMs via binary exponentiation.


\subsubsection{Direct Iteration}

The direct method iterates $O_t$ without maintaining separate state for the polynomial argument $A_t$. Instead, $A_t$ is recomputed at each step from the current iterate. Three cases arise depending on $a/b$:

\paragraph{General case ($a/b \notin \{1/2, 1\}$).}
A frozen Gram power $(GG^\top)^{k(a/b-1/2)}$ is precomputed once; each step forms:
\begin{equation}\label{eq:direct_general}
\begin{cases}
  O_{t+1} = R_t(A_t)\,O_t, & O_0 = G, \\[4pt]
  A_t = (O_t O_t^\top)^{k/2}\,(GG^\top)^{k(a/b\,-\,1/2)}.
\end{cases}
\end{equation}

\paragraph{Polar case ($a/b = 1/2$).}
The frozen factor vanishes ($k(a/b-1/2)=0$) and $k/2 = 1$, simplifying the argument to the plain Gram matrix of the current iterate:
\begin{equation}
\begin{cases}
  O_{t+1} = R_t(A_t)\,O_t, & O_0 = G, \\[4pt]
  A_t = O_t O_t^\top.
\end{cases}
\end{equation}

\paragraph{Freon case ($a/b = 1$, $k=1$).}
The general formula \eqref{eq:direct_general} would require an expensive matrix square root $(X_t X_t^\top)^{1/2}$ at every step. Instead, $A_t$ is defined via a mixed product with the \emph{frozen} initial transpose $G^\top$:
\begin{equation}\label{eq:coupled_freon}
\begin{cases}
  O_{t+1} = R_t(A_t)\,O_t, & O_0 = G, \\[4pt]
  A_t = O_t\,G_0^\top.
\end{cases}
\end{equation}
Using $G_0^\top$ avoids both the matrix square root and condition-number squaring, since $\kappa(O_t G^\top) \le \kappa(O_t)\kappa(G)$ compared to $\kappa(O_t O_t^\top) = \kappa(O_t)^2$.

\medskip
In all three cases, convergence to $O_\infty O_\infty^\top = (GG^\top)^{1-2a/b}$ ensures $A_\infty = I$ (the fixed point of $P$).

\paragraph{Classical connection.} For $a/b = 1/2$ and degree-1 $P_t$, the update reduces to the Newton--Schulz iteration for the polar decomposition~\cite{higham2008functions}. The higher power inside the polynomial allows for general $a/b \ne 1/2$.

\paragraph{Complexity.}
\begin{itemize} \itemsep0em
  \item \textit{Init:} 1 G-MM; $\lceil \log_2(k(a/b-1/2))\rceil$ S-MMs (zero if $a/b = 1/2$).
  \item \textit{Per step:} 2 G-MMs; $(1 + \lceil \log_2(k/2)\rceil + \mathbf{1}_{[a/b \ne 1/2]})$ S-MMs.
\end{itemize}

\subsubsection{M-Accumulator Iteration}

Instead of iterating $O_t$ directly, this method maintains a preconditioner $M_t \in \mathbb{R}^{m \times m}$, applied to the frozen gradient $G$ only upon convergence:
\begin{equation}
\begin{cases}
  M_{t+1} = R_t(A_t)\,M_t, & M_0 = I, \\[4pt]
  A_t = M_t^k\,(GG^\top)^{ka/b}.
\end{cases}
\end{equation}
The final result is $O_T = M_T G$. At convergence, $M_\infty = (GG^\top)^{-a/b}$ and $A_\infty = I$.

\paragraph{Classical connection.} For $a/b = 1/2$, this is the uncoupled polynomial analogue of the Denman--Beavers iteration for the matrix square root~\cite{higham2008functions}. This is also the method of \citep{GramNewtonSchulz}.

\paragraph{Complexity.}
\begin{itemize} \itemsep0em
  \item \textit{Init:} 1 G-MM; $\lceil \log_2(ka/b)\rceil$ S-MMs.
  \item \textit{Per step:} 0 G-MMs; $(\lceil \log_2 k\rceil + 3)$ S-MMs.
  \item \textit{Final:} 1 G-MM ($M_T G_0$).
\end{itemize}
M-accumulator is the cheapest per-step method when $n \gg m$, deferring the large-matrix product to the end.

\subsubsection{Coupled Iteration}

The coupled method eliminates the per-step recomputation of $A_t$ by tracking it as its own state. Because $P_t(A_t)$ commutes with $A_t$, the system closes exactly:
\begin{equation}
\begin{cases}
  O_{t+1} = R_t(A_t)\,O_t, & O_0 = G, \\[4pt]
  A_{t+1} = R_t(A_t)^k\,A_t, & A_0 = (GG^\top)^{ka/b}.
\end{cases}
\end{equation}

\paragraph{Classical connection.} Tracking the argument matrix alongside the iterand echoes Schulz-type coupled iterations for matrix inversion~\cite{higham2008functions}.

\paragraph{Complexity.}
\begin{itemize} \itemsep0em
  \item \textit{Init:} 1 G-MM; $\lceil \log_2(ka/b)\rceil$ S-MMs.
  \item \textit{Per step:} 1 G-MM; $(\lceil \log_2 k\rceil + 2)$ S-MMs.
\end{itemize}
Compared to M-accumulator, this adds one G-MM per step but saves one S-MM, making it favorable only when $m \approx n$.

\subsubsection{Coupled-$ab$ Iteration}

A variant of coupled iteration where $A_t$ initializes at the raw Gram matrix $GG^\top$, avoiding the fractional matrix power during initialization:
\begin{equation}
\begin{cases}
  O_{t+1} = R_t(A_t)^a\,O_t, & O_0 = G, \\[4pt]
  A_{t+1} = R_t(A_t)^b\,A_t, & A_0 = GG^\top.
\end{cases}
\end{equation}
The invariant $R(x)^b x \approx 1$ ensures $A_t \to I$ while $O_t \to (GG^\top)^{-a/b} G$.

\paragraph{Complexity.}
\begin{itemize} \itemsep0em
  \item \textit{Init:} 1 G-MM; 0 S-MMs.
  \item \textit{Per step:} 1 G-MM; $(2 + \lceil \log_2 a\rceil + \lceil \log_2 b\rceil)$ S-MMs.
\end{itemize}

\subsubsection{Dual-$ab$ Iteration}

This method maintains a factored representation $A_t = X_t Y_t$ to avoid forming $GG^\top$ explicitly at each step:
\begin{equation}
\begin{cases}
  A_t = O_t Y_t, \quad R_t = R_t(A_t), \\[4pt]
  O_{t+1} = R_t^a\,O_t, & O_0 = G, \\[4pt]
  Y_{t+1} = Y_t\,R_t^{b-a}, & Y_0 = G^\top.
\end{cases}
\end{equation}
A balancing step $O_t \leftarrow \gamma_t O_t$, $Y_t \leftarrow Y_t / \gamma_t$ with $\gamma_t = \sqrt{\|Y_t\|_F / \|X_t\|_F}$ prevents mismatched sequence growth.

\paragraph{Classical connection.} The factored form averts condition-number squaring, mirroring product formulas for evaluating matrix functions of products~\cite{higham2008functions}.

\paragraph{Complexity.}
\begin{itemize} \itemsep0em
  \item \textit{Init:} 0 G-MMs (transpose only).
  \item \textit{Per step:} 3 G-MMs; $(1 + \lceil\log_2 a\rceil + \lceil\log_2(b-a)\rceil)$ S-MMs.
\end{itemize}

\subsubsection{Coupled-Dual Iteration}

This method runs two coupled masters in parallel for exponents $a/b$ and $1 - a/b$. A cross-stabilizer $E_t = I - X_t Y_t \to 0$ corrects the slave iterates directly, avoiding fractional matrix powers inside the loop:
\begin{equation}
\begin{cases}
  R_A = R_t(A_t), \quad R_B = R_t(B_t), \quad E_t = I - O_t Y_t, \\[4pt]
  O_{t+1} = R_A\,O_t + \gamma\,E_t O_t, & O_0 = G, \\[4pt]
  Y_{t+1} = Y_t\,R_B + \gamma\,Y_t E_t, & Y_0 = G^\top, \\[4pt]
  A_{t+1} = R_A^k\,A_t, & A_0 = (GG^\top)^{ka/b}, \\[4pt]
  B_{t+1} = R_B^k\,B_t, & B_0 = (GG^\top)^{k(b-a)/b},
\end{cases}
\end{equation}
where $\gamma > 0$ is a scalar hyperparameter.

\paragraph{Complexity.}
\begin{itemize} \itemsep0em
  \item \textit{Init:} 1 G-MM; $\lceil\log_2(ka/b)\rceil + \lceil\log_2(k(b-a)/b)\rceil$ S-MMs.
  \item \textit{Per step:} 3 G-MMs; $(4 + 2\lceil\log_2 k\rceil)$ S-MMs.
\end{itemize}

\subsubsection{Rational-Chol (QDWH-style) Iteration}

Instead of polynomials, this method applies a rational function $R_t(x) = (\alpha_t + \beta_t x)/(1 + \delta_t x)$ targeting $x^{-1/b}$. It maintains a preconditioner $M_t \to (GG^\top)^{-1/b}$:
\[
  M_{t+1} = R_t(M_t^b\,GG^\top)\,M_t, \quad M_0 = I.
\]
Writing $b = 2r$ and taking a thin QR $LL^\top = GG^\top$, we define $Y_t = M_t^r L$. A block QR avoids forming $(I + \delta_t Y_t Y_t^\top)^{-1}$:
\begin{equation}
\begin{cases}
  K_t = \begin{bmatrix} \sqrt{\delta_t}\,Y_t^\top \\ I_m \end{bmatrix},
  \quad [Q_1;\,Q_2]\,R = K_t \;\text{(thin QR)}, \\[6pt]
  M_{t+1} = \gamma_t\,M_t + (\alpha_t - \gamma_t)\,Q_2\,Q_2^\top M_t,
\end{cases}
\end{equation}
where $\gamma_t = \beta_t / \delta_t$. The final result is $M_T^a\,G$.

\paragraph{Classical connection.} For $a/b = 1/2$, this is the QDWH algorithm for polar decomposition, using Zolotarev approximants via block QRs to avoid condition-number squaring~\cite{zolopd}.

\paragraph{Complexity.}
\begin{itemize} \itemsep0em
  \item \textit{Init:} 1 thin QR of $G^\top$ (equivalent to 1 G-MM).
  \item \textit{Per step:} 1 QR; $(3 + \lceil\log_2 r\rceil)$ S-MMs.
  \item \textit{Final:} $\lceil\log_2 a\rceil$ S-MMs; 1 G-MM ($M_T^a G$).
\end{itemize}

\subsubsection{Coupled-Chol Iteration}

This method merges Cholesky-factor tracking (from Coupled-$ab$) with the block QR technique (from Rational-Chol). It tracks a Cholesky-like factor $L_t \in \mathbb{R}^{m \times m}$ ($L_t L_t^\top \to I$) directly alongside an accumulator $C_t \in \mathbb{R}^{m \times m}$:
\begin{equation}
\begin{cases}
  W_t = R_t(L_t L_t^\top) & \text{(via block QR on } L_t\text{)}, \\[4pt]
  L_{t+1} = W_t^r\,L_t, & L_0 = L \;\text{(thin QR factor of } G^\top\text{)}, \\[4pt]
  C_{t+1} = W_t\,C_t, & C_0 = I.
\end{cases}
\end{equation}
The final result is $C_T^a\,G$. 

\paragraph{Classical connection.} This directly generalizes QDWH. Instead of exponentiating the accumulated preconditioner $M_t^r$, it exponents only the well-conditioned, single-step update $W_t^r$, improving forward-stability.

\paragraph{Complexity.}
\begin{itemize} \itemsep0em
  \item \textit{Init:} 1 thin QR of $G^\top$.
  \item \textit{Per step:} 1 QR; $(3 + \lceil\log_2 r\rceil)$ S-MMs.
  \item \textit{Final:} $\lceil\log_2 a\rceil$ S-MMs; 1 G-MM ($C_T^a G$).
\end{itemize}
Costs match Rational-Chol, but the numerical stability is superior because $W_t \approx I$.

\subsubsection{Remez Visualisations}

\begin{figure*}[htbp]
  \centering
  \begin{tabular}{cc}
    \includegraphics[width=0.49\textwidth]{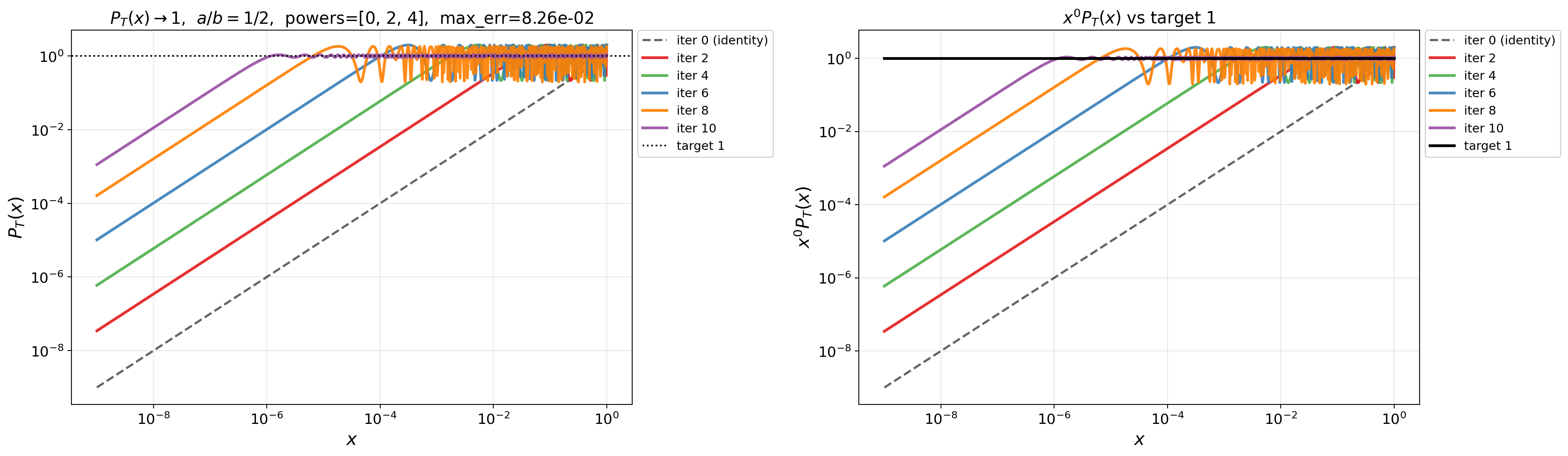} &
    \includegraphics[width=0.49\textwidth]{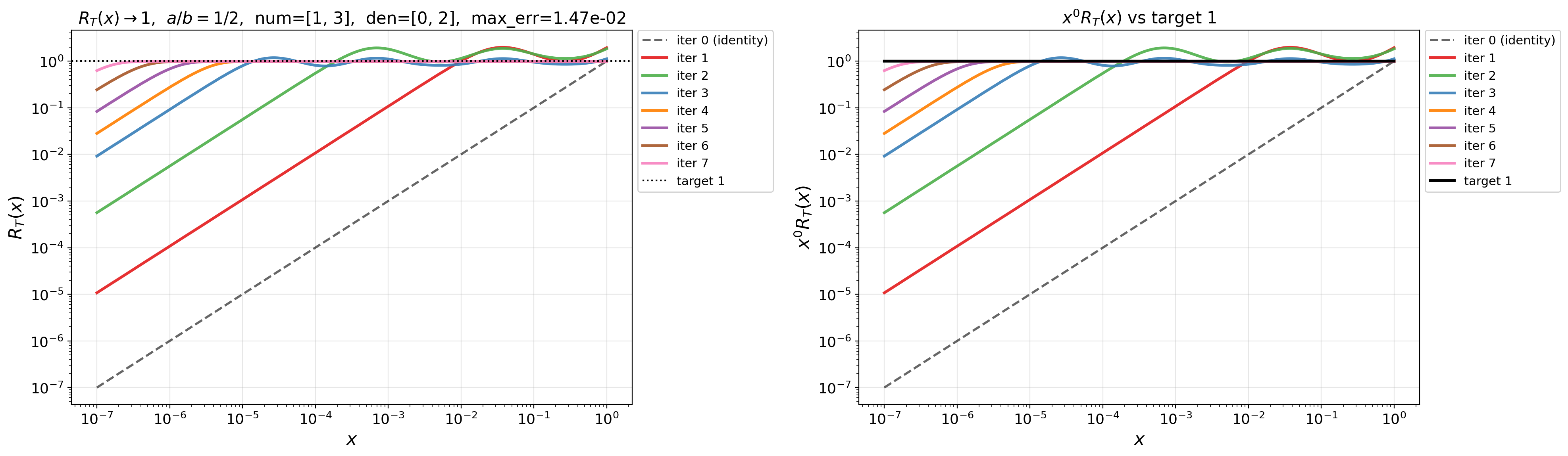} \\[2pt]
    \includegraphics[width=0.49\textwidth]{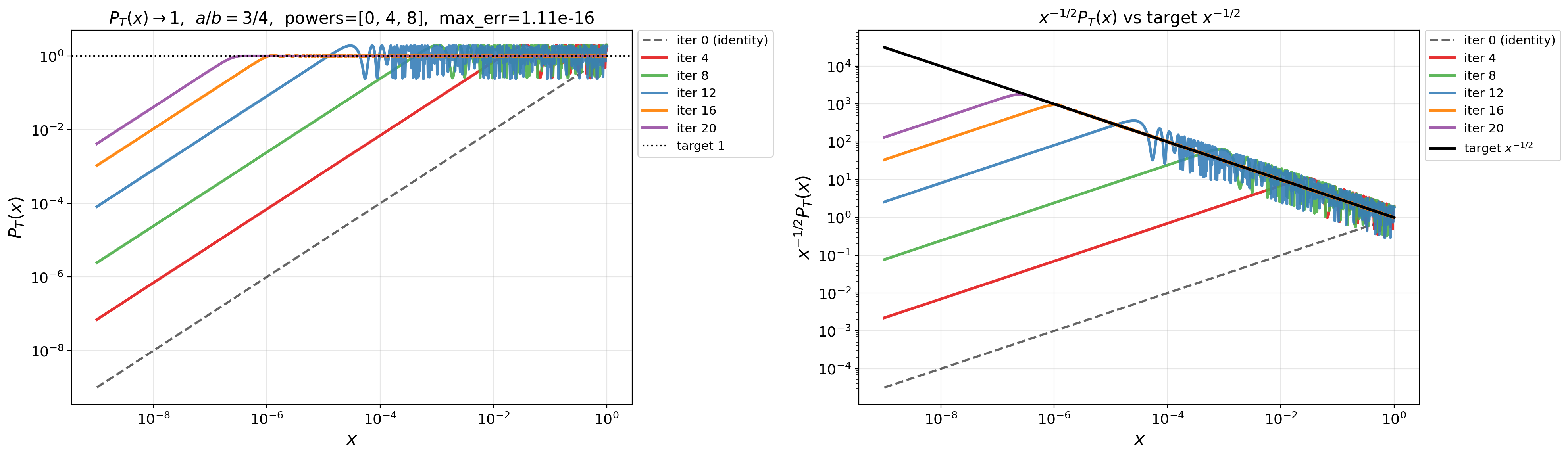} &
    \includegraphics[width=0.49\textwidth]{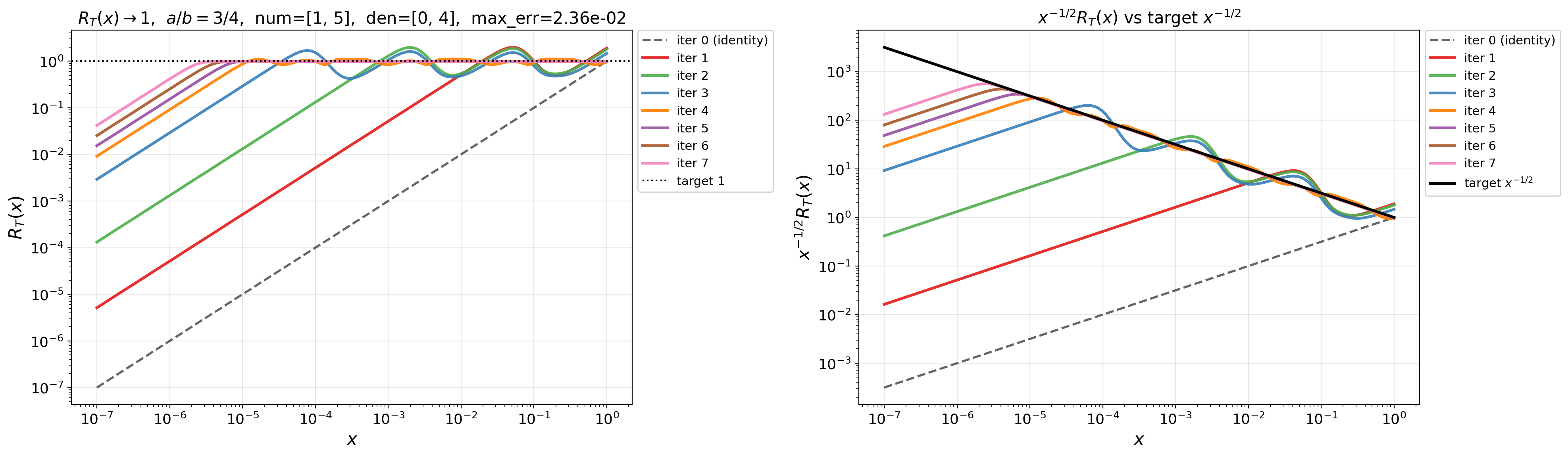} \\[2pt]
    \includegraphics[width=0.49\textwidth]{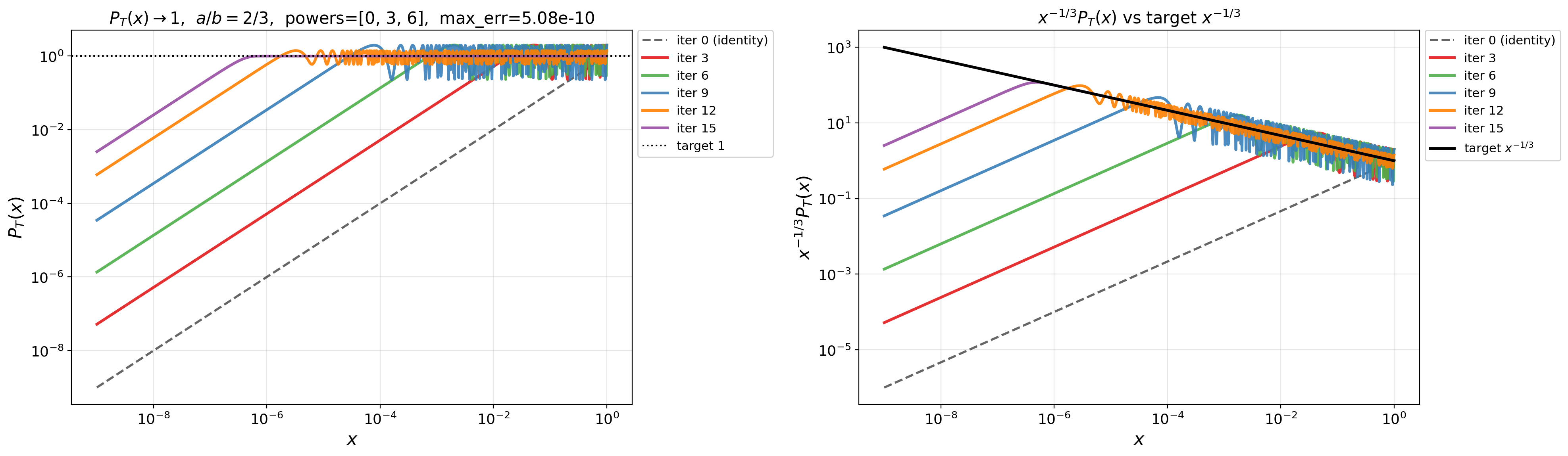} &
    \includegraphics[width=0.49\textwidth]{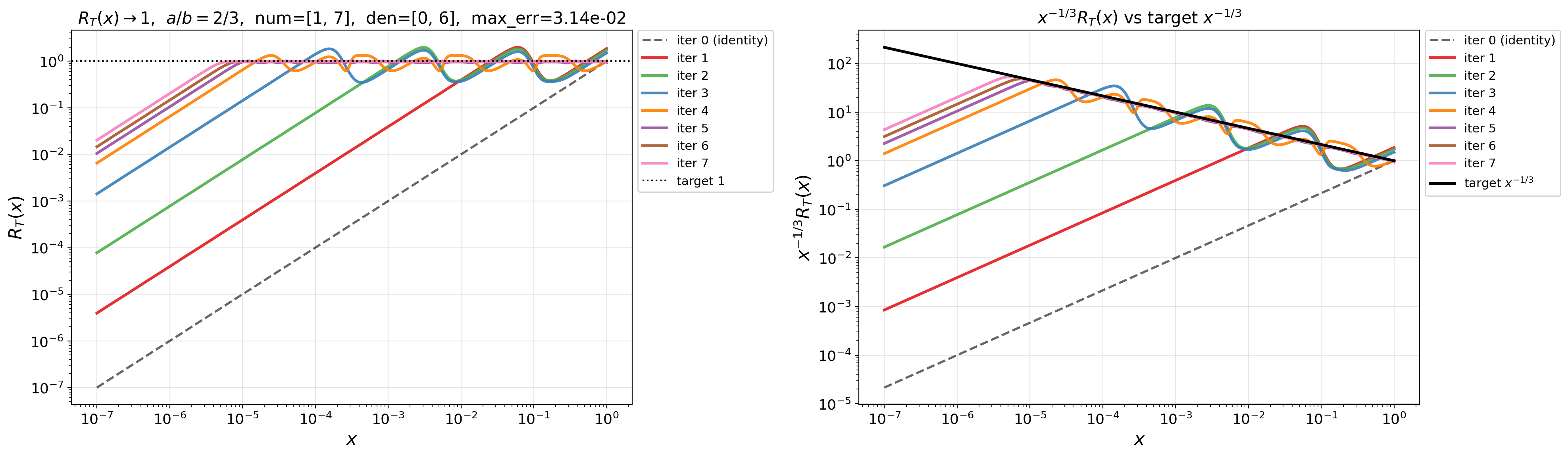} \\[2pt]
    \includegraphics[width=0.49\textwidth]{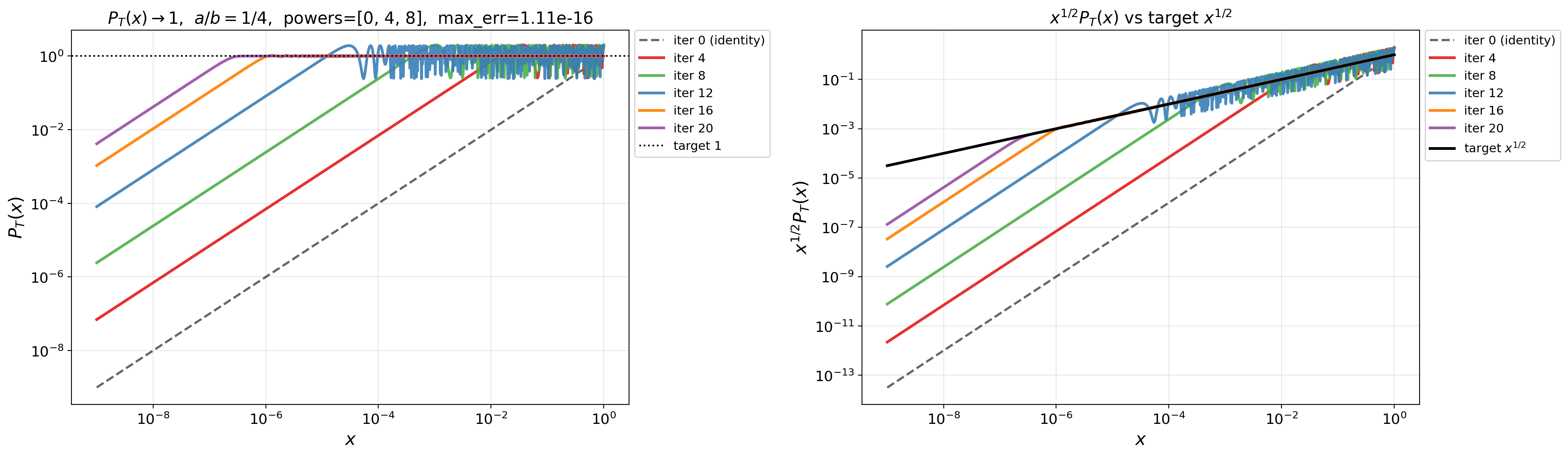} &
    \includegraphics[width=0.49\textwidth]{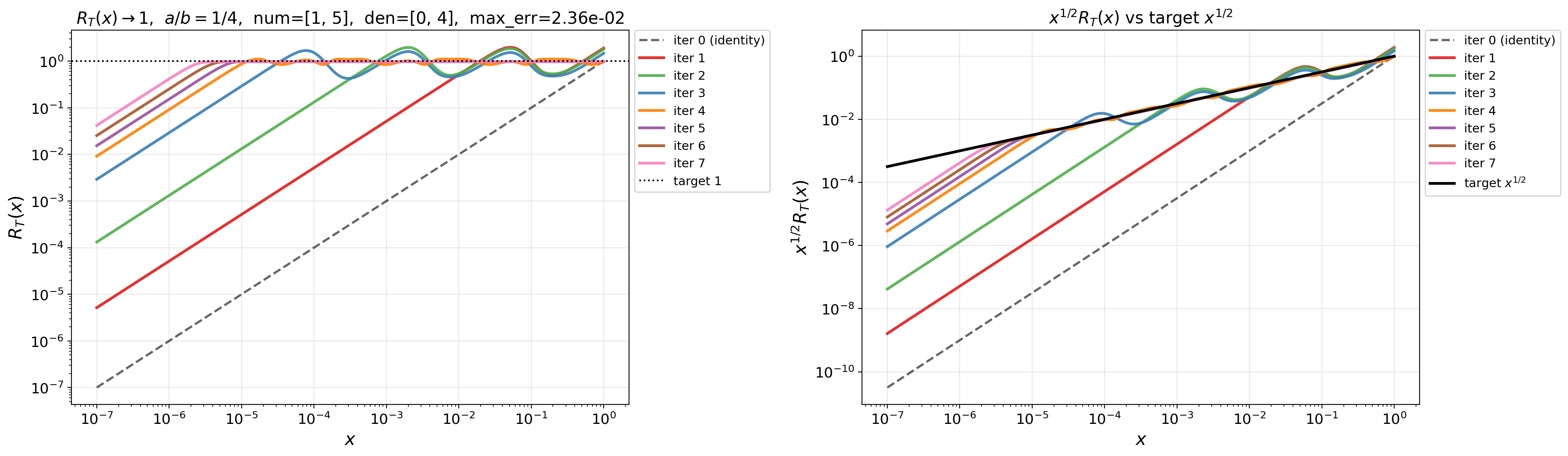} \\[2pt]
    \includegraphics[width=0.49\textwidth]{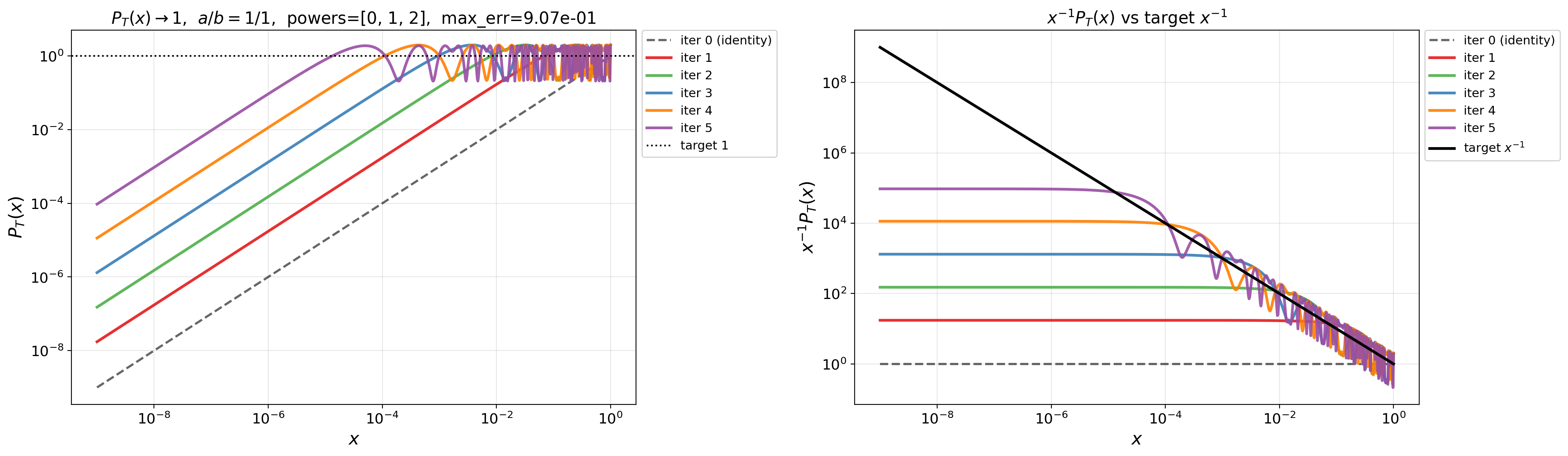} &
    \includegraphics[width=0.49\textwidth]{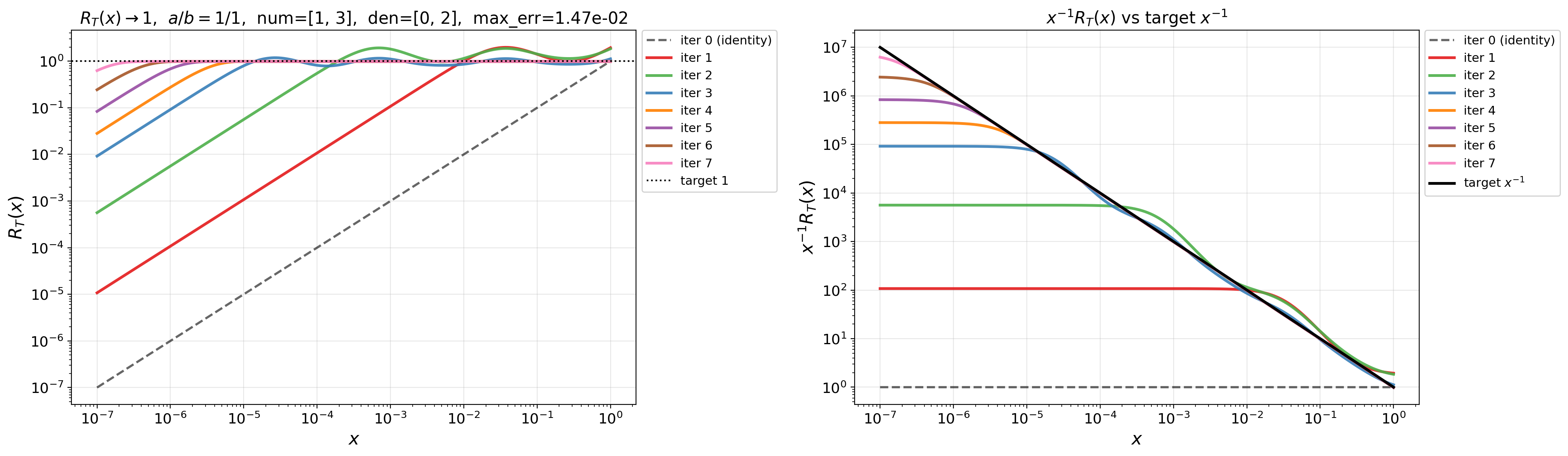} \\
  \end{tabular}
    \caption{Convergence of polynomial (left columns) and rational (right columns) Remez iterations for computing $(GG^\top)^{-a/b}G$, shown for $a/b\in{1/2,3/4,2/3,1/4,1/1}$. Each pair of panels shares an $x$-axis representing a scalar $x\in[\ell,1]$, which stands in for the singular values of the (normalised) input matrix after mapping to the unit interval. \emph{Left panel of each pair}: the raw iterated composition $P_T(x)$ (polynomial) or $R_T(x)$ (rational), which is initialised at the identity and should converge to $1$ on $[\ell,1]$ as the number of steps $T$ increases; curves for successive values of $T$ are overlaid, progressing from light to dark. \emph{Right panel of each pair}: the rescaled quantity $x^{1-2a/b}P_T(x)$ or $x^{1-2a/b}R_T(x)$ plotted against the target $x^{1-2a/b}$; since the singular values of the exact output $(GG^\top)^{-a/b}G$ are $\sigma_i^{1-2a/b}$, convergence of the composition to $1$ is equivalent to convergence of the rescaled curve to the target. Polynomial coefficients are obtained obtained by the Remez algorithm on $[\ell,1]$ with $\ell=10^{-16}$; rational coefficients are obtained by the Remez algorithm with $\ell=10^{-5}$. For rational approximations, exponents $a/b$ with odd denominator (e.g.\ $2/3$), the rational iteration internally doubles to $4/6$ to ensure the companion-form construction remains valid.}
  \label{fig:compositions}
\end{figure*}

\subsection{Picking Parameters for Freon Iterations: Cushion and l}

The rational Newton--Schulz iteration requires two hyperparameters: a cushion factor $\delta$ that stabilises the Remez denominator coefficients, and a lower spectral bound $l$ that determines how wide a range the iteration must cover. Both must be set as functions of $b$ to avoid either numerical blow-up or wasted iterations. Given the iterations wildly oscillate until $l$ gets mapped close to $1$, which are then stabilized $l$ is chosen such that the 5th iteration stabilizes. Cushion primarily controls the size of the resulting coefficients out of the fitting, with smaller cushions resulting in ever larger values of $\gamma$. For this reason we introduce an artificial numerics based bound of $10^5$. From empirical sweeps (\Cref{fig:ns-hyperparams}) we derive closed-form schedules $\delta(b) = (1.84 \times 10^{-8})^{1/b}$ and $l(b) = (10^{-11})^{2/b}$, which we use throughout without further tuning.

\begin{figure}[t]
\centering
\includegraphics[width=0.48\linewidth]{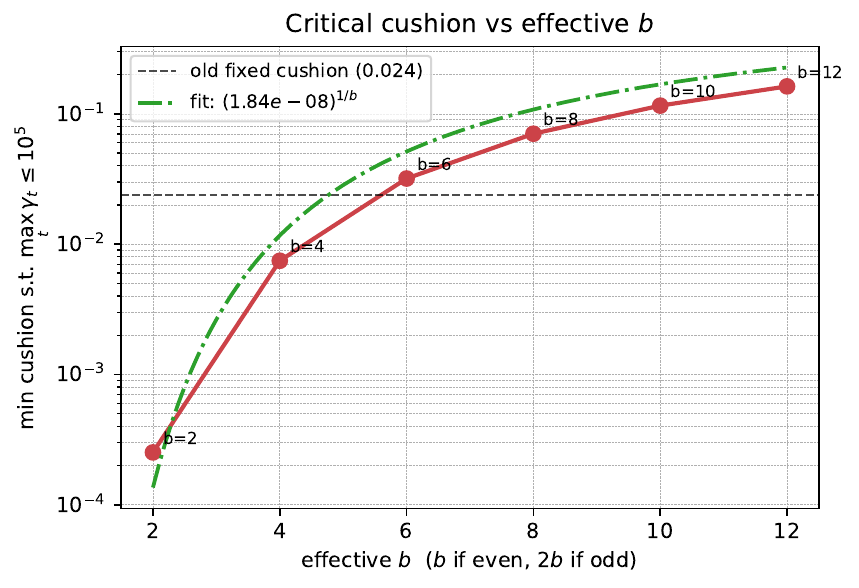}
\hfill
\includegraphics[width=0.48\linewidth]{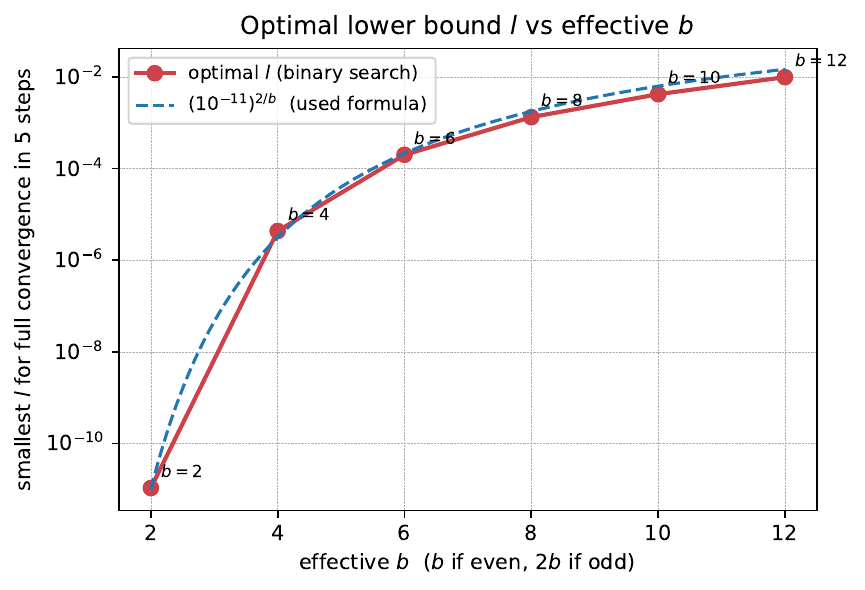}
\caption{
    Hyperparameter selection for the rational Newton--Schulz iteration.
    \textbf{(Left)} Minimum cushion factor required to keep the denominator
    coefficients bounded ($\max_t \gamma_t \leq 10^5$) as a function of
    effective~$q$ (where odd $q$ is replaced by $2q$).
    The fitted formula $(1.84 \times 10^{-8})^{1/q}$ safely exceeds
    the critical threshold for all $q$ and replaces the old fixed value of $0.024$.
    \textbf{(Right)} Smallest lower bound~$l$ such that the iteration
    fully converges within five Newton--Schulz steps, i.e.\ the final
    interval width reaches its asymptotic fixed point.
    The formula $(10^{-11})^{2/q}$ matches the empirical optimum across all
    tested values of~$q$ and ensures all steps contribute to convergence.
}
\label{fig:ns-hyperparams}
\end{figure}


\begin{algorithm}[t]
\caption{Coupled QDWH Cholesky Iteration for $(G G^\top)^{-a/b} G$ }
\label{alg:coupled_chol}
\SetAlgoLined
\DontPrintSemicolon
\KwIn{Gradient $G \in \mathbb{R}^{k \times n}$ ($k \le n$), powers $a, b \in \mathbb{N}$, rational coefficients $\{(\alpha_t, \beta_t, \gamma_t)\}_{t=1}^N$, number of steps $N$, regularization $\epsilon$.}
\KwOut{Approximation to $(G G^\top)^{-a/b} G$.}

\tcp{Ensure $b$ is even to yield an integer $r$ for the block QR step}
\If{$b \pmod 2 \neq 0$}{
    $a \gets 2a, \quad b \gets 2b$
}
$r \gets b / 2$\;

\tcp{Normalize gradient to improve numerical stability}
$\nu \gets \|G\|_F + \epsilon$\;
$G \gets G / \nu$\;

\tcp{Initialize $L$ via thin QR of augmented gradient}
$Q_0, R_0 \gets \text{QR}\left(\begin{bmatrix} G^\top \\ \sqrt{\epsilon} I_k \end{bmatrix}\right)$\;
$L \gets R_0^\top$ \tcp*{Yields $L L^\top = G G^\top + \epsilon I$}
$L_0 \gets L$ \tcp*{Saved for log-determinant; $\log\det(GG^\top) = 2\sum_i \log [L_0]_{ii} + 2k\log\nu$}
$C \gets I_k$ \tcp*{Accumulator}

\For{$t = 1$ \KwTo $N$}{
    \tcp{Compute $W = \alpha_t I + \beta_t A_t(I + \gamma_t A_t)^{-1}$ where $A_t = LL^\top$}
    $\rho_t \gets \beta_t / \gamma_t$\;

    \tcp{Construct Block Matrix $K$ optimally to avoid numeric blowup}
    \eIf{$\gamma \le 1.0$}{
        $K \gets \begin{bmatrix} \sqrt{\gamma_t} L^\top \\ I_k \end{bmatrix}$\;
    }{
        $K \gets \begin{bmatrix} L^\top \\ \frac{1}{\sqrt{\gamma_t}} I_k \end{bmatrix}$\;
    }

    $Q_K, R_K \gets \text{QR}(K)$\;
    $Q_2 \gets \text{bottom } k \times k \text{ block of } Q_K$\;
    $V \gets Q_2 Q_2^\top$ \tcp*{Identity: $V = (I + \gamma_t L L^\top)^{-1}$}
    $W \gets \rho_t I + (\alpha_t - \rho_t) V$\;

    $W \gets \frac{1}{2}(W + W^\top)$ \tcp*{Ensure numerical symmetry}

    \tcp{Coupled state updates}
    $L \gets W^r L$ \tcp*{Equivalent to: $L_{t+1}L_{t+1}^\top = W^b A_t$}
    $C \gets W C$ \tcp*{Accumulate product $\prod W_t$}
}

\tcp{Rescale and apply to gradient}
\KwRet{$\nu^{1 - 2a/b} C^a G$, $L_0$}
\end{algorithm}

\section{Further Optimizer Details}
\subsection{Kaon}\label{ap:Kaon}
\begin{figure}[h]
    \centering
    \begin{subfigure}[b]{\textwidth}
        \centering
        \includegraphics[width=\textwidth]{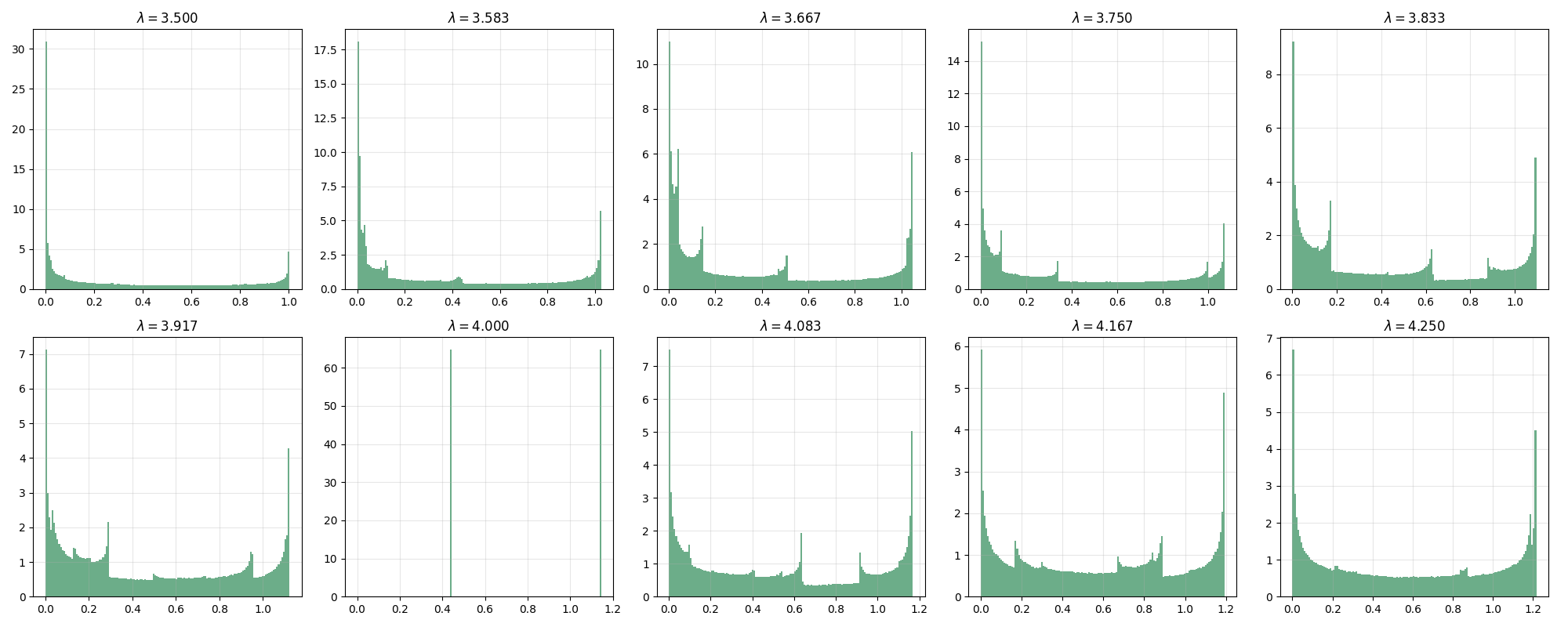}
        \caption{Empirical stationary distributions of the Kaon map
            $x_{n+1} = \lambda\,x_n(1-x_n^2)^2$ for ten values of $\lambda$
            uniformly spaced in $[3.5, 4.25]$.
            Each histogram is obtained by evolving $5{,}000$ particles
            for 500 burn-in iterations followed by 200 collection steps.}
        \label{fig:kaon_pdf_grid}
    \end{subfigure}
    \vspace{1em}
    \begin{subfigure}[b]{0.5\textwidth}
        \centering
        \includegraphics[width=\textwidth]{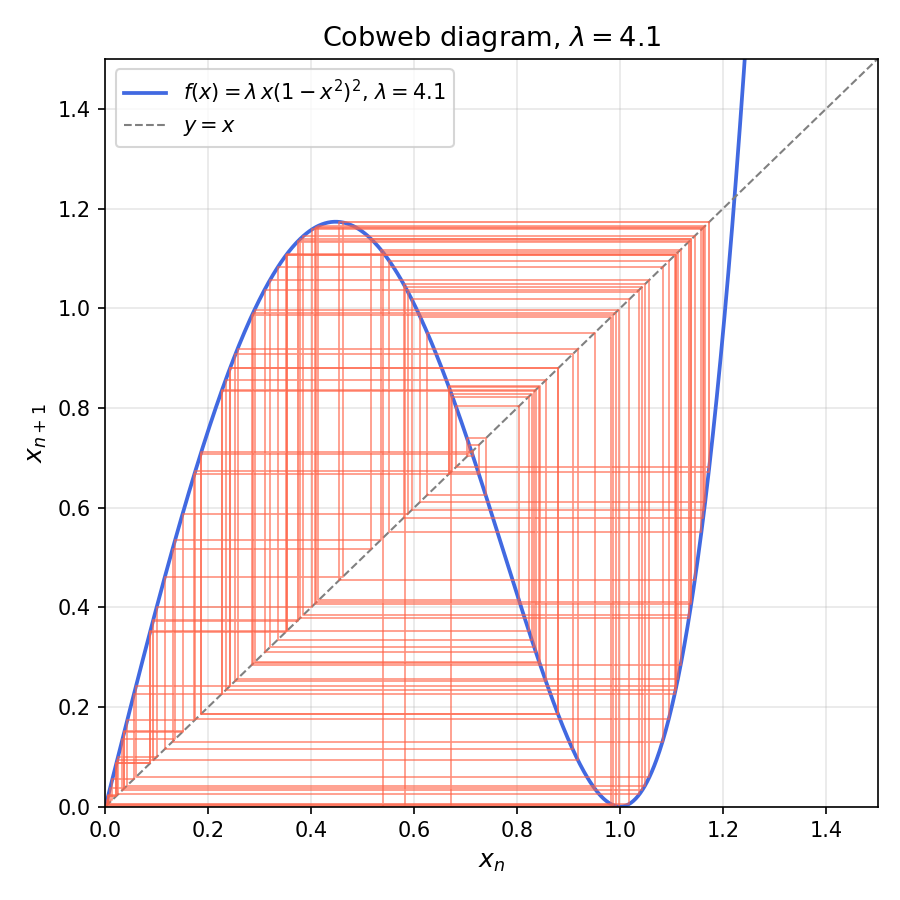}
        \caption{Cobweb diagram of the Kaon map at $\lambda=4.1$,
            showing three trajectories of 40 iterations each
            from independent random initial conditions in $(0.1,\,0.9)$.
            The map curve $y=f(x)$ (blue) and the diagonal $y=x$ (dashed)
            are overlaid; the chaotic wandering of the iterates is evident.}
        \label{fig:kaon_cobweb}
    \end{subfigure}
    \caption{Dynamics of the Kaon polynomial map
        $f(x) = \lambda\,x(1-x^2)^2$ used in the Kaon optimiser
        as a Newton-Schulz-style spectral normalisation step.}
    \label{fig:kaon_dynamics}
\end{figure}

\paragraph{Stationary Distribution}
The Kaon optimiser applies the polynomial map $f(x) = \lambda x(1-x^2)^2$ iteratively to the singular values of the gradient matrix as a cheap alternative to the exact polar factorisation used by Muon. Unlike Newton--Schulz iterations, which converge the singular values to $\pm 1$, the Kaon map at $\lambda=4.1$ operates in a chaotic regime: rather than converging, the iterates explore a broad stationary distribution supported on $(0,1.175)$. \Cref{fig:kaon_pdf_grid} displays the empirical stationary distributions for ten values of $\lambda$ spanning $[3.5,4.25]$, estimated by running $5{,}000$ particles through 500 burn-in followed by 200 collection steps of the map. The distributions shift and spread as $\lambda$ increases, with the support and shape varying considerably across the range. Figure~\ref{fig:kaon_cobweb} shows a cobweb diagram at $\lambda=4.1$, confirming the chaotic character of the dynamics: trajectories starting from different initial conditions in $(0.1,0.9)$ neither converge to a fixed point nor settle into a periodic orbit, but instead wander erratically through the attractor. This stochasticity in the effective preconditioner is an intrinsic feature of Kaon rather than a deficiency.

\paragraph{Cost of Kaon}
\Kaon uses the same number of steps $T$ as \Muon's Newton--Schulz iteration and matches its matrix-multiplication cost. Each step of both methods forms $XX^\top$, squares one $r\times r$ matrix, and multiplies the resulting degree-two matrix polynomial by $X$, costing $T(4r^2s + 2r^3)$ FLOPs to leading order.

\section{Random Feature Asymptotics}\label{app:LimitingQuantities}
In this section we show how to compute our two quantities in this asymptotic framework.

\begin{definition}
    The $i$th normalized trace is defined by
    \begin{equation*}
        \tau_i = \frac{1}{n}\tr(C^i)
    \end{equation*}
\end{definition}

\begin{theorem}
    Assume that $A = C^{1/2}Z$ with $C$ symmetric nonnegative definite with $\lim \sup \|C\|_{op} < \infty$, and $Z$ a matrix with entries which are iid Gaussians.
    
    Let $H = \frac{1}{b} AA^{\top}$. Then there exists universal polynomials $p_k$ such that we have almost surely
    \begin{equation*}
        \|V^{\top}H^kV - V^{\top} (p_k(C))V\|_{op} \rightarrow 0
    \end{equation*}

    The first three such polynomials are
    \begin{align*}
        p_1(C) &= C \\
        p_2(C) &= C^2 + \delta \tau_1 C \\
        p_3(C) &= C^3 + \delta \tau_1 C^2 + (\delta \tau_2 + \delta^2 \tau_1^2)C
    \end{align*}
\end{theorem}
\begin{proof}
    Choose an orthonormal basis $e_i$ of $\mathbb{R}^r$, as $V$ is an $n \times r$ matrix. It is sufficient to bound this operator norm by bounding the inner product over all vectors forming the orthonormal basis. We thus fix $x = Ve_i$ and $y = Ve_j$ and compute the inner product $\langle x, Hy\rangle$. By Cauchy's Integral Formula, we have
    \begin{equation*}
        \langle x, H^k y\rangle = \frac{1}{2\pi i} \oint_\Gamma z^k \langle x, (W - zI)^{-1}y\rangle dz
    \end{equation*}
    where $\Gamma$ is a contour which encloses the spectrum of $H$. First we explicitly choose the contour for almost every $H$. Note that
    \begin{equation*}
        H = \frac{1}{b} C^{1/2}ZZ^{\top}C^{1/2}
    \end{equation*}
    where $Z$ has iid standard Gaussian columns. Therefore
    \begin{equation*}
        \|H\|_{op} \leq \|C\|_{op} \|\frac{1}{b} ZZ^{\top}\|_{op} = K\|\frac{1}{b}ZZ^{\top}\|_{op}
    \end{equation*}
    By standard Marchenko-Pasteur theory, the spectrum of $\frac{1}{b}ZZ^{\top}$ is almost surely eventually contained in the Marchenko-Pasteur bulk $[(1 - \sqrt\delta)^2, (1+\sqrt\delta)^2]$. Therefore almost surely
    \begin{equation*}
        \|H\|_{op} \leq K((1+\sqrt\delta)^2 + o(1))
    \end{equation*}
    Choose $R > \|H\|_{op}$. Take the contour $\Gamma$ to be the circle of radius $R$.

    Now we use \cite[Theorem 2.6]{Couillet_Liao_2022}. Then we have, similar to Equation 2.46 of loc. cit, 
    \begin{equation*}
        \langle x, H^ky\rangle = \frac{1}{2\pi i} \oint_{\Gamma} z^{k-1} \langle x, (I + \tilde{m}_p(z)C)^{-1}y\rangle dz + o(1)
    \end{equation*}
    where $\tilde{m}_p(z)$ solves the equation
    \begin{equation}\label{eq:mConsistencyEq}
        z = -\frac{1}{\tilde{m}_p(z)} + \frac{\delta}{n} \tr(C(I + \tilde{m}_p(z)C)^{-1})
    \end{equation}
    By expanding $\tilde{m}_p(z)$ in a power series in $z^{-1}$ we obtain a Laurent series. Combining this with the expansion for $(I + \tilde{m}_p(z)C)^{-1}$ gives a Laurent series in $z^{-1}$ with coefficients which are polynomials in $C$. By standard complex analysis, the integral simply selects one of these coefficients. This gives us the result.

    Now we explicitly compute the first three polynomials. Expanding the second term of \Cref{eq:mConsistencyEq} we have
    \begin{equation*}
        \frac{1}{n}\tr(C(I + \tilde{m}_p(z)C)^{-1}) = \tau_1 - \tilde{m}_p(z) \tau_2 + \tilde{m}_p(z)^2 \tau_3 + O(\tilde{m}_p(z)^3)
    \end{equation*}

    This allows us to solve
    \begin{equation*}
        \tilde{m}_p(z) = -\frac{1}{z} - \frac{\delta\tau_1}{z^2} - \frac{\delta\tau_2 + \delta^2 \tau_1^2}{z^3} + O(z^{-4})
    \end{equation*}

    For large $z$, $\tilde{m}_p(z)$ is small hence we may expand
    \begin{equation*}
        (I + \tilde{m}_p(z)C)^{-1} = I - \tilde{m}_p(z)C + (\tilde{m}_p(z)C)^2 - (\tilde{m}_p(z)C)^3 + O(C^4)
    \end{equation*}
    Then by plugging in the expansion for $\tilde{m}_p(z)$ we obtain exactly the polynomials written above.
\end{proof}

\begin{remark}
    We can extend this same theorem, with the same polynomials to $Z$ having entries independent mean zero, unit variances random variables, with a minor condition on the tails, but to keep with the main text we do not expand on this condition in detail.
\end{remark}

    




In order to actually evaluate the limit in this theorem we need to assume that for $i = 1,2,3$ the following limit exists
\begin{equation*}
    V^{\top}C^iV \rightarrow \Gamma_i
\end{equation*}

Then by the above theorem we have that for $i = 1, 2, 3$, and the polynomials as defined above
\begin{equation*}
    V^{\top}H^iV \rightarrow p_i(\Gamma_i, \dots, \Gamma_1)
\end{equation*}
We define for $i = 2,3$, $\widehat{\Gamma}_i = p_i(\Gamma_i, \dots, \Gamma_1)$.

Then almost surely
\begin{equation}
    \gamma_{\sgd} \rightarrow \frac{\tr(\Sigma^2\Gamma_1)}{\tr(\Sigma^2\widehat{\Gamma}_2)}
\end{equation}
\begin{equation}
    \Phi_{\sgd} \rightarrow \frac{\tr(\Sigma^2\widehat{\Gamma})^2}{\tr(\Sigma^2\widehat{\Gamma}_3)}
\end{equation}

    

For Muon we have
\begin{equation*}
    D = \operatorname{polar}(GH) = (GH^2G^{\top})^{-1/2}GH
\end{equation*}
Define $T^i = V^{\top}H^iV$, so that almost surely $T^1 \rightarrow \Gamma_1$ and for $i = 2,3$ almost surely $T^i \rightarrow \widehat{G}_i$. Then 
\begin{align}
    \gamma_{\Muon} &= \frac{\tr((GH^2G^{\top})^{-1/2}GHG^{\top})}{\tr((GH^2G^{\top})^{1/2})} = \frac{\tr((\Sigma T_2 \Sigma)^{-1/2}\Sigma T_1 \Sigma)}{\tr((\Sigma T_2 \Sigma)^{1/2})} \\ &\rightarrow \frac{\tr((\Sigma\widehat{\Gamma}_2\Sigma)^{-1/2}\Sigma\Gamma_1\Sigma)}{\tr((\Sigma\widehat{\Gamma}_2\Sigma)^{1/2})}
\end{align}
\begin{align}
    \Phi_{\muon} &= \frac{\tr((GH^2G^{\top})^{1/2})^2}{\tr((GH^2G^{\top})^{-1}GH^3G^{\top})} = \frac{\tr((\Sigma T_2 \Sigma)^{1/2})^2}{\tr((\Sigma T_2 \Sigma)^{-1}\Sigma T_3 \Sigma)} \\ &\rightarrow \frac{\tr((\Sigma\widehat{\Gamma}_2\Sigma)^{1/2})^2}{\tr((\Sigma\widehat{\Gamma}_2\Sigma)^{-1}(\Sigma\widehat{\Gamma}_3\Sigma)} 
\end{align}

\subsection{The Diagonal Case}
If we assume that $C$ is always aligned with the right singular vector basis of $R$, i.e. that $V^{\top}CV = \Lambda$ for some diagonal matrix $\Lambda$ then we can say a lot more about these various formulas.

Indeed we obtain,
\begin{equation*}
    \Gamma_i = \lim V^{\top} C^i V = \Lambda^i
\end{equation*}

We obtain thusly
\begin{equation}
    \gamma_{\sgd} \rightarrow \frac{\tr(\Sigma^2 \Lambda)}{\tr(\Sigma^2(\Lambda^2 + \delta  \Lambda))}
\end{equation}
\begin{equation}
    \Phi_{\sgd} \rightarrow \frac{\tr(\Sigma^2(\Lambda^2 + \delta \Lambda))^2}{\tr(\Sigma^2(\Lambda^3 + 2\delta \Lambda^2 + (\delta+\delta^2)\Lambda))}
\end{equation}
\begin{equation}
    \gamma_{\Muon} \rightarrow \frac{\tr(\Sigma (\Lambda^2 + \delta \Lambda)^{-1/2} \Lambda)}{\tr(\Sigma(\Lambda^2 + \delta\Lambda)^{1/2})}
\end{equation}
\begin{equation}
    \Phi_{\muon} \rightarrow \frac{\tr(\Sigma(\Lambda^2 + \delta\Lambda)^{1/2})^2}{\tr((\Lambda^2 + \delta \Lambda)^{-1}(\Lambda^3 + 2\delta \Lambda^2 + (\delta + \delta^2)\Lambda))}
\end{equation}

Define $x_i = \sigma_i \sqrt{\lambda_i(\lambda_i + \delta)}$ and $h_i = \frac{1}{\lambda_i + \delta}$. Then we rewrite the alignments as
\begin{equation}
    \gamma_{\muon} = \frac{\sum x_i h_i}{\sum x_i}
\end{equation}
\begin{equation}
    \gamma_{\sgd} = \frac{\sum x_i^2 h_i}{\sum x_i^2}
\end{equation}

For $\Phi$ define
\begin{equation*}
    D_i = \frac{\lambda_i^3 + 2\delta\lambda_i^2 + (\delta + \delta^2)\lambda_i}{\lambda_i^2 + \delta\lambda_i}
\end{equation*}
Then we get
\begin{equation}
    \Phi_{\muon} = \frac{(\sum x_i)^2}{\sum D_i}
\end{equation}
\begin{equation}
    \Phi_{\sgd} = \frac{(\sum x_i^2)^2}{\sum x_i^2 D_i}
\end{equation}

\propalignment
\begin{proof}
    Because we have that the $\lambda_i$ are in non-increasing order, and the $\sigma_i$ are singular values hence are in non-increasing order, we obtain that the $x_i$ are also in non-increasing order. Finally, the $h_i$ are in non-decreasing order as $h_i$ is decreasing in $\lambda_i$.

    Then
    \begin{equation*}
        \gamma_{\muon} - \gamma_{\sgd} = \sum_{i,j} x_i x_j^2 h_i - \sum_{i, j}x_i x_j^2 h_j = \sum_{i < j} x_ix_j(x_j - x_i)(h_i - h_j) \geq 0
    \end{equation*}
    Equality holds if and only if one of the two sequence $x_i$ or $h_i$ are not strictly increasing or decreasing. This is equivalent to one of the sequence $\sigma_i$ and $\lambda_i$ being not strictly increasing.
\end{proof}

\propIsotropic
\begin{proof}
    We simply compute that
    \begin{equation*}
        D_i = \frac{1 + 3\delta + \delta^2}{1 + \delta}
    \end{equation*}
    Therefore
    \begin{equation*}
        \Phi_{\sgd} = \frac{(1+\delta)^2}{1+3\delta+\delta^2}\sum \sigma_i^2
    \end{equation*}
    \begin{equation*}
        \Phi_{\muon} = \frac{(1+\delta)^2}{1 + 3\delta + \delta^2}\frac{(\sum \sigma_i)^2}{d}
    \end{equation*}
    Then Cauchy-Schwartz gives $\Phi_{\muon} \leq \Phi_{\sgd}$.
\end{proof}

\thmdescent
\begin{proof}
    For $\delta \rightarrow \infty$ we have
    \begin{equation*}
        \lim_{\delta \rightarrow \infty} \gamma_{\muon}^2\Phi_{\Muon} = \frac{1}{\delta^2} \frac{(\sum \sigma_i \lambda_i^{1/2})^2}{d}
    \end{equation*}
    \begin{equation*}
        \lim_{\delta \rightarrow \infty} \gamma_{\sgd}^2\Phi_{\sgd} = \sum \sigma_i^2 \lambda_i
    \end{equation*}
    Then by Cauchy-Schwartz, we have
    \begin{equation*}
        \lim_{\delta \rightarrow \infty} \gamma_{\muon}^2\Phi_{\muon} \leq \lim_{\delta \rightarrow \infty} \gamma_{\sgd}^2\Phi_{\sgd}
    \end{equation*}
\end{proof}

\section{Random Feature Experiments}
\subsection{Detailed Setup}\label{ap:rf_details}
For the Random Feature experiments considered in this paper, we utilize the setup and code from \citet{davis2026spectralgradientupdateshelp}, provided in \href{https://github.com/damek/specgd/}{https://github.com/damek/specgd/}. The specific setup is as follows: We consider a random feature regression problem with objective
$ f(W) = \frac{1}{N\sqrt{d}}\|WA - Y\|_F^2$,
where $W \in \mathbb{R}^{o \times d}$ is the weight matrix,
$A \in \mathbb{R}^{d \times N}$ is a fixed random feature matrix,
and $Y = W^\star A$ for a randomly initialised ground-truth $W^\star$.
Since the loss is exactly quadratic in $W$, the gradient is
$G = \nabla_W  f = \frac{2}{N\sqrt{d}}(W - W^\star)AA^\top$
and the Hessian action along any direction $D$ is
$\langle D, \nabla^2 f[D]\rangle = \frac{2}{N\sqrt{d}}\|DA\|_F^2$.
Consequently, the exact optimal step size for any update direction $D$ is
\begin{equation}
    \eta^\star(D) = \frac{\langle G, D\rangle}{\frac{1}{\sqrt{d}}\|DA\|_F^2}.
\end{equation}

We run two experiments with this setup.
In the \textbf{ReLU} experiment, $A$ is obtained by applying a ReLU activation
to a random linear map of $N=400$ inputs, yielding $W \in \mathbb{R}^{120\times100}$,
$A \in \mathbb{R}^{100\times400}$, stable rank $\mathrm{st}(A)\approx 3.06$,
$\|A\|_{\mathrm{op}}\approx 80.87$, and $\|A\|_F\approx 141.40$;
training runs for $1000$ iterations.
In the \textbf{SwiGLU} experiment the same dimensions are used but the activation
is SwiGLU, giving $\mathrm{st}(A)\approx 29.96$, $\|A\|_{\mathrm{op}}\approx 22.15$,
and $\|A\|_F\approx 121.22$; training runs for $400$ iterations.
Both experiments use base learning rate $\eta=10^{-2}$, double precision (\texttt{float64}), and fixed random seed.
The Optimal~C methods select $c$ and $\eta$ jointly at each step by maximising
the quadratic lower bound $\Delta f(c,\eta) = -\eta\,a(c) + \eta^2 b(c)$
over a grid, where $a(c)=\langle G, D_c\rangle$ and $b(c)=\frac{1}{\sqrt{d}}\|D_c A\|_F^2$
with $D_c = (GG^\top)^{-c}G$. \Freon is implemented using SVD, while \Kaon is implemented with SVD with random uniform noise.

\subsection{Exponent Approximation for Optimal Schatten Norm}\label{ap:rf_approx}
We are in the random feature model setting and assume to have $G \in \mathbb{R}^{m \times n}$ be gradient with a compact Singular Value Decomposition $G = U S V^\top$, where $S = \text{diag}(\sigma_1, \dots, \sigma_r)$. We seek a spectrally scaled matrix $D = U S^a V^\top$ that maximizes optimal decrease $\Phi_k$.

The objective is to maximize decrease at optimal step size, for $\|DA\|_F^2 = \sum k_i \sigma_i^{2a}$, where $k_i = \|A^\top v_i\|^2$ represents the projection energy of an activation matrix $A$:
\begin{equation}
    \max_{a} g(a) = \frac{\left( \sum_{i=1}^r \sigma_i^{a+1} \right)^2}{\sum_{i=1}^r k_i \sigma_i^{2a}}
\end{equation}
Setting the derivative of $\log g(a)$ to zero yields the stationarity condition for optimality:
\begin{equation}
    \frac{\sum_{i=1}^r \sigma_i^{a+1} \log \sigma_i}{\sum_{i=1}^r \sigma_i^{a+1}} = \frac{\sum_{i=1}^r k_i \sigma_i^{2a} \log \sigma_i}{\sum_{i=1}^r k_i \sigma_i^{2a}}
\end{equation}

To approximate an optimal $a$ value, we will assume the singular values of $G$ and the penalty weights follow a power-law decay:
\begin{equation}
    \sigma_i \propto i^{-\alpha} \quad (\alpha > 0), \qquad k_i \propto i^{-\beta}
\end{equation}
In this case, equality is achieved by equating exponents of LHS and RHS, i.e. $\alpha(a+1) = 2a\alpha + \beta$, meaning that
\[
a = 1 - \frac{\beta}{\alpha}
\]

\subsection{Optimal C methods}
Both adaptive-exponent methods take the update direction
$D_c = (GG^\top)^{-c}G = U\,\mathrm{diag}(\tilde{\sigma}^{1-2c})\,V^\top$,
where $G = U\Sigma V^\top$ is the SVD of the gradient and
$\tilde{\sigma}_i = \sigma_i/\mu_c$ are the singular values normalised by the
power mean $\mu_c = \bigl(\tfrac{1}{r}\sum_i \sigma_i^{2(1-c)}\bigr)^{1/(2(1-c))}$.
The optimal step size $\eta^\star(c) = n\,a(c)/b(c)$ is used exactly at every
step, where $a(c) = \sum_i \sigma_i\,\tilde{\sigma}_i^{1-2c}$ and
$b(c) = \|D_c A\|_F^2/\sqrt{d}$.

\paragraph{Optimal C [Greedy]} determines $c$ by exhaustive search: at each
iteration a grid of 41 values $c \in [-0.5, 1.5]$ is evaluated and the value
minimising the quadratic predicted loss decrease
$\Delta f^\star(c) = -n\,a(c)^2/(2\,b(c))$ is selected.

\paragraph{Optimal C [Scaling]} instead infers $c$ analytically from the spectral
structure of the gradient and the feature matrix based on \Cref{ap:rf_approx}.
Fitting $\log\sigma_i \approx -\alpha\log i$ and $\log k_i \approx -\beta\log i$
by least squares, where $k_i = \|v_i^\top A\|^2$ is the squared projection of
the $i$-th right singular vector onto the feature matrix, the theoretically
optimal exponent is
\begin{equation}
    c^\star = \frac{\beta}{2\alpha},
\end{equation}
which is clamped to $[-0.5, 1.5]$ and used with step size $\eta^\star(c^\star)$.

\begin{figure}[t]
  \centering
  \includegraphics[width=\linewidth]{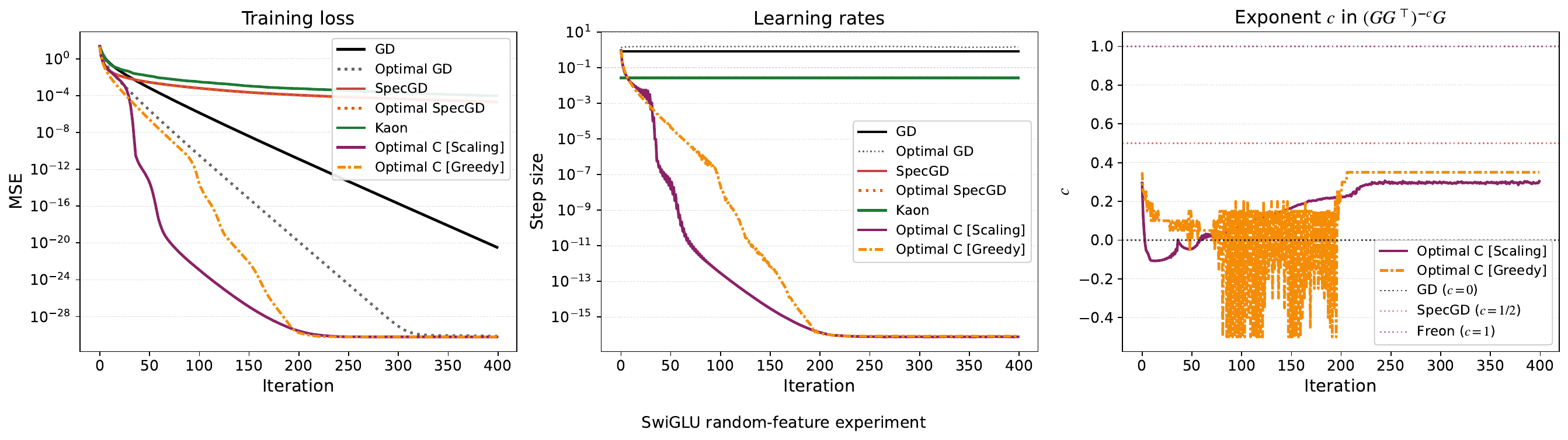}
  \caption{%
    \textbf{Random-feature regression with SwiGLU activations.}
    Same setup as Figure~\ref{fig:rf-relu} but with a SwiGLU random-feature
    matrix $A \in \mathbb{R}^{100 \times 400}$
    ($\mathrm{st}(A) \approx 29.96$, substantially higher than the ReLU
    case), trained for 400 iterations.
    \emph{Left:} training loss.
    \emph{Centre:} effective step sizes; Optimal GD exhibits strongly
    oscillatory step sizes throughout training.
    \emph{Right:} exponent $c$ in $(GG^\top)^{-c}G$ for the two
    Optimal~C methods.
    The higher stable rank of $A$ under SwiGLU changes the curvature
    landscape relative to the ReLU setting, and both Optimal~C methods
    converge to a smaller value of $c$ (closer to GD) than in the
    ReLU experiment.
  }
  \label{fig:rf-swiglu}
\end{figure}

\section{GPT2 Additional Experiments}

\begin{figure}[t]
  \centering
  \includegraphics[width=\textwidth]{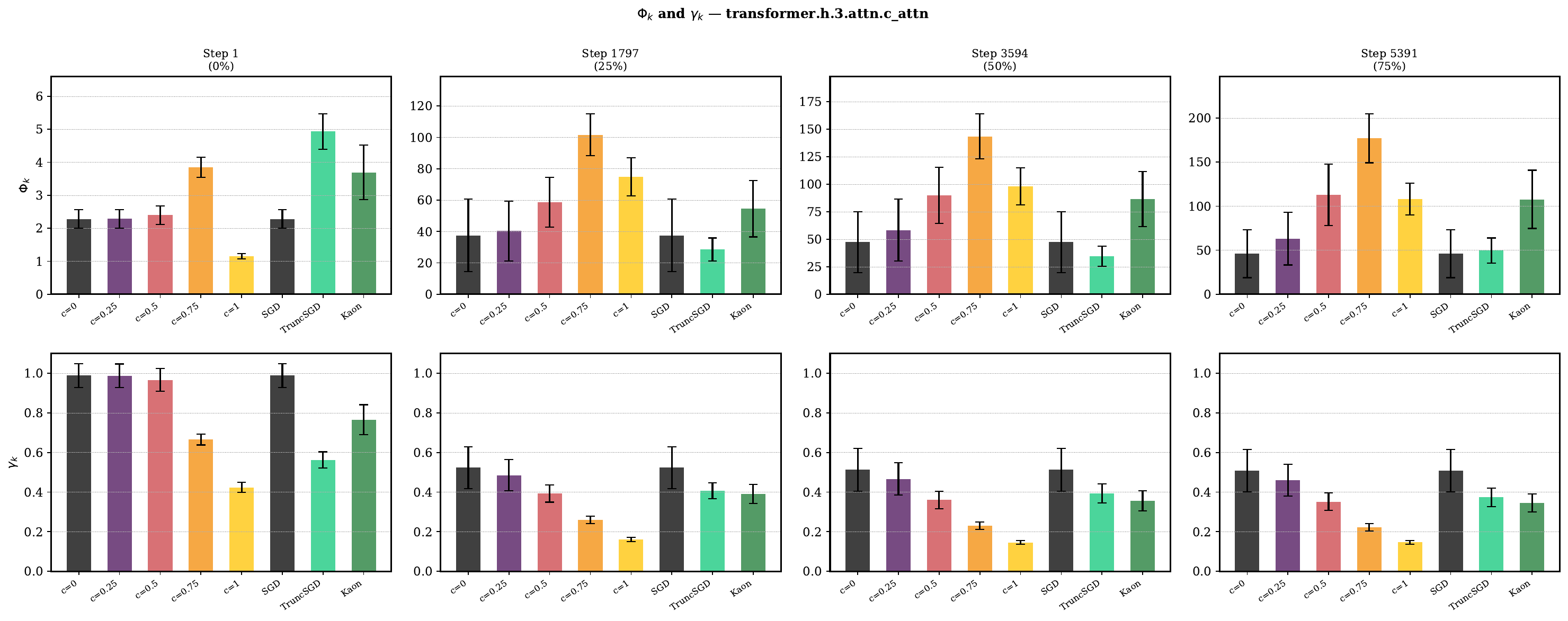}
\caption{%
Per-layer curvature alignment quantities $\Phi_k$ (top row, linear scale) and $\gamma_k$ (bottom row, linear scale) for the weight matrix \texttt{h.3.attn.c\_attn} at four training checkpoints (steps 1, 1797, 3594, 5391). Each bar shows the mean $\pm$ one standard deviation over validation batches, separately for each candidate update direction $D_k$.
The directions include \Freon's power-mean normalised descent direction at interpolation levels $c \in {0, 0.25, 0.5, 0.75, 1}$ (where $c=0$ matches \sgd, and $c=1$ is the full \Freon\ direction), \tsgd\ (batch gradient with the top 1\% of singular values zeroed), and \Kaon.
Formally, for a given direction $D_k$ and batch gradient $g_b$, we calculate $\langle g_b, D_k \rangle^2$, the directional curvature $\lambda_k = \langle D_k, \mathcal{F}_{\mathrm{val}}[ D_k ]\rangle$ where $\mathcal{F}_{\mathrm{val}}$ is the Gauss--Newton (GGN) matrix estimated over all validation batches via Jacobian--vector products, and the batch gradient alignment $\gamma_k = \langle G_k, D_k \rangle / \langle g_b, D_k \rangle$ where $G_k$ is the gradient averaged over all validation batches. A large $\Phi_k$ indicates that the direction is strongly aligned with the batch gradient and lies in a low-curvature region, supporting a larger step, while $\gamma_k \approx 1$ indicates that the batch gradient faithfully represents the global gradient in that direction.
}
  \label{fig:phi_gamma_h3_attn}
\end{figure}

\begin{figure}[h]
    \centering
    \includegraphics[width=0.65\textwidth]{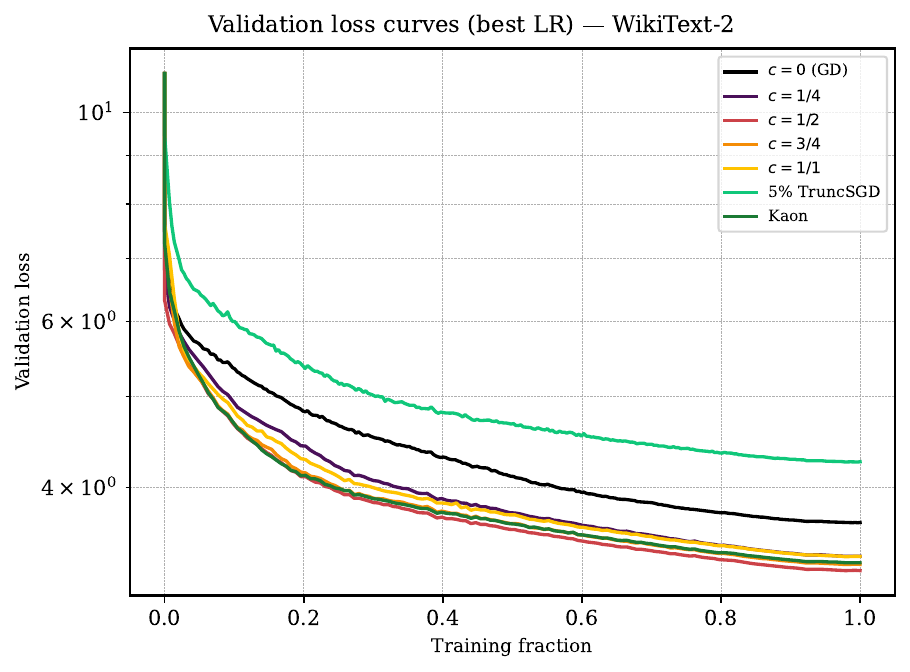}
    \caption{Validation loss as a function of training progress for each optimiser
        at its best learning rate on WikiText-2.
        Curves shown for Freon with exponent $c \in \{0,\tfrac{1}{4},\tfrac{1}{2},\tfrac{3}{4},1\}$,
        5\% TruncatedSGD, and Kaon.}
    \label{fig:val_curves_best_overview}
\end{figure}

\begin{figure}[htbp]
    \centering
    \includegraphics[width=\textwidth]{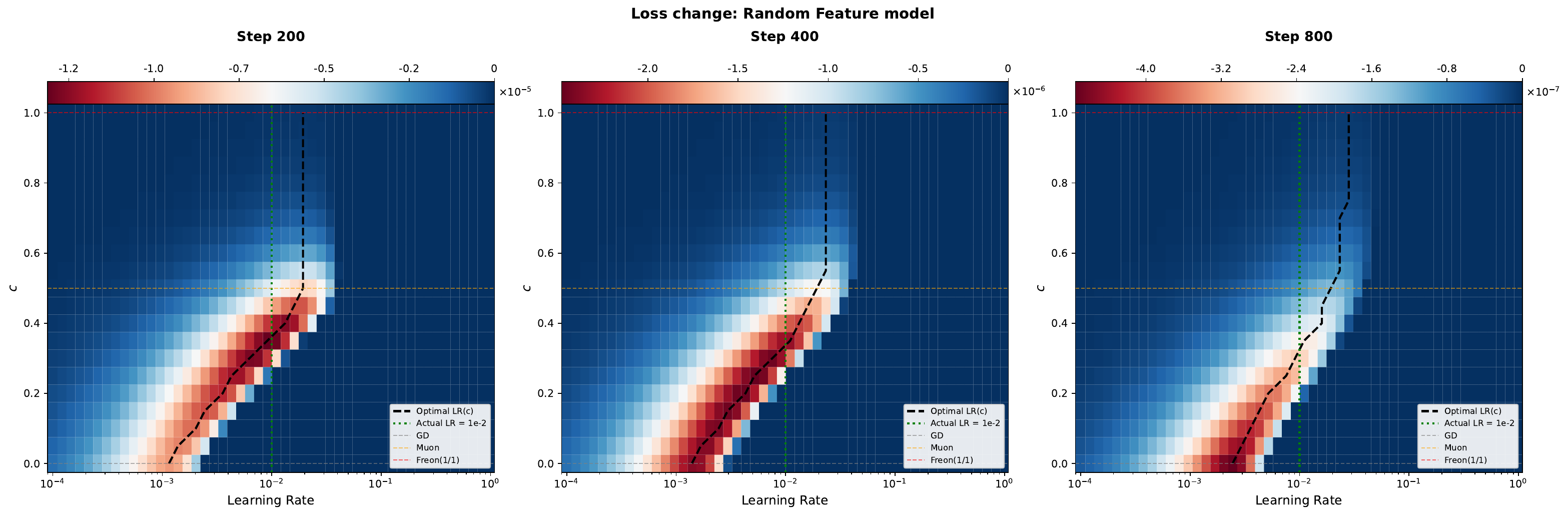}
    \caption{Loss change $\Delta  f$ over the joint $(c, \eta)$ grid for the quadratic model, evaluated at three training iterations of the random feature experiment. Each panel uses the gradient computed at that iteration to evaluate $\Delta  f(c, \eta) = -\eta\, a(c) + \eta^2 b(c)$, where $a(c) = \langle G, D_c \rangle$ and $b(c) = \frac{1}{2n}\|D_c A\|_F^2$. The dashed black curve marks the optimal learning rate per $c$; the green dotted line marks the actual training learning rate $\eta = 0.01$.}
    \label{fig:pq_lr_heatmap_rf}
\end{figure}

\begin{equation}\label{ap:p_tracking}
\Delta_p = \tfrac{\|G\|^2_{\operatorname{dual}_1(p)}}{\|A\|^2_{{\operatorname{dual}_2(p)}}},
\end{equation}
for $\tfrac{1}{p} + \tfrac{1}{\operatorname{dual} _1p} = 1$
and $\tfrac{1}{p} + \tfrac{1}{\operatorname{dual}_2p} = \frac12$.

\begin{figure}[tp]
    \centering
    \includegraphics[width=\textwidth]{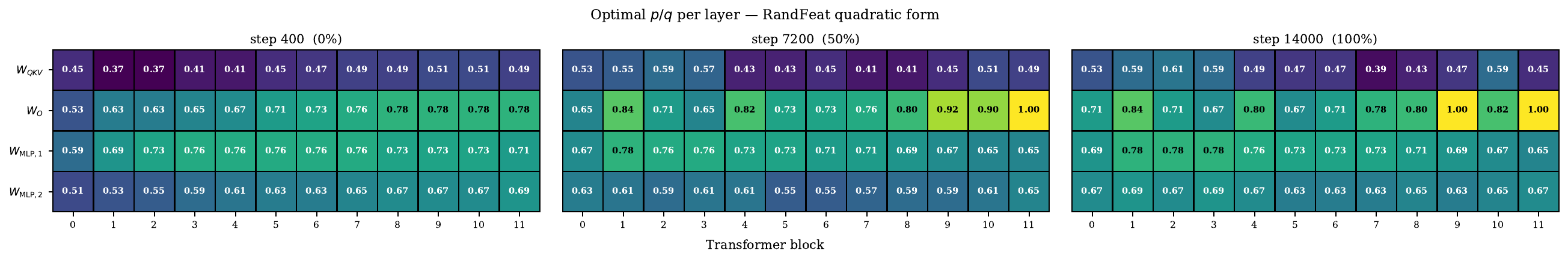}
    \caption{\textbf{Layer-wise optimal $c$ at three training snapshots (RandFeat quadratic model).}
    Each cell shows the value of $c \in [0,1]$ that minimises the predicted one-step
    loss decrease for that layer under a random-feature approximation to the 
    curvature. Panels correspond to the beginning, middle, and end of training.
    Rows correspond to the four weight matrices per transformer block
    ($W_{QKV}$, $W_O$, $W_{\mathrm{MLP},1}$, $W_{\mathrm{MLP},2}$);
    columns index the 12 transformer blocks.
    $c=0$ recovers gradient descent, $c=\tfrac{1}{2}$ recovers Muon,
    and $c=1$ recovers the pseudoinverse-like endpoint.}
    \label{fig:freon_layerwise_heatmap_rf}
\end{figure}

\begin{figure*}[tp]
    \centering
    \includegraphics[width=\textwidth]{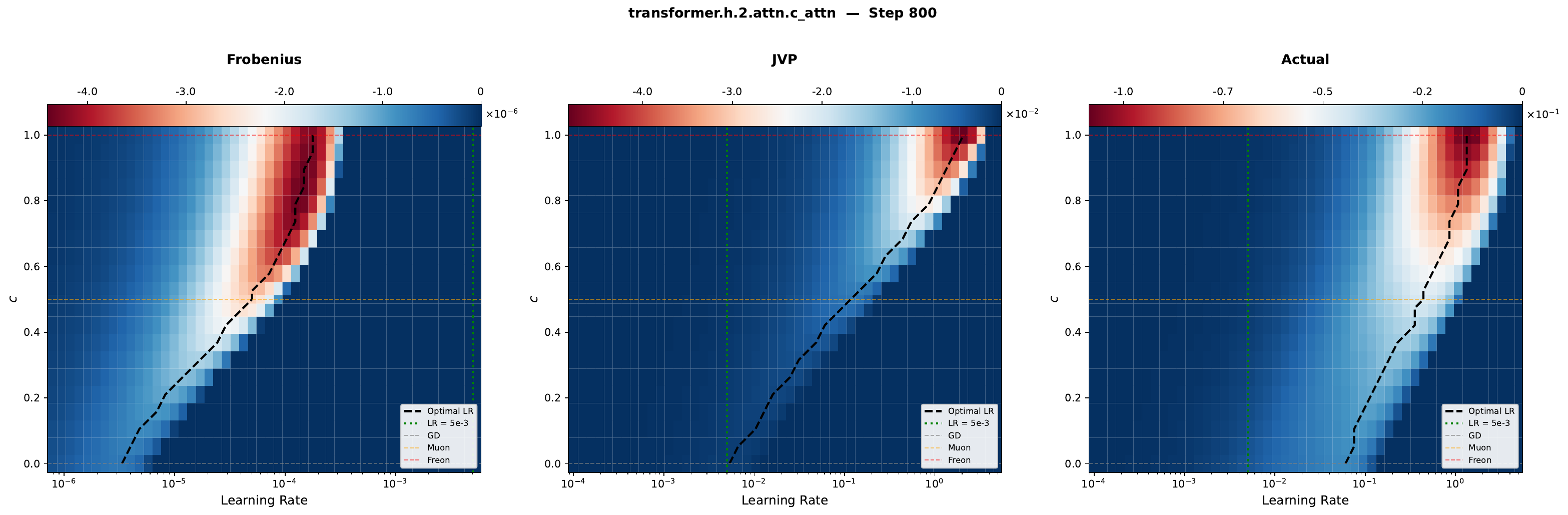}
    \caption{\textbf{Per-step loss decrease vs.\ learning rate and $c$ under three local quadratic models} (layer \texttt{h.2.attn.c\_attn}, step 800).
    \emph{Left (Frobenius):} The naive random feature model predicts $c\approx 2/3$ as optimal.
    \emph{Centre (JVP):} Under the GGN quadratic approximation on a single batch, $c=1$ achieves the largest per-step decrease -- similar to the Newton step being optimal for a fixed-batch quadratic.
    \emph{Right (Actual):} Evaluating on the actual single-batch loss agrees with the GGN picture, motivating the inter-batch SNR analysis of \Cref{sec:asymptotics}. 
    All heatmaps are clamped at zero (only improvements shown); the dashed black curve marks the per-$c$ optimal learning rate and the green dotted line the configured learning rate.}
    \label{fig:ggn_optimal_c}
\end{figure*}

\begin{figure}[htbp]
    \centering
    \includegraphics[width=\textwidth]{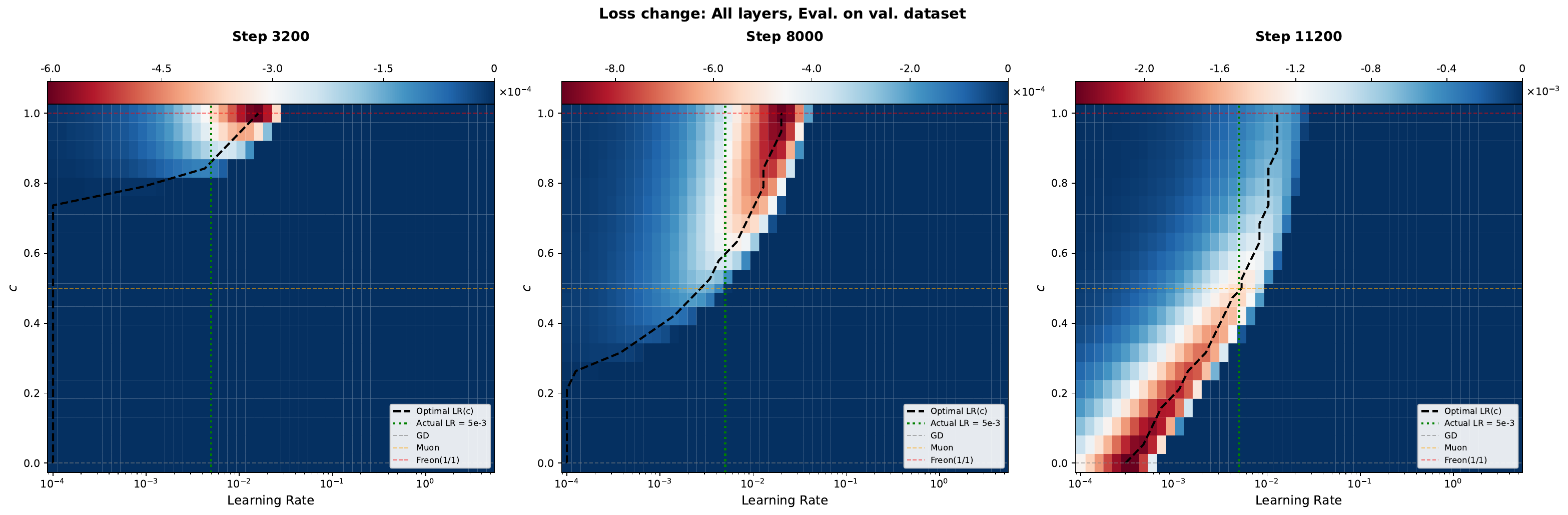}
    \caption{Validation loss change $\Delta f(c, \eta) = f(\theta - \eta, D_c) - f(\theta)$, evaluated on the full validation set after applying a single synchronised update to all model layers, over a joint grid of direction interpolation $c \in [0,1]$ (y-axis) and learning rate $\eta$ (x-axis, logarithmic). Each panel corresponds to a different training checkpoint, i.e. the gradient $G_k$ and momentum state used to construct the update direction are taken from that training step.
    For a given $c$, the per-layer update direction $D_c$ is obtained by applying the Freon power-mean preconditioner with exponent $c$ to the layer's gradient: $c=0$ recovers the vanilla gradient-descent direction ($D = G / |G|_F$) and $c=1$ is the fully preconditioned \Freon direction; intermediate values interpolate smoothly between the two. The same $c$ and $\eta$ are applied to every weight matrix simultaneously, so $\Delta f$ measures the \emph{joint} effect of a global update rather than a per-layer optimum. The heatmap is clamped at zero so that only loss-decreasing regions (negative $\Delta f$) are shown; the colour axis is scaled independently per panel. The dashed black curve traces $\eta^*(c) = \arg\min_\eta \Delta f(c, \eta)$, i.e.\ the empirically optimal learning rate for each direction interpolation $c$. The green dotted vertical line marks the actual learning rate used during training, allowing direct comparison of where training operates relative to the optimal ridge.
}
    \label{fig:pq_lr_heatmap_2}
\end{figure}

\begin{figure}[tp]
    \centering
    \includegraphics[width=\textwidth]{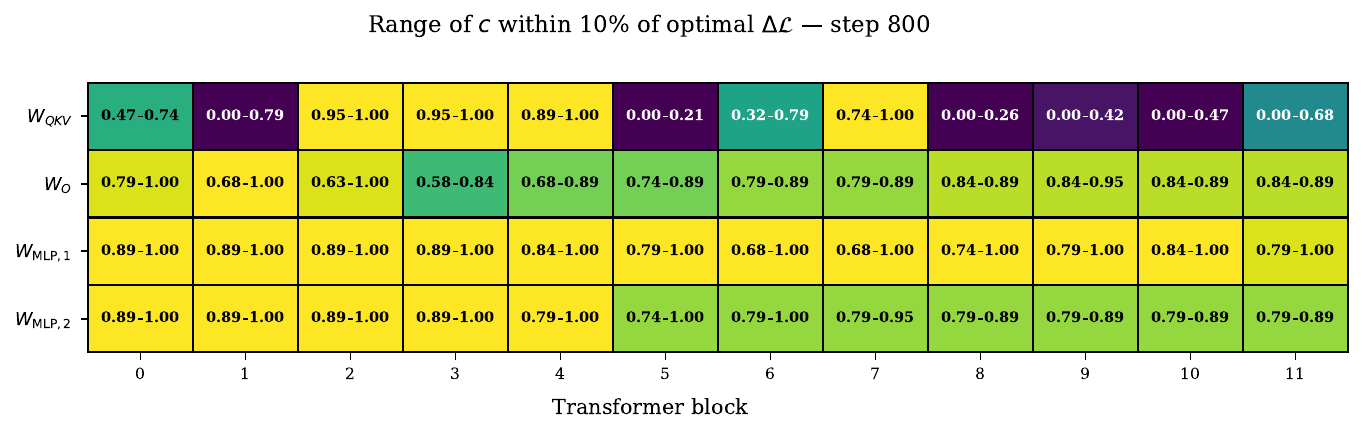}
    \caption{\textbf{Layer-wise range of near-optimal $c$ at step 800.}
    Each cell shows the range of values $c \in [0,1]$ whose best-over-$\eta$
    one-step loss decrease is within 10\% of the global optimum for that layer,
    found by a grid search over learning rates and $c$ using real forward passes
    on the full validation set (actual mode).
    Cell colour indicates the optimal $c$.
    Rows correspond to the four weight matrices per transformer block
    ($W_{QKV}$, $W_O$, $W_{\mathrm{MLP},1}$, $W_{\mathrm{MLP},2}$);
    columns index the 12 transformer blocks.
    $c=0$ recovers gradient descent, $c=\tfrac{1}{2}$ recovers Muon,
    and $c=1$ recovers the pseudoinverse-like endpoint.}
    \label{fig:freon_layerwise_range_actual}
\end{figure}



\begin{figure}[t]
  \centering
  \includegraphics[width=.6\textwidth]{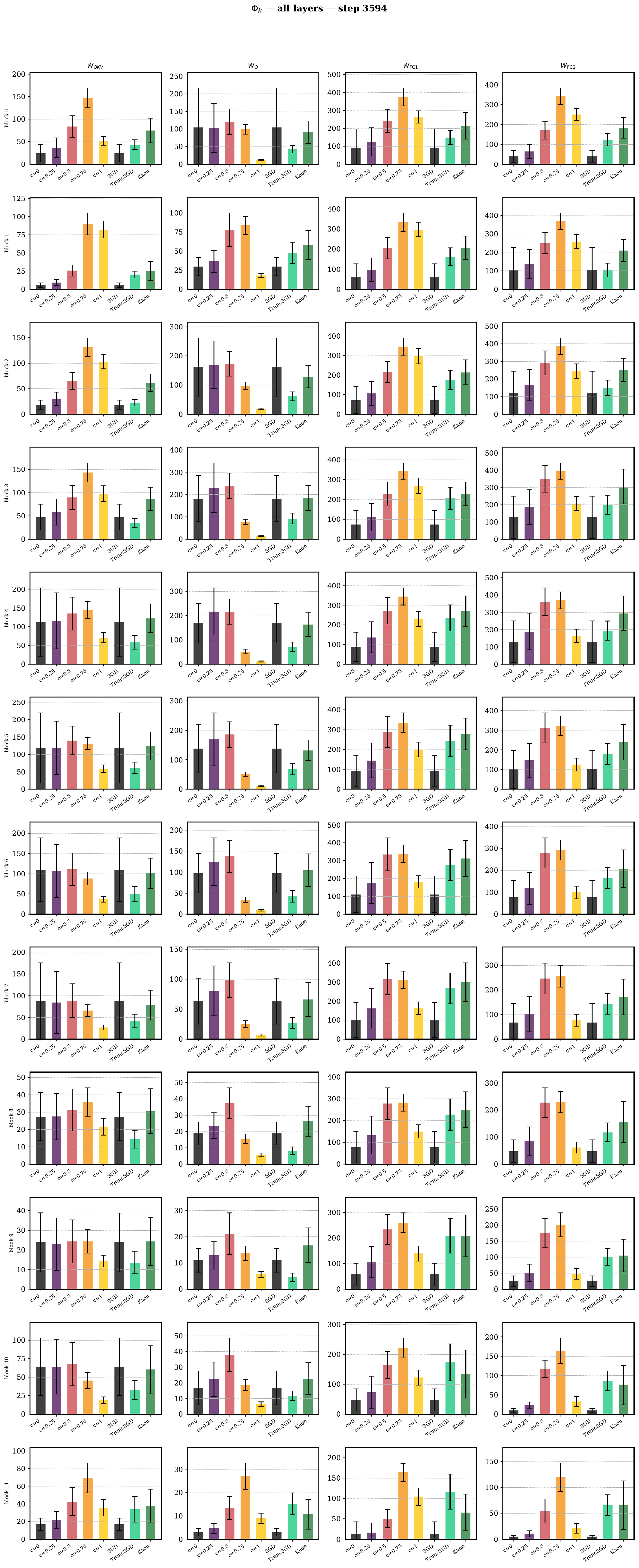}
  \caption{%
  $\Phi_k$ at training step 3594 ($\approx\!25\%$ of training)
  for all 48 weight matrices, arranged as 12 transformer blocks (rows)
  $\times$ 4 sublayer types (columns: $W_{QKV}$, $W_O$, $W_{\mathrm{FC1}}$, $W_{\mathrm{FC2}}$).
  Each bar shows the mean $\pm$ std over validation batches for a given update direction.
}

  \label{fig:grid_phi}
\end{figure}

\begin{figure}[t]
  \centering
  \includegraphics[width=.6\textwidth]{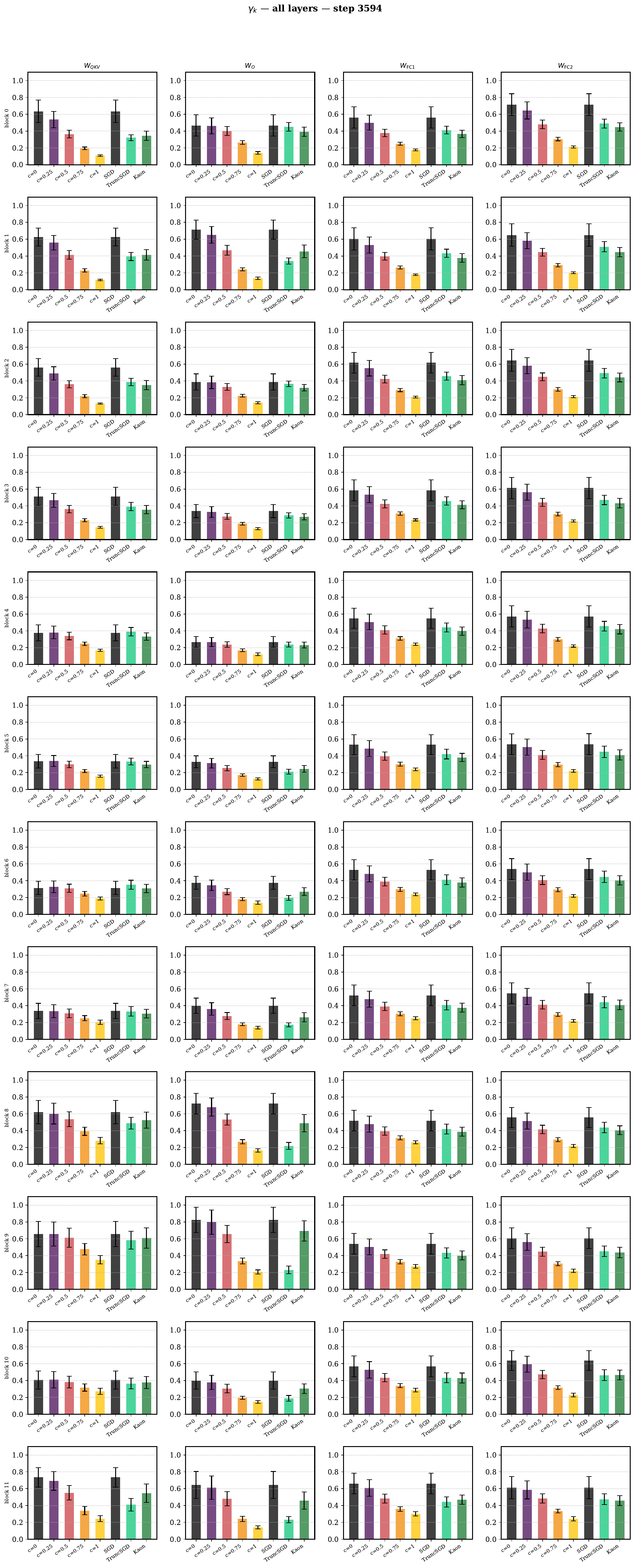}
    \caption{%
      Batch-gradient alignment $\gamma_k = \langle G, D\rangle / \langle g_b, D\rangle$
      at training step 3594 for all 48 weight matrices (same layout as
      Fig.~\ref{fig:grid_phi}).
      A value of $\gamma_k = 1$ indicates the direction $D$ aligns perfectly with the
      full gradient; values below $1$ reflect mini-batch noise or directional
      mismatch.
    }
  \label{fig:grid_gamma}
\end{figure}

\begin{figure}[htbp]
    \centering
    \includegraphics[width=.6\textwidth]{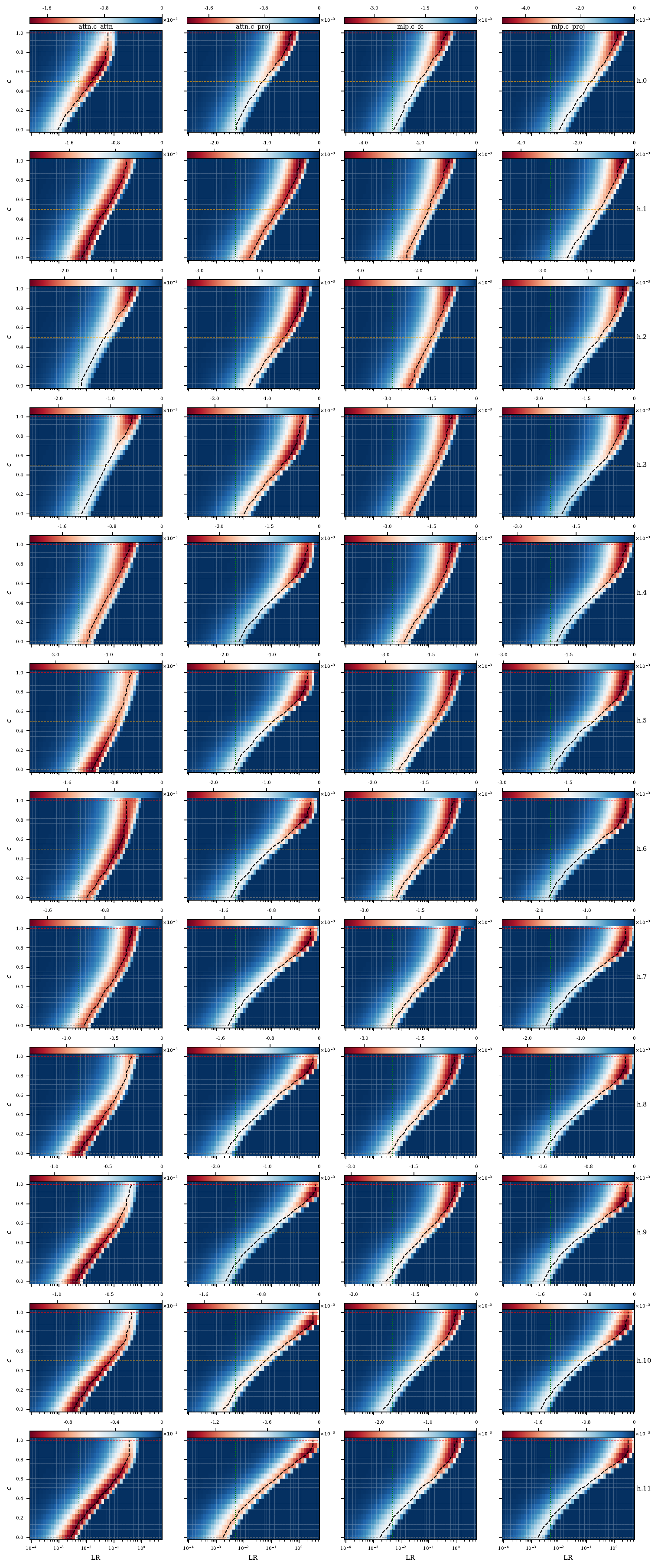}
    \caption{Loss change $\Delta  f$ over the joint $(c, \eta)$ grid at training step 800, evaluated per layer on the \textbf{full validation set}. Rows correspond to transformer blocks $h.0$–$h.11$; columns correspond to the four weight matrices within each block. The dashed black curve marks the optimal learning rate per $c$; the green dotted line marks the actual training learning rate.}
    \label{fig:pq_lr_heatmap_perlayer_val}
\end{figure}

\begin{figure}[htbp]
    \centering
    \includegraphics[width=.6\textwidth]{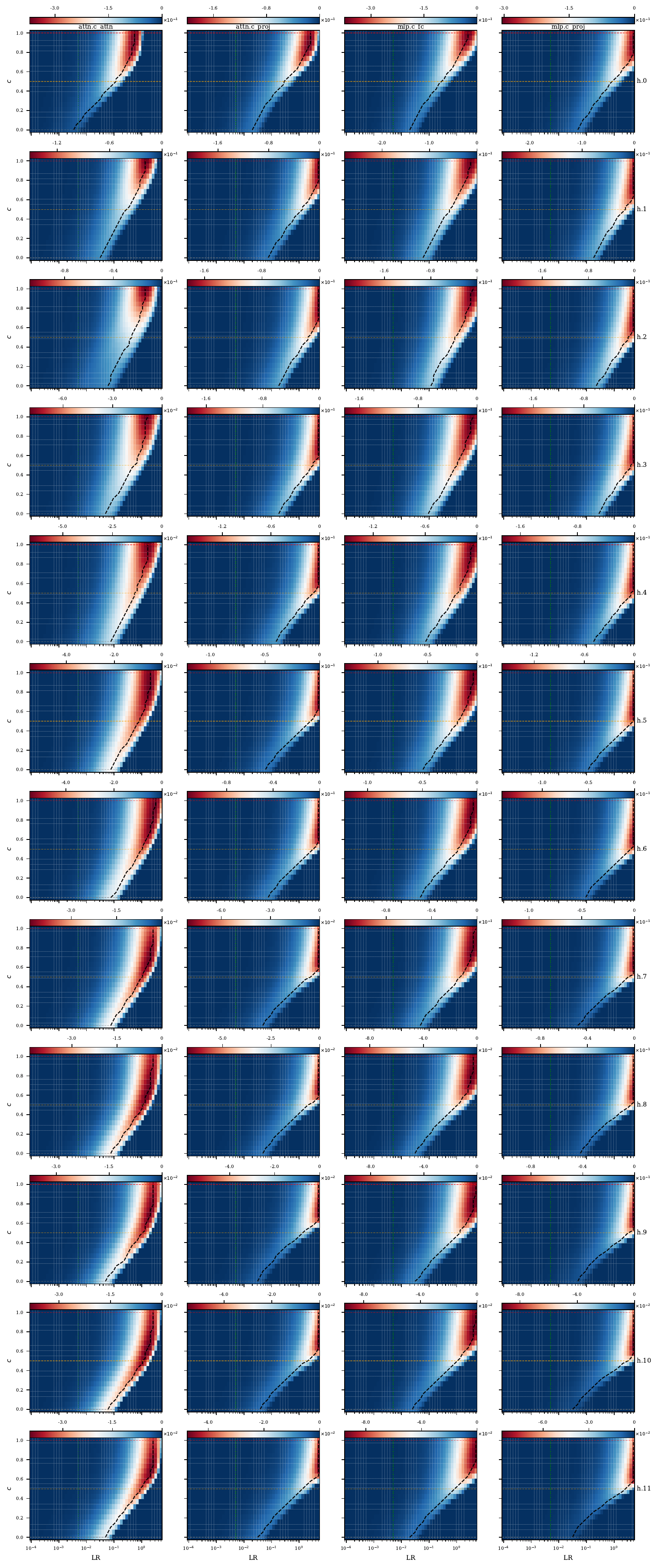}
    \caption{Loss change $\Delta  f$ over the joint $(c, \eta)$ grid at training step 800, evaluated per layer on the \textbf{same training batch as gradient}. Rows correspond to transformer blocks $h.0$–$h.11$; columns correspond to the four weight matrices within each block. The dashed black curve marks the optimal learning rate per $c$; the green dotted line marks the actual training learning rate.}
    \label{fig:pq_lr_heatmap_perlayer_train}
\end{figure}


\end{document}